%% file: main.tex
\documentclass{article} 
\usepackage{iclr2027_conference,times}
\iclrfinalcopy

\usepackage[utf8]{inputenc} 
\usepackage[T1]{fontenc}    

\usepackage{microtype}
\usepackage{graphicx}
\usepackage{subcaption}
\usepackage{booktabs} 
\usepackage{amsmath,amsfonts,bm,amsthm}
\usepackage{mathtools}
\usepackage{multirow}
\usepackage{enumitem}
\usepackage{pgfplots} 
\pgfplotsset{compat=1.18}
\usepackage{verbatim}
\usepackage{placeins}
\usepackage{needspace}

\usepackage{xcolor}
\definecolor{linkblue}{RGB}{40,80,190}
\usepackage[
    colorlinks=true,
    linkcolor=linkblue,
    citecolor=linkblue,
    urlcolor=linkblue
]{hyperref}

\usepackage{makecell} 

\input{vmr-symbols-vecbold.tex} 
\input{standard-macros.tex}        

\usepackage{url}

\usepackage{algorithm}
\usepackage{algorithmic}

\usepackage{amssymb}
\usepackage{longtable}
\usepackage{xcolor}

\usepackage[capitalize,noabbrev]{cleveref}

\theoremstyle{plain}
\newtheorem{theorem}{Theorem}[section]
\newtheorem{proposition}[theorem]{Proposition}
\newtheorem{lemma}[theorem]{Lemma}

\theoremstyle{definition}
\newtheorem{definition}[theorem]{Definition}

\theoremstyle{remark}

\usepackage{tikz}
\usetikzlibrary{patterns, positioning, calc}
\usetikzlibrary{arrows.meta}
\tikzset{
    cell/.style={draw, minimum size=0.5cm, inner sep=0pt, anchor=south west, thin},
    gray cell/.style={cell, fill=gray\!20},
    blue cell/.style={cell, fill=blue\!20},
    red cell/.style={cell, fill=red\!20},
    white cell/.style={cell, fill=white},
    label node/.style={circle, fill=white, fill opacity=0.7, text opacity=1, inner sep=1pt, font=\large\bfseries, align=center},
    title node/.style={font=\large, anchor=south}
}

\usepackage{wrapfig}

\title{On State Reduction in Linear Attention}

\author{%
  Philipp Nazari\thanks{philipp.nazari@tuebingen.mpg.de} \\
  Max Planck Institute for Intelligent Systems\\
  Department of Computer Science, ETH Z\"urich\\
  ELLIS Institute T\"ubingen\\
  \And
  T.~Konstantin Rusch\thanks{tkrusch@tue.ellis.eu}\\
  Max Planck Institute for Intelligent Systems \\
  ELLIS Institute T\"ubingen \\
  T\"ubingen AI Center \\
  Liquid AI\\
}

\begin{document}

\maketitle

\begin{abstract}
Linear attention offers a computationally efficient yet expressive alternative to softmax attention. However, recent empirical results indicate that the hidden state of trained linear attention models often exhibits a low-rank structure, suggesting that these models underexploit their capacity in practice. To understand this phenomenon, we analyze how keys and values shape the rank of the recurrent state, providing a theoretical perspective on memory utilization in linear attention. In addition to these theoretical insights, we conjecture that low-rank states can be substantially reduced after pre-training. To this end, we propose a hardware-aware strategy to structurally prune the key and query matrices, reducing the state size while retaining compatibility with existing fast kernels. We adapt several existing pruning strategies to fit our framework and, building on our theoretical analysis, propose a robust structured pruning method based on a rank-revealing QR decomposition. Our empirical evaluations across model sizes show that the recurrent state of pretrained linear attention models can often be halved with only a modest increase in language modeling perplexity, while maintaining competitive performance on downstream tasks. The code for this project can be found at \url{https://github.com/camail-official/LinearAttentionPruning}.\looseness-1
\end{abstract}

\section{Introduction}
Linear Attention~\citep{katharopoulos2020transformers,schlag2021linear,sun2023retentive,peng2023rwkv,gu2024mamba,yang2024parallelizing,dao2024transformers,yang2025gated,team2025kimi} has emerged as an efficient alternative to softmax attention \citep{vaswani2017attention}, enabling high-throughput chunkwise parallel training with linear time complexity and constant memory inference~\citep{hua2022transformer,lingle2024transformer,yang2023gated}. These efficiency gains have driven the development of large hybrid models \citep{blakeman2025nemotron,team2026qwen3} which predominantly employ linear attention layers, interspersed with only a few softmax attention layers.

Despite their impressive performance, prior work indicates that linear attention models underutilize their capacity in practice~\citep{siems2026deltaproduct,parnichkun2025quantifying}. In particular, their matrix-valued hidden states exhibit a low-rank structure after training. This state acts, through the lens of linear associative memories~\citep{ramsauer2020hopfield,wang2025test}, as a store of sequence history. Thus, low effective rank indicates that the model might be using its memory inefficiently, effectively wasting capacity. In this work, we first provide a theoretical perspective on this phenomenon. We then investigate whether the hidden state of pretrained linear attention models can be reduced to improve inference speed and memory efficiency.

Towards this end, we propose a structured pruning framework that relies on the observation that linear attention models are invariant under orthogonal transformations applied jointly to the queries and keys. Since the per-channel convolutions typically employed in recurrent models are not compatible with rotations~\citep{chollet2017xception,so2021searching,fu2022hungry}, we restrict the space of applicable transformations to key-channel selection methods. Within this space of transformations, we provide a reformulation of several established pruning strategies, including those based on parameter magnitude and gradient saliency, to reduce the dimension of the hidden state. In addition, motivated by theoretical insights into the role of rank in linear attention, we introduce a novel structured pruning approach that targets higher state rank utilization through rank-revealing channel selection.

Our results show that approximately 50\% of the key and query channels can often be removed with only modest degradation on language-modeling evaluations, even before recovery fine-tuning~\citep{hu2021lora,ma2023llm,ashkboos2024slicegpt}. In addition, our long-context results confirm a known weakness of linear attention models: purely recurrent models struggle with long horizon recall~\citep{arora2024simple, arora2024just}, whereas the hybrid Qwen3.5 9B~\citep{team2026qwen3} preserves retrieval capabilities, potentially because its attention layers retain an explicit long-context path.

\begin{figure}[t]
    \centering
    \resizebox{\textwidth}{!}{%
    \begin{tikzpicture}[scale=1] 
        \tikzset{
            cell/.style={draw, minimum size=0.5cm, inner sep=0pt, anchor=south west, thin},
            gray cell/.style={cell, fill=gray!20},
            blue cell/.style={cell, fill=blue!20},
            red cell/.style={cell, fill=red!20},
            white cell/.style={cell, fill=white},
            label node/.style={circle, fill=white, fill opacity=0.85, text opacity=1, inner sep=1pt, font=\Large\bfseries, align=center},
            title node/.style={font=\LARGE, anchor=south, align=center}
        }

        \begin{scope}[shift={(0,0)}]
            \node[title node] at (7.7, 4.5) {Unstructured sparsity};
            
            \begin{scope}[shift={(0,0)}]
                \foreach \x in {0,...,5} \foreach \y in {0,...,7} {
                    \node[gray cell] at (\x*0.5, \y*0.5) {};
                }
                \node[label node] at (1.5, 2) {X};
            \end{scope}

            \node[font=\Huge] at (3.4, 2) {$\cdot$};

            \begin{scope}[shift={(3.8, 0.5)}] 
                \foreach \x in {0,...,5} \foreach \y in {0,...,5} {
                    \node[white cell] at (\x*0.5, \y*0.5) {};
                }
                \foreach \x/\y in {0/0, 1/1, 2/2, 3/3, 4/4, 5/5, 
                                   0/5, 1/4, 5/0, 2/5, 3/1, 4/2, 
                                   0/2, 5/3, 1/0} {
                    \node[blue cell] at (\x*0.5, \y*0.5) {};
                }
                \node[label node] at (1.5, 1.5) {$\mathbf{W}_\bK$};
            \end{scope}

            \node[font=\huge] at (7.35, 2) {$=$};

            \begin{scope}[shift={(7.9, 0)}]
                \foreach \x in {0,...,5} \foreach \y in {0,...,7} {
                    \node[gray cell] at (\x*0.5, \y*0.5) {};
                }
                \node[label node] at (1.5, 2) {K};
            \end{scope}

            \draw[-{Stealth[scale=1.5]}, thick] (11.25, 2) -- (12.05, 2);

            \begin{scope}[shift={(12.4, 0.5)}]
                \foreach \x in {0,...,5} \foreach \y in {0,...,5} {
                    \node[red cell] at (\x*0.5, \y*0.5) {};
                }
                \node[label node] at (1.5, 1.5) {$\mathbf{S_t}$};
            \end{scope}

        \end{scope}

        \begin{scope}[shift={(17,0)}]
            \node[title node] at (7.7, 4.5) {Structured sparsity};

            \begin{scope}[shift={(0,0)}]
                \foreach \x in {0,...,5} \foreach \y in {0,...,7} {
                    \node[gray cell] at (\x*0.5, \y*0.5) {};
                }
                \node[label node] at (1.5, 2) {X};
            \end{scope}

            \node[font=\Huge] at (3.4, 2) {$\cdot$};

            \begin{scope}[shift={(3.8, 0.5)}] 
                \foreach \x in {0,...,5} \foreach \y in {0,...,5} {
                    \node[white cell] at (\x*0.5, \y*0.5) {};
                }
                \foreach \x in {0, 2, 3} {
                    \foreach \y in {0,...,5} {
                        \node[blue cell] at (\x*0.5, \y*0.5) {};
                    }
                }
                \node[label node] at (1.5, 1.5) {$\mathbf{W}_\bK\mathbf{T^\top}$};
            \end{scope}

            \node[font=\huge] at (7.35, 2) {$=$};

            \begin{scope}[shift={(7.9, 0)}]
                \foreach \x in {0,...,5} \foreach \y in {0,...,7} {
                    \node[white cell] at (\x*0.5, \y*0.5) {};
                }
                \foreach \x in {0, 2, 3} {
                    \foreach \y in {0,...,7} {
                        \node[gray cell] at (\x*0.5, \y*0.5) {};
                    }
                }
                \foreach \x in {1, 4, 5} {
                     \fill[pattern=north east lines, pattern color=black] (\x*0.5, 0) rectangle (\x*0.5+0.5, 4);
                     \foreach \y in {0,...,7} {
                        \draw[thin] (\x*0.5, \y*0.5) rectangle (\x*0.5+0.5, \y*0.5+0.5);
                     }
                }
                \node[label node] at (1.5, 2) {K};
            \end{scope}

            \draw[-{Stealth[scale=1.5]}, thick] (11.25, 2) -- (12.05, 2);

            \begin{scope}[shift={(12.4, 0.5)}]
                \foreach \x in {0,...,2} \foreach \y in {0,...,5} {
                    \node[red cell] at (\x*0.5, \y*0.5) {};
                }
                \node[label node] at (0.75, 1.5) {$\mathbf{S_t}$};
            \end{scope}

        \end{scope}

    \end{tikzpicture}
    }
    \caption{\textbf{Left:} Unstructured pruning yields a sparse weight matrix $\mathbf{W}_{\bK}$ yet preserves the column dimension of $\mathbf{K}$, leaving the domain of the state matrix $\mathbf{S}_t \in \mathbb{R}^{d_v, d_k}$ invariant. \textbf{Right:} Structured pruning eliminates basis vectors, mapping keys to a lower-dimensional space $\mathbb{R}^{d'_k}$ where $d'_k < d_k$. This results in a compressed state $\mathbf{S}_t \in \mathbb{R}^{d_v, d'_k}$. Figure inspired by \citet[Figure 1]{ashkboos2024slicegpt}.}
    \label{fig:one}
\end{figure}
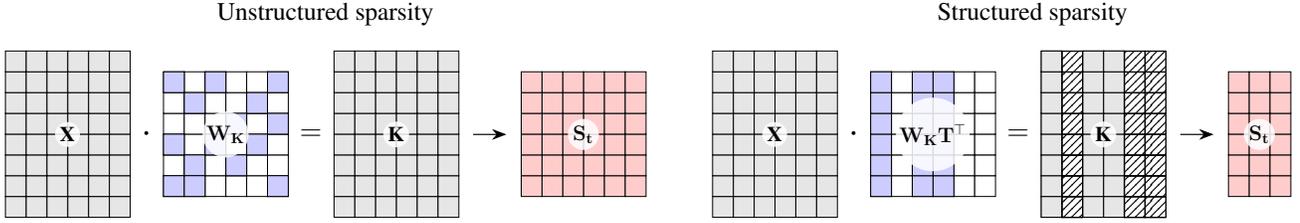

In summary, our \textbf{main contributions} are:
\begin{itemize}
    \item We provide theoretical insights into the role of rank in linear attention (Section~\ref{sec:rank-considerations}). In particular, we show that rank utilization affects readout error and can amplify query noise.
    \item Motivated by the low rank utilization observed in practice, we formulate a post-training state-size reduction framework (Section~\ref{sec:method}). Specifically, we show that the state can be reduced substantially by jointly selecting subsets of channels from the keys and queries.
    \item We adapt several existing pruning strategies to fit our framework and, building on our theoretical analysis, propose a structured pruning method based on rank-revealing QR decompositions (Section~\ref{sec:proposed-methods}).
    \item We demonstrate state reduction across recurrent models from 370M to 7B parameters and a 9B hybrid model, including end-to-end efficiency and long-context retrieval evaluations (Section~\ref{sec:exp}).
\end{itemize}

\section{Theoretical Insights}
This section establishes the theoretical foundations for our analysis. We begin by introducing the linear attention mechanism and its interpretation as an associative memory, before developing a framework to characterize the effect of rank inefficiencies within this setting.

\subsection{Background on Linear Attention}
\label{sec:background}
In its simplest form, linear attention~\citep{katharopoulos2020transformers, schlag2021linear} arises by replacing the exponential kernel in self attention with a linear one and discarding the normalization~\citep{vaswani2017attention}. Let $\bX \in \mathbb R^{T, h}$ be an input sequence that gets mapped to queries $\bQ = \bX \bW_{\bQ} \in \mathbb R^{T, d_k}$, keys $\bK = \bX \bW_{\bK} \in \mathbb R^{T, d_k}$, and values $\bV = \bX \bW_v \in \mathbb R^{T, d_v}$ via linear projections. By the associativity of matrix multiplication, linear attention can then be expressed recurrently as
\begin{equation}
\label{eq:lsa}
    \bS_t = \bS_{t-1} + \bmv_t \bmk_t^\top,
\end{equation}
where $\bS_t \in \mathbb{R}^{d_v, d_k}$ is a matrix-valued hidden state used to compute the output $\bmo_t = \bS_t \bmq_t$, and $\bS_0={\bf 0}$. \citet{katharopoulos2020transformers} provide a detailed derivation. The state $\bS_t$ implements a linear associative memory~\citep{ramsauer2020hopfield,wang2025test} that is addressed using the queries. Throughout this work, we will refer to the update rule of $\bS_t$ as the sequence mixer.

The mechanism defined in Equation~\eqref{eq:lsa} can only write into the associative memory. DeltaNet~\citep{schlag2021linear,yang2024parallelizing} addresses this limitation by explicitly erasing old information correlated with the current key before writing the new value to $\bS_t$,
\begin{equation}
\label{eq:dn-recursion}
    \bS_t = \bS_{t-1} (\mathbf{I} - \beta_t \bmk_t \bmk_t^\top) + \beta_t \bmv_t \bmk_t^\top.
\end{equation}
Gated DeltaNet~\citep{yang2025gated} further refines this mechanism by introducing a data-dependent decay term.

\subsection{On the Role of the Rank}
\label{sec:rank-considerations}
\begin{wrapfigure}[16]{r}{0.5\textwidth}
\vspace{-1.2\baselineskip}
    \centering
    \input{media/svd_shared_range_L7_skip128.pgf}%
    \caption{Singular value spectra for a DeltaNet 370M head (FineWeb, $T=2048$, first 128 tokens skipped). \textbf{Left:} Uncompressed baseline. \textbf{Right:} DRRQR at 75\% compression (pre-RFT).}
    \label{fig:histograms-main}
\end{wrapfigure}
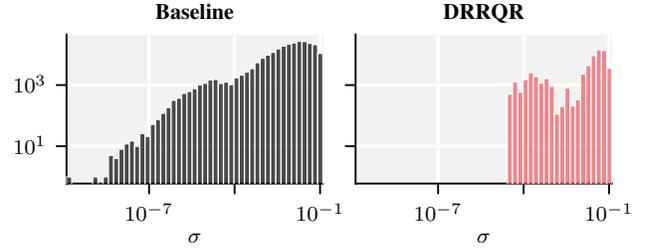
Linear attention compresses the sequence history into a fixed-size hidden state, which generally entails a loss of information. This compression is commonly viewed as a limitation relative to softmax attention, whose growing KV cache allows each query to directly attend to individual past tokens. However, recent empirical results indicate that pretrained linear attention models don't manage this memory well~\citep{siems2026deltaproduct,parnichkun2025quantifying}, exhibiting an effective low-rank structure in practice. We confirm this finding. Figure~\ref{fig:histograms-main} shows the characteristic tail of the singular value distribution of a pretrained DeltaNet 370M. In the remainder of this section, we develop a framework that illuminates this phenomenon.

By the nature of linear attention, the subspace spanned by the associative memory is governed by the keys and values. Formally, for a family of linear attention models, the row and column spaces of the associative memory satisfy
\begin{equation*}
    \operatorname{row} \bS_t \subseteq \operatorname{span}(\bmk_1,\ldots,\bmk_t),  \qquad \operatorname{col} \bS_t \subseteq \operatorname{span}(\bmv_1,\ldots,\bmv_t).
\end{equation*}
It thus follows that
\begin{equation}
\label{eq:rank-S-bound}
    \rank \bS_t \leq \min \left( \rank \bK_t, \rank \bV_t \right) \leq t,
\end{equation}
where $\bK_t = [\bmk_1,\ldots,\bmk_t]$ and $\bV_t = [\bmv_1,\ldots,\bmv_t]$ are the matrices obtained by stacking the keys and values, respectively. The proof is presented in Appendix~\ref{appendix:proof-rank}.

Equation~\eqref{eq:rank-S-bound} shows that the algebraic rank of the associative memory is bounded by the algebraic ranks of the keys and values. However, algebraic rank is too rigid of a metric for noisy real-world data. For a non-zero matrix $\bS \in \mathbb R^{d_v,d_k}$, we therefore use its \textit{effective rank} (or \textit{stable rank})~\citep{rudelson2007sampling,tropp2015introduction,vershynin2019high,ipsen2025stable}, defined as
\begin{equation*}
    \operatorname{er}(\bS) \coloneqq \frac{\|\bS\|_F^2}{\|\bS\|_2^2}
    = \frac{\sum_i \sigma_i^2}{\sigma_1^2},
\end{equation*}
where $\sigma_1 \geq \ldots \geq \sigma_{\min(d_v,d_k)} \geq 0$ are the singular values of $\bS$. Intuitively, $\operatorname{er}(\bS)$ measures the skewness of the singular-value spectrum of $\bS$.

An effective rank analogue of the bound in Equation~\eqref{eq:rank-S-bound} relates the effective rank of $\bS$ to the conditioning of the keys. In particular, for the linear attention recurrence \mbox{$\bS_t = \bS_{t-1} + \bmv_t\bmk_t^\top$}, there exists a scalar quantity $\nu(\bV_t)$ such that
\begin{equation}
\label{eq:effective-rank-bound}
    \frac{\nu(\bV_t)}{\kappa^2(\bK_t)} \leq \operatorname{er}(\bS_t),
\end{equation}
where $\kappa$ denotes the $l^2$ condition number. Appendix~\ref{appendix:proof-stable-rank} proves this bound. The bound extends to DeltaNet after a simple reparameterization. Define the pseudo-values
\begin{equation*}
    \bmu_t \coloneqq \beta_t(\bmv_t-\bS_{t-1}\bmk_t).
\end{equation*}
Then Equation~\eqref{eq:dn-recursion} becomes $\bS_t=\bS_{t-1}+\bmu_t\bmk_t^\top$. Equation~\eqref{eq:effective-rank-bound} therefore holds with $\nu(\bU_t)$ in place of $\nu(\bV_t)$. This bound motivates considering key conditioning when selecting channels for pruning later on.

\subsubsection{Rank Utilization}
While the effective rank measures the effective dimensionality of the state matrix, it does not capture how efficiently a recurrent model uses its available memory. For a non-zero state $\bS \in \mathbb{R}^{d_v,d_k}$ with maximum capacity $d \coloneqq \min(d_k,d_v)$, we define its \textit{rank utilization} as
\begin{equation}
\label{eq:rank-utilization}
    \operatorname{u}(\bS) \coloneqq \frac{\operatorname{er}(\bS)}{d}.
\end{equation}
Rank utilization sets effective rank in relation to the theoretical maximum capacity and thus serves as a measure for memory utilization in linear attention models. We identify two edge cases:
\begin{itemize}
    \item \textbf{Low Utilization ($\operatorname{u} \ll 1$):} The memory suffers from {rank collapse}. Its energy is concentrated in a few principal components. Information is stored in an effectively low-dimensional subspace.
    \item \textbf{High Utilization ($\operatorname{u} = 1$):} The memory is isotropic. Energy is distributed evenly across all dimensions.
\end{itemize}
Dividing Equation~\eqref{eq:effective-rank-bound} by $d$ yields
\begin{equation*}
    \frac{\nu(\bV_t)}{d\kappa^2(\bK_t)} \leq \operatorname{u}(\bS_t), 
\end{equation*}
which establishes a connection between the conditioning of the keys and the rank utilization of the corresponding memory.

\subsubsection{Why Does Rank Utilization Matter?}
\label{sec:why-utilization-matters}
The previous sections introduced effective rank and rank utilization as measures of the isotropy of an associative memory. We now connect these measures to noise sensitivity during readout. Intuitively, a memory with a skewed singular spectrum responds more strongly to some directions than to others. As a result, a small relative change in the query can lead to a larger relative change in the output.

To formalize this phenomenon, consider a perturbed query $\tilde{\bmq}=\bmq^*+\bm{\varepsilon}$, $\|\bmq^*\|_2=1$, with outputs $\bmo^*=\bS\bmq^*$ and $\bmo=\bS\tilde{\bmq}$, where we drop the time index for brevity. Define the relative error amplification
\begin{equation}
\label{eq:relative-readout-amplification}
    R(\bmq^*,\bm{\varepsilon}) \coloneqq
    \frac{\|\bmo-\bmo^*\|_2/\|\bmo^*\|_2}{\|\bm{\varepsilon}\|_2/\|\bmq^*\|_2}.
\end{equation}
Values above one indicate amplification, while values below one indicate attenuation of the query noise. The amplification ratio satisfies
\begin{equation*}
    R(\bmq^*,\bm{\varepsilon})
    =\frac{\|\bS\bm{\varepsilon}\|_2/\|\bm{\varepsilon}\|_2}{\|\bS\bmq^*\|_2} \leq\frac{\sigma_1(\bS)}{\|\bS\bmq^*\|_2}.
\end{equation*}
This bound is attained when the perturbation $\bm{\varepsilon}$ follows the strongest right singular direction of the memory $\bS$. Its denominator is small when the query follows the weakest right singular direction, giving \mbox{$\|\bS\bmq^*\|_2=\sigma_{d_k}(\bS)$}. The worst case combines a strong response to noise with a weak query signal:\looseness-1
\begin{proposition}[Rank Utilization and Readout Sensitivity]
\label{thm:snr-degradation}
    Let $\bS\in\mathbb{R}^{d_v,d_k}$ have full column rank. The worst-case relative amplification satisfies
    \begin{equation}
    \label{eq:snr-degradation}
    \max_{\|\bmq^*\|_2=1}\max_{\bm{\varepsilon}\neq\mathbf{0}} R(\bmq^*,\bm{\varepsilon})
        =\kappa(\bS)\geq\frac{1}{\sqrt{u(\bS)}}.
    \end{equation}
\end{proposition}
The proof of this statement is provided in Appendix~\ref{appendix:output-util-proof}. This proposition shows that error amplification is governed by the alignment between the state and the noise compared to the query. In particular, it reveals how low rank utilization can lead to large error amplification.

We distinguish two regimes:
\begin{itemize}
    \item \textbf{Low Utilization ($u(\bS)\ll1$):} The memory's skewed singular value spectrum permits large error amplification when noise follows a strong singular direction and the query follows a weak one.
    \item \textbf{High Utilization ($u(\bS)= 1$):} The memory becomes isotropic. All dimensions carry the same energy and relative error is preserved.
\end{itemize}
Beyond amplifying readout error, low rank utilization also adversely affects the conditioning of the query weight gradients during the backward pass. See Appendix~\ref{app:matrix-derivative} for details.

\section{The Proposed Pruning Framework}
\label{sec:method}
Motivated by the above insights into the role of rank in linear attention, we propose a structured pruning framework to reduce the per-head key dimension of linear attention models. Besides revealing redundancy in pretrained linear attention models, our pruning strategy yields a strictly smaller hidden state for larger throughput during inference.

\subsection{Invariant Transformations for Linear Attention}
Decomposing $\bS_t$ row-wise reveals that the sequence mixer simulates $d_v$ parallel linear time-varying (LTV) dynamical systems. Indeed, let $\bmh_t^{(i)} \in \mathbb R^{d_k}$ denote the transpose of the $i$-th row of $\bS_t$. Each value channel $i \in \{1, \dots, d_v\}$ follows the vector-valued dynamics
\begin{equation}
\label{eq:dyn-system}
    \bmh_t^{(i)} = \mathbf{A}_t \bmh_{t-1}^{(i)} + \mathbf{B}_t v_{t,i},
\end{equation}
where the system matrices $\mathbf{A}_t = \bI - \beta_t \bmk_t \bmk_t^\top$ and \mbox{$\mathbf{B}_t = \beta_t \bmk_t$} are shared across all value channels. The readout is computed via contraction with the vector $\bC_t = \bmq_t^\top$.

The dynamical system in Equation~\eqref{eq:dyn-system} is invariant under a choice of basis in state-space~\citep{chahine2026curious,chen1984linear}. That is, every transformation
\begin{equation*}
(\mathbf{A}_t,\mathbf{B}_t,\mathbf{C}_t) \to (\mathbf{T}\mathbf{A}_t\mathbf{T}^{-1}, \bT\mathbf{B}_t, \mathbf{C}_t\mathbf{T}^{-1})
\end{equation*}
derived from an invertible matrix $\bT \in \operatorname{GL}(d_k)$ leaves the input-output mapping invariant. If the transformed system should also be expressible in the same functional form as the original recurrence, requiring invertibility is not sufficient. Indeed, the transformed transition matrix
\begin{equation*}
    \tilde{\bA}_t = \bT \bA_t \bT^{-1} = \bI - \beta_t(\bT\bmk_t)(\bT^{-\top}\bmk_t)^\top,
\end{equation*}
is generally not symmetric and can thus not be expressed as a delta-corrected linear recurrence. If additionally $\bT \in O(d_k)$ is an orthogonal change of basis, this transformation can be absorbed to preserve the structure of the recurrence:
\begin{proposition}[Orthogonal Invariance of Sequence Mixing]
\label{prop:invariance}
Let $\bT \in O(d_k)$ be an orthogonal matrix. Then the (Gated) DeltaNet attention mechanism is invariant under the \emph{simultaneous} transformation \mbox{$(\bmk_t, \bmq_t) \mapsto (\bT \bmk_t, \bT \bmq_t)$.}
\end{proposition}
Our structured pruning framework, presented in the subsequent section, makes use of this result by applying a {semi}-orthogonal transformation jointly to queries and keys.

\subsection{State Size Reduction Requires Structured Pruning}
Existing pruning methods for Large Language Models generally rely on unstructured or semi-structured sparsity~\citep{frantar2023sparsegpt,sun2024simple}. However, these methods do not reduce the state dimension of the dynamical system, as they leave the query and key vectors dense (see Figure~\ref{fig:one}). To address this issue, we focus on {structured} pruning.

Formally, given a semi-orthogonal matrix $\bT \in \mathbb R^{d_k', d_k}$ with target dimension $d_k' < d_k$, we apply it {jointly} to the keys and queries before computing the dynamical system:
\begin{equation*}
    (\bmk_t, \bmq_t) \to (\bT\bmk_t, \bT \bmq_t) \quad \forall t = 1,\ldots,T.
\end{equation*}
By Proposition~\ref{prop:invariance}, if $d_k' = d_k$, this transformation yields an equivalent realization of the dynamical system. Our goal is finding a $\mathbf{T}$ with $d_k' \ll d_k$ that entails a small error in the reduced system. A natural framework for deriving such a transformation would be Principal Component Analysis (PCA). By computing the empirical covariance $\hat{\mathbf{\Sigma}} = \frac{1}{N} \sum_{t} \mathbf{k}_t \mathbf{k}_t^\top$ over a calibration set, one could derive a projection matrix $\mathbf{T}$ from the top-$d_k'$ principal components. However, while theoretically optimal for key-reconstruction~\citep{eckart1936approximation,mirsky1960symmetric}, this approach is not compatible with the causal convolutions employed in linear attention architectures:

\paragraph{The Convolution Constraint.}
To achieve actual wall-clock speedup, we aim to absorb $\mathbf{T}$ into the weights $\mathbf{W}_{\bK}$ and $\mathbf{W}_{\bQ}$. However, linear attention models typically apply causal convolutions to queries and keys~\citep{chollet2017xception,so2021searching,fu2022hungry,yang2023gated,gu2024mamba,dao2024transformers,yang2024parallelizing}. While the sequence mixer is compatible with orthogonal transformations (Proposition~\ref{prop:invariance}), depthwise convolutions are not. A dense semi-orthogonal matrix $\mathbf{T}$, such as one derived from PCA, would mix channels, misaligning them with their corresponding convolution filters. Preserving the dynamics under such rotations would require converting the efficient diagonal convolution filters into expensive dense matrices. Therefore, we effectively constrain our framework to {axis-aligned} transformations. The case of general orthogonal matrices is discussed in Appendix~\ref{appendix:conv-gen-orth}.

\subsection{Axis-Aligned Methods}
\label{sec:axis-aligned}
To circumvent the above intricacies related to general dense rotations, we focus on transformations that preserve the channel-wise convolution filters. Formally, this restricts transformations to {axis-aligned} semi-orthogonal matrices:
Let \mbox{$\mathcal{I} = \{i_1, \dots, i_{d_k'}\} \subseteq \{1, \dots, d_k\}$} be a set of distinct indices with cardinality $d_k' < d_k$. We define the corresponding axis-aligned projection $\mathbf{P}_{\mathcal{I}} \in \mathbb{R}^{d_k', d_k}$ by selecting the associated standard basis vectors as rows:
\begin{equation*}
    \mathbf{P}_{\mathcal{I}} \coloneqq \begin{bmatrix} \mathbf{e}_{i_1} & \cdots & \mathbf{e}_{i_{d_k'}} \end{bmatrix}^\top.
\end{equation*}
This matrix is semi-orthogonal. Applying it corresponds to a structural pruning operation that selects the subset of key-channels $\mathcal{I}$ and discards the rest. Crucially, this operation preserves the independence of the remaining channels and can be obtained by slicing the corresponding columns from the projection matrices $\mathbf{W}_{\bK}$ and $\mathbf{W}_{\bQ}$, along with the corresponding entries of the convolution filters.\looseness-1

In the following section, we present a plethora of different strategies to select the index set $\mathcal{I}$ for each head, ranging from simple magnitude-based heuristics to a novel rank-revealing approach. By Proposition~\ref{prop:invariance}, every transformation must be applied to the queries and keys {simultaneously}. Consequently, the proposed column selection mechanisms incorporate both queries and keys.

\subsection{Proposed Pruning Methods}
\label{sec:proposed-methods}

\textbf{Weight-Magnitude-based ($\mathbf{L^1}$).}
As a baseline, we implement a weight-magnitude based pruner, defining the importance of the $j$-th key channel as the sum of the $l^1$-norms of the corresponding columns in the query and key projection matrices:
\begin{equation*}
s_j = \|\mathbf{W}_{\bQ, :j}\|_1 + \|\mathbf{W}_{\bK, :j}\|_1.
\end{equation*}
We rank these scores locally within each head and select the top-$d_k'$ indices to form the retained set $\mathcal{I}$.

\textbf{S-Wanda.}
We adapt Wanda~\citep{sun2024simple} to our structured setting (coining the method {S-Wanda}), and define the saliency of the $j$-th channel by aggregating the element-wise Wanda scores across the channel dimension:
\begin{equation*}
    s_j = \sum_{i=1}^{h} \left(|\mathbf{W}_{q, ij}| + |\mathbf{W}_{k, ij}|\right) \|\mathbf{X}_{:i}\|_2,
\end{equation*}
where $\|\mathbf{X}_{:i}\|_2$ is the $l^2$ norm of the $i$-th input feature computed over a calibration set.
Unlike $L^1$, this metric accounts for the distribution of the input.

\textbf{Sensitivity-based.}
We employ a gradient-based saliency criterion~\citep{lecun1989optimal,wang2019eigendamage,ma2023llm} to identify dimensions that maximally influence the training objective. It is our only method that takes into account information beyond the current layer and is task-aware. We quantify the importance of the $j$-th key dimension using the first-order Taylor expansion of the loss on a calibration set:
\begin{equation*}
    s_j = \sum_{i} \left( \left| \mathbf{W}_{\bQ, ij} \nabla_{\mathbf{W}_{\bQ, ij}} \mathcal L \right| + \left| \mathbf{W}_{\bK, ij} \nabla_{\mathbf{W}_{\bK, ij}} \mathcal L \right| \right).
\end{equation*}
\textbf{Rank-based (DRRQR).}
Motivated by the bound in Equation~\eqref{eq:effective-rank-bound}, we propose a simple algorithm that uses conditioning to guide channel selection. \emph{Deep Rank Revealing QR (DRRQR)} applies a Strong Rank Revealing QR factorization~\citep{gu1996efficient} to the activation statistics $\mathbf{M} = [\bK; \bQ]$ to select a subset of $d_k'$ feature channels. Strong RRQR selects feature channels with a provable bound on their condition number. Algorithm~\ref{alg:drrqr} starts from a QR decomposition with column pivoting and then iteratively permutes columns between a current chosen basis and a set of candidates to satisfy numerical stability bounds. We detail the specific swap-gain metric, update rules, and conditioning bound in Appendix~\ref{app:drrqr_details}. Unlike heuristic approaches, DRRQR relies on deterministic guarantees for rank revelation of the joint key-query activation matrix, rendering the method tractable and interpretable.\looseness-1

\section{Related Work}
\label{sec:related-work}

\textbf{Rank Considerations.}
Recent empirical results show that the memory of linear attention models often exhibits a low-rank structure~\citep{siems2026deltaproduct,parnichkun2025quantifying}, indicating that they do not manage their associative memory well. Our work builds on these observations, forcing the model to operate in a lower-dimensional space with increased rank utilization and a minimal decrease in effective rank~\citep{ipsen2025stable}.

More generally, rank collapse is a well-known phenomenon in transformers. \citet{noci2022signal} reveal that collapsed queries and keys hinder gradient flow at initialization, while \citet{dong2021attention} show that skip connections help alleviate it. In this work, we tie the rank of the queries and keys to that of the hidden state in linear attention and show how a skewed spectrum of the associative memory can amplify query noise during readout.

\textbf{Pruning Methods.}
Conventional pruning methods are usually either unstructured or semi-structured~ and thus require specialized kernels to realize sparsity-induced speedups~\citep{lecun1989optimal,frantar2023sparsegpt,sun2024simple}. SparseGPT~\citep{frantar2023sparsegpt} frames pruning as a local reconstruction problem. \citet{sun2024simple} propose a simpler, gradient-free metric based on the product of weight magnitudes and input activation norms. Several works have sought to enhance this metric by re-incorporating gradients. GBLM~\citep{das2023beyond} and Pruner-Zero~\citep{dong2024pruner} utilize gradients derived from full-model backpropagation to refine pruning scores. Wanda++~\citep{yang2025wanda++} introduces regional gradients to decrease the computational cost.

\begin{wrapfigure}[19]{r}{0.5\textwidth}
\vspace{-1.2\baselineskip}
    \centering
    \input{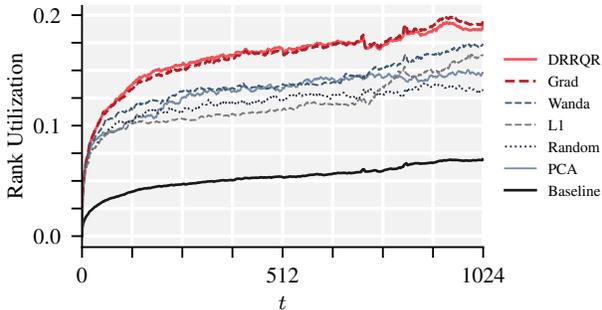}%
    \caption{Rank utilization of DeltaNet 370M as a function of the token index for a random sample of Fineweb, averaged over layers and heads, at a compression ratio of 75\% pre-RFT. All models except for the baseline have the same state dimensions and are therefore directly comparable.}
    \label{fig:rank-util-curve}
\end{wrapfigure}

A challenge in structured pruning is the handling of coupled structures, where removing a neuron or head in one layer breaks dimensional consistency in subsequent layers. \citet{ma2023llm} address this issue by constructing dependency graphs to identify groups of parameters that must be excised simultaneously. Similarly, our work addresses the structural coupling of linear attention models, specifically the dependency between the projection matrices of queries and keys and the depthwise convolutions.

SliceGPT~\citep{ashkboos2024slicegpt} prunes a model's backbone. This affects the rows of $\bW_\bQ$ and $\bW_\bK$. In particular, the attention mechanism operates in the same space pre- and post-slicing. Our work can be considered complementary to this approach. By pruning columns of $\bW_\bQ$ and $\bW_\bK$, we explicitly reduce the dimension of the attention mechanism. ThinK~\citep{xu2025think} prunes KV cache channels in softmax Transformers based on attention weight magnitudes. While conceptually related to our framework, this method was designed for the KV cache of softmax attention layers, not the state of linear attention ones.

\textbf{Quantization Methods.}
Recent Transformer quantization approaches like QuaRot~\citep{ashkboos2024quarot} and SpinQuant~\citep{liu2025spinquant} apply orthogonal rotations to both the residual stream and the state-space to disperse outliers. While effective for softmax Transformers, these state-space rotations are incompatible with the causal convolutions typically employed in linear attention models. 

Recent works on quantizing linear attention variants, including Quamba~\citep{chiang2025quamba, chiang2025quamba2}, MambaQuant~\citep{yue2025mambaquant}, and SSDi8~\citep{kim2026ssdi8}, further demonstrate the community's interest in compressing the sequence mixer of linear attention models. Low-rank factorization methods such as FWSVD~\citep{hsu2022language} and ASVD~\citep{yuan2023asvd} approximate weight matrices but do not reduce the state dimension: the recurrence still operates in the original $d_k$-dimensional space, yielding no speedup in the sequence mixer. Our framework instead utilizes axis-aligned transformations, which are compatible with causal convolutions and directly reduce state size.\looseness-1

\begin{table}[t!]
\caption{Post-RFT evaluation of Gated DeltaNet 1.3B. We report WikiText-2 and Lambada perplexity and average zero-shot accuracy (ZS,~\citet{eval-harness}). Best results are \textbf{bold} and second-best results are \underline{underlined}.}
\label{tab:full_results_gdn_13B_paper}
\centering
\small
\setlength{\tabcolsep}{1.5pt}
\begin{sc}
\begin{tabular}{l ccc ccc ccc}
\toprule
& \multicolumn{3}{c}{\textbf{75\%}} & \multicolumn{3}{c}{\textbf{50\%}} & \multicolumn{3}{c}{\textbf{30\%}} \\
Method & {Wiki~$\downarrow$} & {LMB~$\downarrow$} & {ZS~$\uparrow$} & {Wiki~$\downarrow$} & {LMB~$\downarrow$} & {ZS~$\uparrow$} & {Wiki~$\downarrow$} & {LMB~$\downarrow$} & {ZS~$\uparrow$} \\
\midrule
Rand    & 19.3 & 18.1 & 54.6 & 16.9 & \underline{11.5} & 57.8 & \underline{16.1} & \underline{10.2} & 58.7 \\
L1      & 18.6 & 16.8 & 55.3 & 16.8 & 13.0 & 57.2 & 16.3 & 11.3 & 58.3 \\
S-Wanda & 18.2 & 15.1 & 55.4 & 16.6 & 11.7 & 57.8 & \underline{16.1} & 11.0 & 58.6 \\
Grad    & \textbf{17.8} & \textbf{12.5} & \textbf{56.8} & \underline{16.4} & \textbf{10.5} & \textbf{58.3} & \textbf{15.9} & \textbf{9.9} & \textbf{59.0} \\
DRRQR   & \underline{17.9} & \underline{14.7} & \underline{56.2} & \textbf{16.3} & 12.0 & \underline{58.0} & \textbf{15.9} & 10.7 & \underline{58.8} \\
\midrule
Baseline & 16.8 & 9.7 & 59.4 & 16.8 & 9.7 & 59.4 & 16.8 & 9.7 & 59.4 \\
\bottomrule
\end{tabular}
\end{sc}
\end{table}

\section{Experiments}
\label{sec:exp}
We use flame~\citep{yang2025flame} for recovery fine-tuning (RFT,~\citet{ashkboos2024slicegpt}) and for pre-training the 370M parameter models. RFT uses LoRA~\citep{hu2021lora}. The details on our experimental setup can be found in Appendix~\ref{app:exp-setup}. Besides the channel selection methods introduced in Section~\ref{sec:method}, we also report a ``Rand'' baseline which randomly selects key channels to drop.

\begin{wraptable}[11]{r}{0.3\textwidth}
\vspace{-1.1\baselineskip}
\caption{DeltaNet 370M WikiText-2 perplexity at 75\% state reduction before RFT. Each calibration set contains 128 sequences of length 2048.}
\label{tab:calibration-main}
\centering
\small
\begin{tabular*}{\linewidth}{@{\extracolsep{\fill}}lcc@{}}
\toprule
Corpus & DRRQR & Grad \\
\midrule
FineWeb & 41.6 & 44.9 \\
PG19 & 41.5 & 53.8 \\
Pile & 41.7 & 51.0 \\
\bottomrule
\end{tabular*}
\end{wraptable}

\paragraph{Accuracy}
We evaluate our structured pruning framework on DeltaNet and Gated DeltaNet~\citep{yang2024parallelizing,yang2025gated} with both 370M and 1.3B parameters, as well as Mamba-Codestral-7B\footnote{\url{https://huggingface.co/mistralai/Mamba-Codestral-7B-v0.1}} and Qwen3.5 9B~\citep{team2026qwen3}. Specifically, we compare the optimization-based approaches DRRQR (rank-based) and Grad (gradient-based) with the magnitude-based strategies.

Figure~\ref{fig:rank-util-curve} shows the rank utilization during a forward pass through DeltaNet 370M. It shows that the two best performing compression methods, Grad and DRRQR, have the highest rank utilization. A fine-grained per-head analysis (Appendix~\ref{app:per-head-rank}) reveals that rank utilization is heterogeneous across heads, and that DRRQR lifts it consistently. Figure~\ref{fig:histograms-main} shows how DRRQR removes the tail of the hidden states' spectrum.

\begin{wraptable}[11]{r}{0.21\textwidth}
\vspace{-1.2\baselineskip}
\caption{Mamba-Codestral-7B before RFT at 50\% state reduction.}
\label{tab:mamba-scale-main}
\centering
\small
\setlength{\tabcolsep}{2pt}
\begin{tabular}{lcc}
\toprule
Method & Wiki $\downarrow$ & ZS $\uparrow$ \\
\midrule
Baseline & 14.5 & 64.4 \\
Grad & 16.7 & 59.9 \\
DRRQR & 17.6 & 57.1 \\
L1 & 25.7 & 48.8 \\
\bottomrule
\end{tabular}
\end{wraptable}

Surprisingly, we find that pruning half of the key dimension causes only a modest degradation in WikiText-2 and Lambada perplexity as well as zero-shot language-modeling performance, especially under Grad and DRRQR. Table~\ref{tab:full_results_gdn_13B_paper} reports the post-RFT results across compression ratios. Table~\ref{tab:unified_summary_50} in the appendix additionally compares pre- and post-RFT performance across model families, showing that pre-RFT degradation depends on the model and pruning method. These results indicate that the evaluated models use substantially less than their nominal recurrent capacity for next-token prediction.

Notably, DRRQR and Grad often perform similarly well. DRRQR is thus competitive with the only task-aware, gradient-based method. Table~\ref{tab:calibration-main} isolates the practical distinction between the two selectors. DRRQR transfers across calibration corpora. Grad is competitive in-domain but suffers under distribution shifts.

\begin{wraptable}[12]{r}{0.4\textwidth}
\caption{Long-context retrieval after RFT at 50\% state reduction. Only the recurrent components of Qwen3.5 9B are finetuned.}
\label{tab:recall-hybrid-main}
\centering
\small
\setlength{\tabcolsep}{1.8pt}
\begin{tabular}{llccccc}
\toprule
Model & Method & \makecell{RULER~$\uparrow$\\4K} & \makecell{NIAH~$\uparrow$\\4K} \\
\midrule
\multirow{2}{*}{GDN 1.3B} & Baseline & 38.8 & 81.3 \\
& DRRQR & 27.5 & 61.9 \\
\midrule
\multirow{2}{*}{Qwen3.5 9B} & Baseline & 93.9 & 100.0 \\
& DRRQR & 93.6 & 100.0 \\
\bottomrule
\end{tabular}
\end{wraptable}

Table~\ref{tab:mamba-scale-main} demonstrates the applicability of our pruning framework to a 7B Mamba-2 model before RFT. DRRQR and Grad degrade substantially less than magnitude based pruning. Appendix~\ref{app:mamba2} reports more compression ratios and the full task breakdown.

Recurrent models are known to perform poorly in long-context recall~\citep{arora2024simple,arora2024just}. Our experiments confirm that observation. We furthermore find that, at 50\% reduction, the purely recurrent Gated DeltaNet loses recall accuracy. At the same time, the hybrid Qwen3.5 9B matches its baseline (Table~\ref{tab:recall-hybrid-main}) under a 50\% state reduction in its linear layers. This contrast supports the hypothesis that exact attention layers preserve a long-range memory path. Appendix~\ref{app:long-context} reports moderate compression ratios and more long-context results.


\begin{wrapfigure}[16]{r}{0.65\textwidth}
\vspace{-1.2\baselineskip}
    \centering
    \input{media/speedup.pgf}%
    \caption{End-to-end speedup of Gated DeltaNet 1.3B on an H100, relative to the uncompressed baseline. (a) Decode, prompt length 512. (b) Prefill, at context length 2K and 8K. Complete results for Gated DeltaNet and DeltaNet are provided in Appendix~\ref{app:end-to-end-efficiency}.}
    \label{fig:speedup}
\end{wrapfigure}

\paragraph{Speedup}

We benchmark end-to-end inference speedup under our framework using flash-linear-attention kernels~\citep{yang2024fla}. Each configuration uses 10 warmup steps and 100 measured steps, and we report the median of three repetitions. We show Gated DeltaNet 1.3B in the main text. Appendix~\ref{app:end-to-end-efficiency} provides the complete results for both Gated DeltaNet and DeltaNet 1.3B.

Figure~\ref{fig:speedup}\,(a) shows that decode speedup increases with state reduction and batch size, reaching up to $\sim50\%$ at $50\%$ compression. Figure~\ref{fig:speedup}\,(b) reports prefill speedup at 2K and 8K tokens. At $50\%$ state reduction, gains are $\sim15\%$ for $B\geq4$, also at 16K tokens (Appendix~\ref{app:end-to-end-efficiency}). Together, both panels establish end-to-end benefits for both inference regimes. The appendix additionally isolates the sequence mixer to expose kernel-level compute and memory reductions.

\section{Discussion}
The low-rank structure observed in pretrained linear attention states raises the question of whether these models could use their memory capacity more effectively. To explore this question, we analyze how keys and values shape the effective rank of the recurrent state and how this rank structure can amplify the effect of query noise during memory readout. As a consequence, we show that structured pruning of the keys and queries allows discarding at least half of the state while entailing only a minor increase in perplexity. In addition to adapting several pruning strategies, we introduce a method with a provable conditioning bound for the selected joint key-query activations. Our empirical evaluations show that the resulting structured pruning framework can substantially reduce state size at scale. While effective, we note that state reduction is generally not lossless, and competitive downstream performance typically requires RFT after pruning.  Our long-context results suggest that hybrid models are more resilient on recall intensive tasks.

\textbf{Future Work.}
As hybrid architectures become increasingly prominent, understanding the role of linear layers in these models is an interesting direction for future work. Our long-context results suggest that hybrids may be more resilient to state pruning than purely recurrent models, raising the possibility of a division of labor between linear and softmax attention layers. Investigating how these layers contribute to memory and retrieval could clarify the importance of recurrent capacity and inform the design of more efficient hybrid architectures.

\newpage

\section*{Acknowledgements}
PN is supported by the Max Planck ETH Center for Learning Systems. This work was supported in part by the Hector Foundation.

The authors would like to thank Shlomo Libo Feigin, Neehal Tumma, Sajad Movahedi, Timur Carstensen, Patrik Wolf,  Heinrich Campe, and Benedict Armstrong for the interesting discussions and valuable feedback on this work.

\bibliography{refs}
\bibliographystyle{iclr2027_conference}

\newpage
\appendix

\section{Experimental Details}
\label{app:exp-setup}

\textbf{Training.}
We use the {flame}~\citep{yang2025flame} library for recovery fine-tuning (RFT,~\citet{ashkboos2024slicegpt}) and pre-training the 370M parameter models. The latter use $10$ billion tokens of FineWeb. The larger DeltaNet and Gated DeltaNet models are taken from fla-hub~\footnote{\url{https://huggingface.co/collections/fla-hub/deltanet}} and m-a-p~\footnote{\url{https://huggingface.co/m-a-p/1.3B-100B-GatedDeltaNet-pure}}, respectively. The scale experiments use Mamba-Codestral-7B-v0.1~\footnote{\url{https://huggingface.co/mistralai/Mamba-Codestral-7B-v0.1}} and Qwen3.5 9B Base. Mamba-Codestral results are pre-RFT. For Qwen3.5, we prune only the Gated DeltaNet layers. During RFT, the exact-attention layers remain frozen.

We perform RFT using LoRA~\citep{hu2021lora} with rank $r=32$ and $\alpha=64$, using $32k$ samples of FineWeb. During training, we use a batch size of $16$ and sequences of length $2048$. After warming up for 5\% of the total steps, we decay the learning rate from $10^{-4}$ down to $10^{-5}$.  We furthermore unfreeze the causal convolutions and all normalization layers, which make up just a fraction of the total parameters. We furthermore apply knowledge-distillation during RFT using the original model as a teacher.

\textbf{Calibration.}
For the compression methods requiring a calibration set, we use 128 samples of FineWeb unless stated otherwise. DRRQR uses a random subsample of 2.5k keys and queries each. The calibration ablation varies this set from 8 to 512 sequences of length 2048, changes the random seed, and replaces FineWeb with PG19 or the Pile while holding the evaluation fixed.

\section{Additional Material}
Algorithm~\ref{alg:deltanet} describes a typical DeltaNet~\citep{yang2024parallelizing} layer.

\begin{center}
    \begin{minipage}{0.7\textwidth}
    
    \begin{algorithm}[H]
       \caption{A Typical DeltaNet Layer}
       \label{alg:deltanet}
    \begin{algorithmic}[1]
       \small 
       \STATE {\bfseries Input:} Hidden states $\mathbf{X} \in \mathbb{R}^{T, h}$, $H$ heads, Previous state $\mathbf{S}_{0}$ per head
       \STATE {\bfseries Per-head parameters:} $\mathbf{W}_q, \mathbf{W}_k \in \mathbb R^{h, d_k}$, $\mathbf{W}_v \in \mathbb R^{h, d_v}$, $\mathbf{W}_\beta \in \mathbb R^{h, 1}$.
       \STATE {\bfseries Output projection:} $\mathbf{W}_o \in \mathbb{R}^{H d_v, h}$.
       \STATE
       \STATE \textit{// 1. Projections and Local Mixing (per head)}
       \STATE $\mathbf{q} \leftarrow \operatorname{SiLU}(\operatorname{Conv1D}(\mathbf{X}\mathbf{W}_q))$
       \STATE $\mathbf{k} \leftarrow \operatorname{SiLU}(\operatorname{Conv1D}(\mathbf{X}\mathbf{W}_k))$
       \STATE $\mathbf{v} \leftarrow \operatorname{SiLU}(\operatorname{Conv1D}(\mathbf{X}\mathbf{W}_v))$
       \STATE $\beta \leftarrow \sigma(\mathbf{X}\mathbf{W}_\beta)$
       \STATE
       \STATE \textit{// 2. Normalization}
       \STATE $\mathbf{q} \leftarrow \mathbf{q} / \|\mathbf{q}\|_2$
       \STATE $\mathbf{k} \leftarrow \mathbf{k} / \|\mathbf{k}\|_2$
       \STATE
       \STATE \textit{// 3. Delta Rule}
       \FOR{$t = 1$ {\bfseries to} $T$}
           \STATE $\mathbf{S}_t \leftarrow \mathbf{S}_{t-1} (\mathbf{I} - \beta_t \mathbf{k}_t \mathbf{k}_t^\top) + \beta_t \mathbf{v}_t \mathbf{k}_t^\top$
           \STATE $\mathbf{o}_t \leftarrow \mathbf{S}_t \mathbf{q}_t$
       \ENDFOR
       \STATE Concatenate heads to form $\mathbf{O} \in \mathbb{R}^{T, H d_v}$
       \STATE
       \STATE \textit{// 4. Output Projection}
        \STATE $\mathbf{O} \leftarrow \operatorname{RMSNorm}(\mathbf{O})$
       \STATE $\mathbf{Y} \leftarrow \mathbf{O}\mathbf{W}_o$
       \STATE
       \STATE {\bfseries Return:} $\mathbf{Y}$
    \end{algorithmic}
    \end{algorithm}
    
    \end{minipage}
\end{center}

\subsection{Structured Pruning from the Test-Time Regression Perspective}
\label{app:ttr}
Linear attention, and specifically DeltaNet, is usually interpreted as performing gradient descent on a linear regression objective~\citep{schlag2021linear,yang2024parallelizing}, where the hidden state $\mathbf{S}_t$ acts as the fast weight matrix trained to map keys $\mathbf{k}_t$ (inputs) to values $\mathbf{v}_t$ (targets). In this framework, our proposed structured pruning strategy admits another interpretation, functioning as a {feature selection} step applied to the input of the online learner. By restricting the input features to the subspace spanned by the transformation matrix $\bT$, we effectively constrain the hypothesis class of the regression. This forces the fast weights to ignore the null space of $\mathbf{T}$, acting as a form of regularization.

\begin{center}
    \begin{minipage}{0.7\textwidth}
\begin{algorithm}[H]
   \caption{Deep Rank Revealing QR (DRRQR). Adapted from Algorithm 4 of \citet{gu1996efficient}.}
   \label{alg:drrqr}
\begin{algorithmic}[1]
   \STATE {\bfseries Input:} Activation Matrix $\mathbf{A} \in \mathbb{R}^{2N, d_k}$, Target Rank $d_k'$, Tolerance $f \ge 1$
   \STATE {\bfseries Output:} Selected Indices $\mathcal{I}$
   \STATE
   \STATE \textit{// 1. Initialization: QR with Column Pivoting}
   \STATE $[\mathbf{Q}, \mathbf{R}, \bm{\pi}] \leftarrow \operatorname{QRCP}(\mathbf{A})$, with column index vector $\bm{\pi}$
   \STATE Partition $\mathbf{R} = \begin{pmatrix} \mathbf{A}_{d_k'} & \mathbf{B}_{d_k'} \\ \mathbf{0} & \mathbf{C}_{d_k'} \end{pmatrix}$, where $\mathbf{A}_{d_k'} \in \mathbb{R}^{d_k', d_k'}$
   \STATE Initialize $\omega_i(\mathbf{A}_{d_k'}) = 1/\|(\mathbf{A}_{d_k'}^{-1})_{i,:}\|_2$ and $\gamma_j(\mathbf{C}_{d_k'}) = \|(\mathbf{C}_{d_k'})_{:,j}\|_2$
   \STATE
   \STATE \textit{// 2. Iterative Swapping for Stability}
   \WHILE{True}
       \STATE Compute $\mathbf{U} = \mathbf{A}_{d_k'}^{-1}\mathbf{B}_{d_k'}$
       \STATE
       \STATE \textit{// Calculate swap gain metric $\rho_{ij}$ for all pairs}
       \STATE \textit{// Checks if $U_{ij}$ is large or if residual $\gamma_j$ is large relative to basis $\omega_i$}
       \STATE $\rho_{ij} \leftarrow \sqrt{ |U_{ij}|^2 + (\gamma_j(\mathbf{C}_{d_k'}) / \omega_i(\mathbf{A}_{d_k'}))^2 }$
       \STATE
       \STATE Let $(i^*, j^*) = \operatorname{argmax}_{i,j} \rho_{ij}$
       \IF{$\rho_{i^*j^*} \le f$}
           \STATE \textbf{break}  // Strong RRQR condition met
       \ENDIF
       \STATE
       \STATE Swap columns $i^*$ and $d_k'+j^*$ of $\mathbf{R}$
       \STATE Swap entries $\pi_{i^*}$ and $\pi_{d_k'+j^*}$
       \STATE Retriangularize $\mathbf{R}$ and update $\omega(\mathbf{A}_{d_k'})$ and $\gamma(\mathbf{C}_{d_k'})$
   \ENDWHILE
   \STATE
   \STATE $\mathcal{I} \leftarrow \{\pi_1,\ldots,\pi_{d_k'}\}$
   \STATE {\bfseries Return:} $\mathcal{I}$
\end{algorithmic}
\end{algorithm}
\end{minipage}
\end{center}

\subsection{Details on Rank Revealing QR (RRQR)}
\label{app:drrqr_details}

In this section, we detail the \emph{Strong Rank-Revealing QR} (RRQR) algorithm (see Algorithm~\ref{alg:drrqr}) proposed by \citet{gu1996efficient}, which forms the basis of our DRRQR pruning method.

\paragraph{Mathematical Formulation.}
Let $\mathbf{M} \in \mathbb{R}^{m, n}$ be the input matrix of concatenated keys and queries (in the main text, $m=2N$ and $n=d_k$). We seek a permutation $\mathbf{\Pi}$ and a target rank $k$ (in our main text denoted as $d_k'$) such that the QR factorization
\begin{equation}
    \mathbf{M}\mathbf{\Pi} = \mathbf{Q} \begin{pmatrix} \mathbf{A}_k & \mathbf{B}_k \\ \mathbf{0} & \mathbf{C}_k \end{pmatrix}
\end{equation}
satisfies specific bounds on the singular values of the leading principal submatrix $\mathbf{A}_k \in \mathbb{R}^{k, k}$ and the trailing submatrix $\mathbf{C}_k \in \mathbb{R}^{(m-k), (n-k)}$. Assume $\operatorname{rank}(\mathbf{M})\geq k$. For tolerance $f\geq1$, the stopping condition below guarantees~\citep[Theorem 3.2]{gu1996efficient}
\begin{equation*}
    \sigma_i(\mathbf{A}_k) \geq \frac{\sigma_i(\mathbf{M})}{c_f} \quad (1\leq i\leq k), \qquad
    \|\mathbf{C}_k\|_2 \leq c_f\sigma_{k+1}(\mathbf{M}),
\end{equation*}
where $c_f=\sqrt{1+f^2k(n-k)}$. The selected columns share the singular values of $\mathbf{A}_k$. Since $\sigma_1(\mathbf{A}_k)\leq\sigma_1(\mathbf{M})$, their condition number satisfies
\begin{equation*}
    \kappa(\mathbf{A}_k) \leq c_f\frac{\sigma_1(\mathbf{M})}{\sigma_k(\mathbf{M})}.
\end{equation*}

\paragraph{Geometric Intuition.}
The volume of the parallelotope formed by the selected columns is $|\det(\mathbf{A}_k)| = \prod_{i=1}^k \sigma_i(\mathbf{A}_k)$. Each accepted column swap increases this volume. The procedure stops when no single-column swap increases it by a factor greater than $f$, ensuring the bounds above.

\paragraph{The Swap Criterion.}
Let $\mathbf{\Pi}$ be the current column permutation matrix. To determine if swapping the $i$-th column of the basis (where $1 \le i \le k$) with the $j$-th column of the residual (where $1 \le j \le n-k$) improves the factorization, we analyze the ratio of the new determinant to the current determinant.

\citet{gu1996efficient} derive an efficiently computable metric for this ratio. Let $\mathbf{U} = \mathbf{A}_k^{-1}\mathbf{B}_k$. We define:
\begin{itemize}
    \item $\gamma_j(\mathbf{C}_k) = \| (\mathbf{C}_k)_{:,j} \|_2$: The $\ell_2$-norm of the $j$-th column of the residual block.
    \item $\omega_i(\mathbf{A}_k) = 1 / \| (\mathbf{A}_k^{-1})_{i,:} \|_2$: The reciprocal of the $\ell_2$-norm of the $i$-th row of the inverse basis.
\end{itemize}

By Lemma 3.1 of \citet{gu1996efficient}, the potential gain $\rho_{ij}$ from swapping basis column $i$ with candidate column $j$ is given by:
\begin{equation}
    \rho_{ij} = \sqrt{ |U_{ij}|^2 + \left( \frac{\gamma_j(\mathbf{C}_k)}{\omega_i(\mathbf{A}_k)} \right)^2 }.
\end{equation}
If $\rho_{ij} > f$ for a chosen tolerance factor $f \ge 1$, swapping these columns guarantees an increase in $|\det(\mathbf{A}_k)|$ by a factor of at least $\rho_{ij}$. The first term, $|U_{ij}|^2$, captures the linear dependence of the candidate vector on the current basis vector, while the second term captures the magnitude of the candidate relative to the stability of the basis vector.

\paragraph{Algorithm and Update Rules.}
The DRRQR procedure (Algorithm~\ref{alg:drrqr}) proceeds as follows:
\begin{enumerate}
    \item \textbf{Initialization:} Compute an initial factorization using standard QRCP. This provides a baseline $\mathbf{\Pi}$, $\mathbf{A}_k$, $\mathbf{B}_k$, and $\mathbf{C}_k$.
    \item \textbf{Identification:} Search for a pair of indices $(i, j)$ such that $\rho_{ij} > f$. Efficient search strategies maximize over $j$ for fixed $i$, or simply identify the first valid pair.
    \item \textbf{Update:} If a valid pair is found:
    \begin{enumerate}
        \item Swap columns $i$ and $j+k$ of $\mathbf{R}$ and update the permutation $\mathbf{\Pi}$ accordingly.
        \item Retriangularize the matrix $\mathbf{R}$ using Givens rotations to restore the upper-triangular structure of $\mathbf{A}_k$.
        \item Update the auxiliary vectors $\omega(\mathbf{A}_k)$ and $\gamma(\mathbf{C}_k)$ using the formulas provided in Section 4 of \citet{gu1996efficient}.
    \end{enumerate}
    \item \textbf{Termination:} The process repeats until no pair $(i, j)$ satisfies $\rho_{ij} > f$, ensuring the matrix satisfies the strong rank-revealing condition.
\end{enumerate}

\subsection{Adapting Convolutions to General Rotations}
\label{appendix:conv-gen-orth}

Handling the more general case of (semi-) orthogonal, non-axis-aligned transformations is more intricate than the axis-aligned one (see Section~\ref{sec:axis-aligned}). If one wishes to employ general orthogonal transformations $\mathbf{T} \in O(d_k)$ (such as those derived from PCA) to prune the sequence mixer, the convolution layers must be adapted.

Depthwise convolutions operate independently on each channel. Rotating the input space via $\mathbf{T}$ generally turns the exact filter matrices into dense matrices. To retain depthwise convolutions, we approximate these rotated filters with diagonal matrices.

\subsection{Optimal Diagonal Filter Approximation}
We choose the per-channel filters that minimize the sum of squared Frobenius distances to the exact rotated filter matrices.

\begin{proposition}[Optimal Diagonal Filter Approximation]
\label{proposition:optimal-conv}
    Let $\mathbf{T} \in O(d)$ be an orthogonal transformation and let $\mathbf{W} = [\mathbf{w}^{(0)},\ldots,\mathbf{w}^{(l-1)}] \in \mathbb{R}^{d,l}$ contain the original convolution filters for $d$ channels and kernel size $l$. Writing $\mathbf{W}' = [\mathbf{w}'^{(0)},\ldots,\mathbf{w}'^{(l-1)}]$, the optimization problem
    \begin{equation*}
        \min_{\mathbf{W}'\in\mathbb{R}^{d,l}} \sum_{j=0}^{l-1}
        \left\| \operatorname{diag}(\mathbf{w}'^{(j)}) - \mathbf{T}\operatorname{diag}(\mathbf{w}^{(j)})\mathbf{T}^\top \right\|_F^2
    \end{equation*}
    has the unique solution
    \begin{equation*}
        \mathbf{W}' = (\mathbf{T} \odot \mathbf{T}) \mathbf{W},
    \end{equation*}
    where $\odot$ denotes the Hadamard (element-wise) product.
\end{proposition}

\begin{proof}
    For rotated inputs $\tilde{\mathbf{x}}_t = \mathbf{T}\mathbf{x}_t$, the exact rotated convolution is
    \begin{equation*}
        \mathbf{T}(\mathbf{W}\ast\mathbf{x})_t
        = \sum_{j=0}^{l-1} \mathbf{M}^{(j)}\tilde{\mathbf{x}}_{t-j},
        \qquad \mathbf{M}^{(j)} \coloneqq \mathbf{T}\operatorname{diag}(\mathbf{w}^{(j)})\mathbf{T}^\top.
    \end{equation*}
    For each lag $j$, the squared Frobenius error decomposes as
    \begin{equation*}
        \left\| \operatorname{diag}(\mathbf{w}'^{(j)}) - \mathbf{M}^{(j)} \right\|_F^2
        = \sum_{a=1}^d \left(w'^{(j)}_a - M^{(j)}_{aa}\right)^2
        + \sum_{a\neq b} \left(M^{(j)}_{ab}\right)^2.
    \end{equation*}
    The second sum is independent of $\mathbf{W}'$. Hence the unique minimizer satisfies
    \begin{equation*}
        w'^{(j)}_a = M^{(j)}_{aa}
        = \sum_{b=1}^d T_{ab}^2 w^{(j)}_b
        = \bigl((\mathbf{T}\odot\mathbf{T})\mathbf{w}^{(j)}\bigr)_a.
    \end{equation*}
    Stacking the coefficients across all lags gives $\mathbf{W}' = (\mathbf{T}\odot\mathbf{T})\mathbf{W}$.
\end{proof}

Each adapted kernel is a weighted average of the original kernels, with nonnegative weights $T_{ab}^2$ that sum to one. This optimality concerns filter approximation in Frobenius norm and requires no assumptions on the activation distribution. We evaluate its effect on model outputs empirically in Appendix~\ref{appendix:results-pca}.

\subsection{Shared Convolutions}
An alternative approach to facilitate general rotations is to constrain the model architecture itself. If we enforce that the convolution filters are shared across all channels within a head, the convolution operation becomes a scalar multiplication at each lag, which commutes with any linear transformation.

\begin{lemma}[Commutativity of Shared Convolutions]
\label{lemma:shared-conv}
    Let $\mathbf{w} \in \mathbb{R}^l$ be a filter shared across all $d$ channels, such that the convolution kernel matrix $\mathbf{W} \in \mathbb{R}^{d, l}$ has identical rows $\mathbf{W}_{k, :} = \mathbf{w}$ for all $k$. For any linear transformation matrix $\mathbf{T} \in \mathbb{R}^{d, d}$ (including orthogonal rotations), the convolution commutes with the transformation:
    \begin{equation*}
        \mathbf{W} \ast (\mathbf{T}\mathbf{x}) = \mathbf{T}(\mathbf{W} \ast \mathbf{x}).
    \end{equation*}
\end{lemma}

\begin{proof}
    For a shared filter, the convolution operation on the vector $\mathbf{x}_t$ can be written as a scalar convolution applied element-wise: $(\mathbf{W} \ast \mathbf{x})_t = \sum_{j=0}^{l-1} w_j \mathbf{x}_{t-j}$.
    Applying the transformation $\mathbf{T}$ first:
    \begin{equation*}
        \mathbf{W} \ast (\mathbf{T}\mathbf{x})_t = \sum_{j=0}^{l-1} w_j (\mathbf{T}\mathbf{x}_{t-j}) = \mathbf{T} \left( \sum_{j=0}^{l-1} w_j \mathbf{x}_{t-j} \right) = \mathbf{T}(\mathbf{W} \ast \mathbf{x})_t.
    \end{equation*}
\end{proof}

Shared convolutions commute with rotations without filter adaptation. This equivariance does not extend to the full layer when SiLU is present, since SiLU generally does not commute with dense rotations.

\section{Additional Experimental Results}

\subsection{Per-Head Rank Utilization}
\label{app:per-head-rank}
We compute the rank utilization of each (layer, head) pair for DeltaNet 370M at 75\% compression (pre-RFT), averaged over the sequence and batch axes for a sample from FineWeb. Table~\ref{tab:per-head-rank} summarizes the results across all (layer, head) pairs. Three findings emerge. First, heads are heterogeneous: rank utilization varies by an order of magnitude across heads, suggesting that the model allocates memory capacity unevenly. Second, DRRQR lifts rank utilization consistently across heads, while magnitude-based methods such as L1 exhibit weaker correlation with the baseline per-head structure ($\rho=0.68$). Third, every method increases rank utilization at the reduced capacity, consistent with the effective-rank bound in Equation~\eqref{eq:effective-rank-bound}.

\begin{table}[ht]
\centering
\caption{Per-head rank utilization statistics across all (layer, head) pairs for DeltaNet 370M at 75\% compression (pre-RFT). $\rho$ denotes the Pearson correlation between baseline and post-pruning per-head rank utilization.}
\label{tab:per-head-rank}
\begin{small}
\begin{tabular}{l c c c c c}
\toprule
\textbf{Method} & \textbf{Mean} & \textbf{Std} & \textbf{Min} & \textbf{Max} & $\bm{\rho}$ \\
\midrule
Baseline ($d_k\!=\!128$) & 0.053 & 0.028 & 0.01 & 0.16 & 1.00 \\
DRRQR ($d_k\!=\!32$) & 0.170 & 0.060 & 0.03 & 0.40 & 0.96 \\
L1 ($d_k\!=\!32$) & 0.124 & 0.036 & 0.03 & 0.30 & 0.68 \\
\bottomrule
\end{tabular}
\end{small}
\end{table}

\subsection{DRRQR, Grad, and Calibration Robustness}
\label{app:drrqr-robustness}
DRRQR and Grad share a median of 26 of the 32 retained channels per head on DeltaNet 370M at 75\% compression, whereas random selection shares 8. Thus, the loss gradient implicitly identifies much of the same rank-preserving subspace as DRRQR.

Table~\ref{tab:selected-subspace} reports per-head medians for the retained columns of the joint $[\bQ;\bK]$ activation matrix. DRRQR raises $\sigma_{\min}$ while preserving $\sigma_{\max}$. L1 and S-Wanda obtain smaller condition numbers only by removing high-energy directions. The condition-number column is the median of per-head ratios and therefore need not equal the ratio of the displayed medians.

\begin{table}[ht]
\centering
\caption{Selected key/query subspace on DeltaNet 370M at 75\% compression.}
\label{tab:selected-subspace}
\begin{small}
\begin{tabular}{lccc}
\toprule
Selection & $\sigma_{\min}$ & $\sigma_{\max}$ & $\kappa$ \\
\midrule
Baseline ($d_k=128$) & 0.10 & 49.4 & 529 \\
DRRQR & 0.96 & 49.4 & 65 \\
Grad & 0.62 & 48.5 & 89 \\
S-Wanda & 0.65 & 22.5 & 42 \\
L1 & 0.64 & 15.3 & 32 \\
Rand & 0.16 & 13.3 & 116 \\
\bottomrule
\end{tabular}
\end{small}
\end{table}

We next vary calibration sample count, seed, and corpus, always using sequences of 2048 tokens and evaluating before RFT. Table~\ref{tab:calibration-robustness} shows that DRRQR is effectively invariant across all three axes. Grad is stable in-domain but substantially less robust to distribution shift because its score depends on the calibration loss.

\begin{table}[ht]
\centering
\caption{Calibration sensitivity on DeltaNet 370M at 75\% compression (WikiText-2 perplexity, pre-RFT).}
\label{tab:calibration-robustness}
\begin{small}
\begin{tabular}{lcc}
\toprule
Calibration & DRRQR & Grad \\
\midrule
FineWeb, $n=8$ & 41.7 & 44.9 \\
FineWeb, $n=32$ & 41.6 & 45.0 \\
FineWeb, $n=128$ & 41.6 & 44.9 \\
FineWeb, $n=512$ & 41.6 & 44.8 \\
FineWeb, $n=128$, seed 111 & 41.6 & 44.8 \\
FineWeb, $n=128$, seed 1234 & 41.5 & 44.6 \\
PG19, $n=128$ & 41.5 & 53.8 \\
Pile, $n=128$ & 41.7 & 51.0 \\
\bottomrule
\end{tabular}
\end{small}
\end{table}

\subsection{Generalization to Mamba-2}
\label{app:mamba2}
To assess generality beyond DeltaNet, we apply our framework to Mamba-2~\citep{dao2024transformers}, where $\bB$ and $\bC$ play the roles of keys and queries. Table~\ref{tab:mamba2} gives the full zero-shot task breakdown for a 370M model at 75\% reduction. Table~\ref{tab:mamba2-scale} adds a 50\% setting and Mistral's Mamba-2 checkpoint Mamba-Codestral-7B-v0.1. All results are pre-RFT and therefore isolate channel-selection quality. Grad and DRRQR degrade substantially more gracefully than magnitude pruning at 7B scale, although Grad is strongest in this setting.

\begin{table}[ht]
\centering
\caption{Mamba-2 370M at 75\% state reduction ($128 \to 32$). Zero-shot accuracy and WikiText-2 perplexity (pre-RFT).}
\label{tab:mamba2}
\begin{small}
\begin{sc}
\begin{tabular}{lcccccc}
\toprule
Method & ARC-E & ARC-C & HellaS & WinoG & PIQA & Wiki $\downarrow$ \\
\midrule
Baseline & 49.8 & 27.1 & 40.2 & 52.4 & 65.7 & 29.2 \\
Grad & 47.1 & 26.9 & 36.3 & 50.0 & 63.8 & 43.4 \\
DRRQR & 45.0 & 25.7 & 36.3 & 50.9 & 63.8 & 46.3 \\
S-Wanda & 41.6 & 24.8 & 35.3 & 52.6 & 62.6 & 52.7 \\
L1 & 40.3 & 24.9 & 34.6 & 51.7 & 62.0 & 54.3 \\
Rand & 41.0 & 25.0 & 34.7 & 50.6 & 61.5 & 52.8 \\
\bottomrule
\end{tabular}
\end{sc}
\end{small}
\end{table}

\begin{table}[ht]
\centering
\caption{Mamba-2 results across scale before RFT. Zero-shot is the mean accuracy over ARC-Easy, ARC-Challenge, HellaSwag, WinoGrande, PIQA, and LAMBADA.}
\label{tab:mamba2-scale}
\begin{small}
\begin{tabular}{llccc}
\toprule
Model & Method & Compression & Wiki $\downarrow$ & Zero-shot $\uparrow$ \\
\midrule
Mamba-2 370M & Baseline & 0\% & 29.2 & 44.5 \\
& Grad & 50\% & 34.1 & 42.5 \\
& DRRQR & 50\% & 35.6 & 41.0 \\
& L1 & 50\% & 35.4 & 40.8 \\
\midrule
Mamba-Codestral 7B & Baseline & 0\% & 14.5 & 64.4 \\
& Grad & 50\% & 16.7 & 59.9 \\
& DRRQR & 50\% & 17.6 & 57.1 \\
& L1 & 50\% & 25.7 & 48.8 \\
& Grad & 75\% & 22.4 & 52.4 \\
& DRRQR & 75\% & 26.7 & 48.8 \\
& L1 & 75\% & 58.1 & 40.8 \\
\bottomrule
\end{tabular}
\end{small}
\end{table}

\begin{table}[ht]
\centering
\caption{Zero-shot performance of Mamba-Codestral-7B before RFT, evaluated using the {lm-eval-harness}~\citep{eval-harness}, given different compression ratios.}
\label{tab:codestral-full}
\resizebox{\columnwidth}{!}{%
\begin{tabular}{ll|cc|cccccc|c}
\toprule
\thead[t]{Method} & \thead[t]{Compr.} & \makecell[t]{Wiki.~$\downarrow$\\ppl} & \makecell[t]{LMB.~$\downarrow$\\ppl} & \makecell[t]{ARC-e~$\uparrow$\\acc\_n} & \makecell[t]{ARC-c~$\uparrow$\\acc\_n} & \makecell[t]{Hella.~$\uparrow$\\acc\_n} & \makecell[t]{Wino.~$\uparrow$\\acc} & \makecell[t]{PIQA~$\uparrow$\\acc\_n} & \makecell[t]{LMB.~$\uparrow$\\acc} & \makecell[t]{Avg~$\uparrow$} \\
\midrule
Grad & 50\% & \textbf{16.7} & \textbf{12.7} & 67.3 & 43.1 & 68.1 & 58.7 & 75.0 & 47.1 & \textbf{59.9} \\
DRRQR & 50\% & \underline{17.6} & \underline{28.7} & 66.3 & 42.2 & 66.3 & 58.5 & 75.0 & 34.3 & \underline{57.1} \\
L1 & 50\% & 25.7 & 73.7 & 55.5 & 30.9 & 56.7 & 53.9 & 71.5 & 23.9 & 48.8 \\
\addlinespace
Grad & 75\% & \textbf{22.4} & \textbf{69.9} & 63.5 & 38.1 & 62.3 & 53.2 & 73.3 & 24.1 & \textbf{52.4} \\
DRRQR & 75\% & \underline{26.7} & \underline{278.0} & 60.7 & 37.3 & 59.3 & 52.8 & 72.2 & 10.6 & \underline{48.8} \\
L1 & 75\% & 58.1 & 1157.9 & 46.1 & 27.5 & 46.1 & 52.2 & 67.3 & 5.5 & 40.8 \\
\midrule
\textbf{Baseline} & 0\% & 14.5 & 6.7 & 71.5 & 44.8 & 70.5 & 64.2 & 77.0 & 58.2 & 64.4 \\
\bottomrule
\end{tabular}%
}
\end{table}

\subsection{Long-Context Retrieval and Hybrid Models}
\label{app:long-context}
We evaluate post-RFT Gated DeltaNet 1.3B checkpoints on RULER~\citep{hsieh2024ruler}, including single-needle retrieval, and on SQuAD- and SWDE-completion. The models are trained at sequence length 2K and evaluated at both 2K and 4K. Table~\ref{tab:long-context-gdn} shows that 20--30\% reduction incurs a moderate cost, whereas at 50\% the finite recurrent capacity begins to bind.

\begin{table}[ht]
\centering
\caption{Long-context and recall results for Gated DeltaNet 1.3B after RFT.}
\label{tab:long-context-gdn}
\small
\setlength{\tabcolsep}{3.5pt}
\begin{tabular}{lcccccc}
\toprule
Model & \makecell{RULER~$\uparrow$\\2K} & \makecell{RULER~$\uparrow$\\4K} & \makecell{NIAH~$\uparrow$\\2K} & \makecell{NIAH~$\uparrow$\\4K} & SQuAD~$\uparrow$ & SWDE~$\uparrow$ \\
\midrule
Baseline & 47.7 & 38.8 & 93.1 & 81.3 & 40.5 & 51.2 \\
DRRQR 20\% & 44.6 & 35.4 & 91.1 & 75.9 & 39.9 & 45.5 \\
Grad 20\% & 46.1 & 34.0 & 92.3 & 77.1 & 39.8 & 47.1 \\
DRRQR 30\% & 43.1 & 32.6 & 89.2 & 72.5 & 39.3 & 43.8 \\
Grad 30\% & 45.0 & 33.9 & 90.7 & 75.6 & 39.3 & 45.8 \\
DRRQR 50\% & 37.9 & 27.5 & 81.9 & 61.9 & 36.4 & 37.0 \\
Grad 50\% & 38.4 & 28.1 & 83.7 & 65.7 & 37.2 & 37.7 \\
\bottomrule
\end{tabular}
\end{table}

We repeat the 50\% experiment on Qwen3.5 9B, a hybrid with softmax attention and Gated DeltaNet layers. RFT updates only the recurrent components. Softmax attention layers remain unchanged. Table~\ref{tab:qwen-long-context} shows that the hybrid preserves long-context retrieval at the same compression ratio that degrades the purely recurrent model. This supports the hypothesis that exact-attention layers carry recall.

\begin{table}[ht]
\centering
\caption{Qwen3.5 9B long-context results after RFT of only the Gated DeltaNet components.}
\label{tab:qwen-long-context}
\small
\setlength{\tabcolsep}{3.5pt}
\begin{tabular}{lcccccc}
\toprule
Configuration & \makecell{RULER~$\uparrow$\\2K} & \makecell{RULER~$\uparrow$\\4K} & \makecell{NIAH~$\uparrow$\\2K} & \makecell{NIAH~$\uparrow$\\4K} & SQuAD~$\uparrow$ & SWDE~$\uparrow$ \\
\midrule
Baseline & 94.7 & 93.9 & 100.0 & 100.0 & 65.3 & 92.7 \\
DRRQR 50\% & 94.3 & 93.6 & 100.0 & 100.0 & 66.0 & 92.3 \\
\bottomrule
\end{tabular}
\end{table}

The pre-RFT results in Table~\ref{tab:qwen-prerft} show that the rank-aware selector is important: DRRQR preserves most zero-shot accuracy at 50\%, whereas L1 collapses. We have not combined state reduction directly with intra-layer token routing or sparse caches such as STILL and LoLA. Those methods act on the token axis and remain complementary future work.

\begin{table}[ht]
\centering
\caption{Qwen3.5 9B before RFT.}
\label{tab:qwen-prerft}
\begin{small}
\begin{tabular}{lccc}
\toprule
Configuration & Wiki $\downarrow$ & LAMBADA $\downarrow$ & Zero-shot $\uparrow$ \\
\midrule
Baseline & 9.0 & 3.8 & 73.3 \\
DRRQR 25\% & 9.7 & 4.3 & 73.0 \\
DRRQR 50\% & 11.6 & 5.7 & 71.4 \\
L1 25\% & 18.8 & 7.1 & 61.3 \\
L1 50\% & 38.0 & 33.4 & 50.9 \\
\bottomrule
\end{tabular}
\end{small}
\end{table}

\subsection{Singular Value Spectrum}
\label{app:singular-value-spectrum}
Figure~\ref{fig:histogram} illustrates the singular value spectrum of a randomly selected head's hidden state, aggregated across all tokens. We compare the uncompressed baseline against DRRQR, Grad and L1 methods at a 75\% compression ratio (prior to recovery fine-tuning). Notably, while the uncompressed model exhibits a heavy tail of singular values, the compressed models display a much sharper spectral truncation. For some heads, we observe the emergence of a spectral gap. This gap can also develop during recovery fine-tuning (see Figure~\ref{fig:histogram_pre_vs_post}). We consider studying this phenomenon an interesting branch for future research.

\begin{figure}[ht]
    \centering
    \input{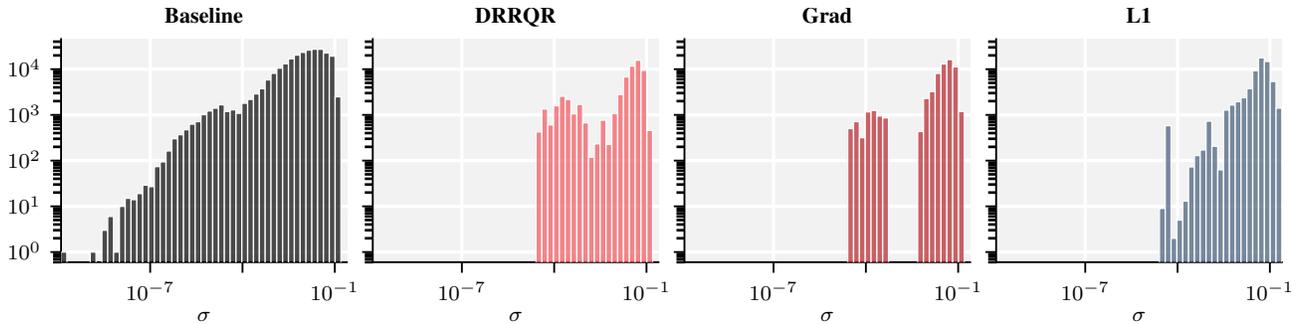}%
    \caption{Singular value spectrum of a DeltaNet 370M head's hidden state (Fineweb, $T=2048$, first 128 tokens skipped). We compare the uncompressed \textit{Baseline} against compressed models at a 75\% compression ratio (pre-RFT).}
    \label{fig:histogram}
\end{figure}

\begin{figure}[ht]
    \centering
    \input{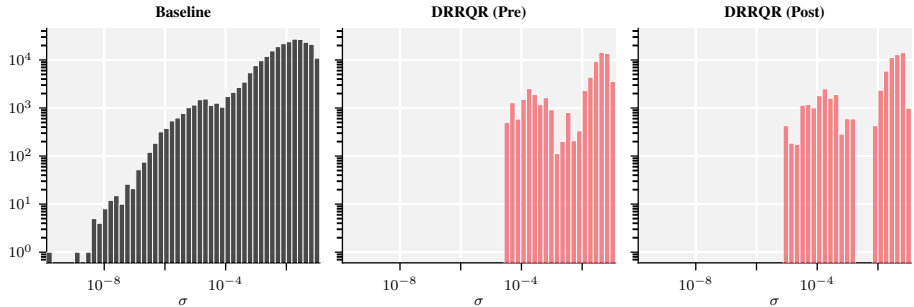}%
    \caption{Impact of recovery fine-tuning (RFT) on the singular value spectrum ($T=2048$, first 128 tokens skipped). We compare the uncompressed \textit{Baseline} with \textit{DRRQR} (75\% compression) both before (\textit{pre}) and after (\textit{post}) RFT.}
    \label{fig:histogram_pre_vs_post}
\end{figure}

\label{app:more-results}
\subsection{On PCA and Convolutions}
\label{appendix:results-pca}

In this section, we provide extended experimental results on the PCA-based pruning strategy. We furthermore analyze the impact of depthwise convolutions on the applicability and performance of semi-orthogonal structured pruning.

As mentioned in Section~\ref{sec:method}, depthwise convolutions are generally not invariant under orthogonal transformations. When pruning via PCA, the features are rotated, causing a misalignment with the per-channel convolution filters. Appendix~\ref{appendix:conv-gen-orth} considers two approaches: {shared convolutions}, which commute with rotations, and {filter mixing}, which gives the closest diagonal approximation to the rotated filter matrices in Frobenius norm (Proposition~\ref{proposition:optimal-conv}).

\subsubsection{Shared Convolutions}
To understand the impact of those two approaches, we train DeltaNet 370M variants using {shared convolutions}, where the convolution filter is tied across all channels within a head. As shown in Lemma~\ref{lemma:shared-conv}, the shared convolution operator is equivariant under rotations.

\begin{table}[ht]
\centering
\caption{DeltaNet 370M models trained on $10B$ tokens evaluated on commonsense zero-shot reasoning tasks. ``Shared Conv'' indicates whether convolution filters are shared across heads and/or channels (for example, $\times$\checkmark means that filters are shared across channels inside of a head but not across heads).}
\label{tab:shared-non-shared-comp-full}
\begin{small}
\setlength{\tabcolsep}{5.5pt}
\begin{tabular}{l|cc|cccccc|c}
\toprule
\thead[t]{Shared\\Conv} & \makecell[t]{Wiki.~$\downarrow$\\ppl} & \makecell[t]{LMB.~$\downarrow$\\ppl} & \makecell[t]{ARC-e~$\uparrow$\\acc\_n} & \makecell[t]{ARC-c~$\uparrow$\\acc\_n} & \makecell[t]{Hella.~$\uparrow$\\acc\_n} & \makecell[t]{Wino.~$\uparrow$\\acc} & \makecell[t]{PIQA~$\uparrow$\\acc\_n} & \makecell[t]{LMB.~$\uparrow$\\acc} & \makecell[t]{Avg~$\uparrow$} \\
\midrule
$\times$$\times$ & 29.8 & 37.0 & 51.1 & 27.6 & 38.1 & 52.2 & 65.0 & 31.3 & 44.2 \\
$\times$\checkmark & 29.5 & 40.4 & 49.7 & 27.1 & 38.3 & 52.4 & 64.7 & 30.4 & 43.8 \\
\checkmark\checkmark & 29.3 & 40.4 & 49.9 & 26.4 & 37.8 & 49.9 & 65.2 & 29.9 & 43.2 \\
\bottomrule
\end{tabular}
\end{small}
\end{table}

\paragraph{Filter Similarity.}
Figure~\ref{fig:filters_dn} visualizes the learned filters of a standard DeltaNet (non-shared). We observe high similarity between filters within specific heads, suggesting that the model naturally learns to share dynamics across channels. This could serve as a justification for sharing filters among channels inside each head.


\paragraph{Performance of shared convolutions.}
We include results on pre-trained DeltaNet models with and without shared convolutions in Table~\ref{tab:shared-non-shared-comp-full}. It shows that, while models with shared convolutions are competitive (especially when just shared inside of a head and not across heads), there is a slight drop-off.

\begin{table}[ht]
\centering
\caption{Comparison of post-compression (pre-RFT) perplexity on WikiText-2 for DeltaNet 370M. We compare standard (non-shared) vs. shared convolutions under different pruning strategies. "PCA (adv.)" removes the highest variance components (keeping noise), while "PCA (prop.)" removes the lowest variance components (keeping signal).}
\label{table:adv_vs_non_adv}
\begin{small}
\begin{tabular}{ll|c|c}
\toprule
\textbf{Comp.} & \textbf{Method} & \textbf{Non Shared} & \textbf{Shared} \\
\midrule
\multirow[t]{3}{*}{75\%} & PCA (adv.) & $3.1 \cdot 10^5$ & $2.1 \cdot 10^5$ \\
 & PCA (prop.) & $3.3 \cdot 10^5$ & ${184.48}$ \\
 & Grad & ${46.7}$ & $69.78$ \\
\midrule
\multirow[t]{3}{*}{50\%} & PCA (adv.) & $2.4 \cdot 10^5$ & $6.1 \cdot 10^4$ \\
 & PCA (prop.) & $3.2 \cdot 10^5$ & ${155.76}$ \\
 & Grad & ${31.75}$ & $33.73$ \\
\midrule
\textbf{0\%} & -- & 29.83 & 29.55 \\
\bottomrule
\end{tabular}
\end{small}
\end{table}

\subsubsection{Impact of Convolutions on PCA-based Pruning}
Towards quantifying the two approaches of handling the per-channel convolutions, we compare a "Proper PCA" method, where we retain the principal components with the highest variance, against an "Adversarial PCA" baseline, where we deliberately retain the dimensions with the lowest variance. In a system robust to rotation, proper PCA should outperform the adversarial baseline.

Table~\ref{table:adv_vs_non_adv} presents the perplexity on WikiText-2 for DeltaNet 370M (pre-RFT). For the shared convolutions, proper PCA generally yields way better perplexity than the adversarial baseline. However, in the standard, non-shared case, the filter averaging does not suffice to make up for the misalignment post-transformation. Furthermore, the Grad pruning method still outperforms PCA even when using shared convolutions.

\subsection{On Coupled Selection}
\label{app:ablation-kq}
We compare selecting indices based on (i) the sum of scores derived from both projections ($\mathbf{K}, \mathbf{Q}$), (ii) keys only ($\mathbf{K}$), and (iii) queries only ($\mathbf{Q}$). Results are reported in Table~\ref{tab:ablation_kq_targets}.

\begin{table}[t]
\centering
\caption{Consolidated ablation study on DeltaNet and Gated DeltaNet models (370M and 1.3B) measuring WikiText-2 perplexity. We compare the impact of picking just keys ($\mathbf{K}$), just queries ($\mathbf{Q}$), or both ($\mathbf{K}, \mathbf{Q}$) for feature selection at a 50\% compression ratio.}
\label{tab:ablation_kq_targets}
\resizebox{\textwidth}{!}{%
\begin{tabular}{l|ccc|ccc|ccc|ccc}
\toprule
 & \multicolumn{6}{c|}{\textbf{DeltaNet}} & \multicolumn{6}{c}{\textbf{Gated DeltaNet}} \\
\cmidrule(lr){2-7} \cmidrule(lr){8-13}
 & \multicolumn{3}{c|}{\textbf{370M}} & \multicolumn{3}{c|}{\textbf{1.3B}} & \multicolumn{3}{c|}{\textbf{370M}} & \multicolumn{3}{c}{\textbf{1.3B}} \\
\cmidrule(lr){2-4} \cmidrule(lr){5-7} \cmidrule(lr){8-10} \cmidrule(lr){11-13}
\textbf{Method} & $\mathbf{K}, \mathbf{Q}$ & $\mathbf{K}$ & $\mathbf{Q}$ & $\mathbf{K}, \mathbf{Q}$ & $\mathbf{K}$ & $\mathbf{Q}$ & $\mathbf{K}, \mathbf{Q}$ & $\mathbf{K}$ & $\mathbf{Q}$ & $\mathbf{K}, \mathbf{Q}$ & $\mathbf{K}$ & $\mathbf{Q}$ \\
\midrule
L1 & 3843.5 & 425218.3 & 33.0 & 66.1 & 1596.3 & 34.9 & 41.8 & 61.0 & 49.9 & 26.5 & 29.75 & 24.5 \\
DRRQR & 31.4 & 286035.7 & 32.1 & 20.5 & 566.9 & 22.9 & 31.6 & 56.9 & 43.5 & 17.3 & 21.9 & 22.0 \\
Grad & 31.7 & 17709.6 & 31.9 & 18.3 & 20.3 & 19.5 & 33.3 & 36.2 & 34.6 & 17.4 & 18.3 & 17.7 \\
S-Wanda & 915.9 & 376074.0 & 33.0 & 60.0 & 1604.4 & 29.1 & 40.0 & 61.3 & 43.5 & 21.9 & 22.5 & 21.5 \\
\midrule
\textit{Baseline} & \multicolumn{3}{c|}{\textit{29.8}} & \multicolumn{3}{c|}{\textit{16.7}} & \multicolumn{3}{c|}{\textit{28.8}} & \multicolumn{3}{c}{\textit{16.8}} \\
\bottomrule
\end{tabular}
}
\end{table}

We observe that selecting columns based on query projections ($\mathbf{Q}$) consistently outperforms selection based solely on keys ($\mathbf{K}$) for magnitude-based methods. Since queries govern readout, pruning based on $\mathbf{Q}$ ensures we discard dimensions with minimal contribution to the output. In contrast, pruning based on $\mathbf{K}$ risks removing information that the model attempts to access with a strong query, leading to significant readout errors.

Curiously, magnitude-based methods often perform better when selecting just based on queries ($\bQ$) than when selecting based on both keys and queries ($\bK, \bQ$). Since they add the scores of keys and queries ($s_j = \|\mathbf{W}_{k, :j}\| + \|\mathbf{W}_{q, :j}\|$), this suggests the model contains large key weights whose corresponding query weights have lower magnitude. DRRQR avoids this by targeting the effective rank of the joint subspace.

The results in Table~\ref{tab:ablation_kq_targets} reveal a clear difference between the studied pruning methods. Magnitude-based methods (L1, Wanda) are unstable when targeting keys, often performing best when restricted to queries. In contrast, optimization-based methods (Grad, DRRQR) consistently achieve the lowest perplexity using joint selection ($\mathbf{K}, \mathbf{Q}$). This indicates that the associative memory's effective rank requires accounting for the coupled interaction between keys and queries, rather than treating them in isolation.

\section{The Effective Rank}
\label{appendix:stable-rank}
In this section we present some properties of the effective rank. \citet{ipsen2025stable} provide more details.

\begin{definition}[effective rank]
    For a non-zero matrix $\bA \in \mathbb{R}^{m, n}$, the effective rank is defined as:
    \begin{equation*}
        \operatorname{er}(\bA) \coloneqq \frac{\|\bA\|_F^2}{\|\bA\|_2^2} = \frac{\sum_{i} \sigma_i^2(\bA)}{\sigma_{\max}^2(\bA)}.
    \end{equation*}
\end{definition}

Unlike the algebraic rank, which is discontinuous, the effective rank is a continuous function of the matrix entries. This implies that small perturbations to the memory state $\bS_t$ (e.g., from gradient noise or quantization) result in bounded changes to $\operatorname{er}(\bS_t)$.

\begin{proposition}[Invariance under Transposition]
\label{proposition:stable-rank-invariance-under-transposition}
    The effective rank is invariant under transposition. For any matrix $\bA$:
    \begin{equation*}
        \operatorname{er}(\bA) = \operatorname{er}(\bA^\top).
    \end{equation*}
\end{proposition}

\begin{proposition}[Invariance under Unitary Transformations and Scaling]
    The effective rank is invariant under unitary transformations and scalar multiplication. For any unitary matrices $\bU, \bV$ and non-zero scalar $c \in \mathbb{R}$:
    \begin{equation*}
        \operatorname{er}(c \bU \bA \bV^\top) = \operatorname{er}(\bA).
    \end{equation*}
\end{proposition}

\begin{proposition}[Bounds and Relation to Algebraic Rank]
    The effective rank is bounded by the algebraic rank:
    \begin{equation*}
        1 \leq \operatorname{er}(\bA) \leq \operatorname{rank}(\bA) \leq \min(m, n).
    \end{equation*}
    The lower bound is achieved if and only if $\bA$ has rank 1. The upper bound is achieved if and only if all non-zero singular values are equal.
\end{proposition}

\begin{proposition}[Relation to Condition Number]
\label{proposition:stable-rank-relation-to-condition-number}
    Let $\bA$ be a (non-zero) matrix. Then
    \begin{equation*}
        \kappa^2(\bA) \geq \frac{\operatorname{rank}(\bA)}{\operatorname{er}(\bA)}.
    \end{equation*}
    In the specific case where $\bA$ has full rank, this implies
    \begin{equation*}
        \kappa^2(\bA) \geq \frac{1}{\operatorname{u}(\bA)},
    \end{equation*}
    where $u$ is the rank utilization defined in Equation~\eqref{eq:rank-utilization}.
\end{proposition}

\textbf{Edge Case (Isotropy):}
For a full-rank matrix $\bA$ with $\kappa(\bA) = 1$, let $r=\operatorname{rank}(\bA)=\min(m,n)$. The inequality becomes an equality:
\begin{equation*}
    1 \geq \frac{r}{\operatorname{er}(\bA)} \implies \operatorname{er}(\bA) \geq r.
\end{equation*}
Since we know $\operatorname{er}(\bA) \leq r$, this forces $\operatorname{er}(\bA) = \operatorname{rank}(\bA)$. This confirms that for isotropic matrices (where all non-zero singular values are equal), the effective rank and algebraic rank coincide. Conversely, a large gap between $r$ and $\operatorname{er}(\bA)$ is a sufficient condition for ill-conditioning.

\section{Dynamical Systems Perspective}
\label{appendix:dynamica-systems}

In this section, we derive the DeltaNet recurrence rule from a continuous-time dynamical systems perspective. We show that the sequence mixer can be interpreted as a discretization of a continuous gradient flow minimizing a linear regression objective.

\paragraph{Continuous-Time Dynamics.}
Consider a time-continuous associative memory $\mathbf{S}(t) \in \mathbb{R}^{d_v, d_k}$ receiving a stream of keys $\mathbf{k}(t) \in \mathbb{R}^{d_k}$ and values $\mathbf{v}(t) \in \mathbb{R}^{d_v}$. We define the instantaneous regression loss at time $t$ as:
\begin{equation*}
    \mathcal{L}(\mathbf{S}(t)) = \frac{1}{2} \| \mathbf{v}(t) - \mathbf{S}(t)\mathbf{k}(t) \|_2^2.
\end{equation*}
The dynamics of the state $\mathbf{S}(t)$ are governed by the gradient flow minimizing this objective:
\begin{align}
    \dot{\mathbf{S}}(t) &= -\nabla_{\mathbf{S}} \mathcal{L}(\mathbf{S}(t)) \\
    &= -(\mathbf{S}(t)\mathbf{k}(t) - \mathbf{v}(t))\mathbf{k}(t)^\top \\
    &= \mathbf{v}(t)\mathbf{k}(t)^\top - \mathbf{S}(t)\mathbf{k}(t)\mathbf{k}(t)^\top. \label{eq:ode-flow}
\end{align}
Equation~\eqref{eq:ode-flow} represents a linear time-varying (LTV) ordinary differential equation (ODE) of the form $\dot{\mathbf{S}}(t) = \mathbf{S}(t)\mathbf{A}(t) + \mathbf{B}(t)$, where $\mathbf{A}(t) = -\mathbf{k}(t)\mathbf{k}(t)^\top$ is the state-transition matrix acting on the right, and $\mathbf{B}(t) = \mathbf{v}(t)\mathbf{k}(t)^\top$ is the input forcing term.

\paragraph{Euler Discretization.}
To obtain the discrete-time update rule employed by DeltaNet, we apply the forward Euler method to Equation~\eqref{eq:ode-flow}. Let $\Delta t$ be the step size, which corresponds to the gating factor $\beta_t$ in the DeltaNet formulation. The discretization yields:
\begin{equation*}
    \frac{\mathbf{S}_t - \mathbf{S}_{t-1}}{\beta_t} \approx \mathbf{v}_t\mathbf{k}_t^\top - \mathbf{S}_{t-1}\mathbf{k}_t\mathbf{k}_t^\top.
\end{equation*}
Rearranging the terms to solve for the next state $\mathbf{S}_t$:
\begin{align*}
    \mathbf{S}_t &= \mathbf{S}_{t-1} + \beta_t (\mathbf{v}_t\mathbf{k}_t^\top - \mathbf{S}_{t-1}\mathbf{k}_t\mathbf{k}_t^\top) \\
    &= \mathbf{S}_{t-1}(\mathbf{I} - \beta_t \mathbf{k}_t\mathbf{k}_t^\top) + \beta_t \mathbf{v}_t\mathbf{k}_t^\top.
\end{align*}
This exactly recovers the DeltaNet update rule (Equation~\eqref{eq:dn-recursion}).

\paragraph{System Stability.}
The stability of this dynamical system is determined by the spectral properties of the transition matrix $(\mathbf{I} - \beta_t \mathbf{k}_t\mathbf{k}_t^\top)$. For the system to be stable (non-divergent), the singular values of this operator must lie within the unit circle. This implies the condition $|1 - \beta_t \|\mathbf{k}_t\|_2^2| \leq 1$. In standard DeltaNet implementations, keys are normalized ($\|\mathbf{k}_t\|_2 = 1$), and $\beta_t$ is the output of a sigmoid function ($\beta_t \in (0, 1)$), strictly satisfying the stability condition and ensuring the memory decays appropriately over time.

\section{Proofs}
\label{appendix:proofs}
In this section, we provide proofs of statements presented in the main paper.

\subsection{Proof of the Effective-Rank Bound}
\label{appendix:proof-stable-rank}

We prove Equation~\eqref{eq:effective-rank-bound}. Using the column-stacking convention of the main text, write the associative memory as $\bS = \bV\bK^\top$, where $\bV \in \mathbb{R}^{d_v, T}$ and $\bK \in \mathbb{R}^{d_k, T}$. Assume $\bK$ has full row rank.

To handle potential misalignment between the subspaces of values and keys, we decompose the values $\bV$ into two orthogonal components relative to the row space of the keys $\bK$:
\begin{equation*}
    \bV = \bV_{\parallel} + \bV_{\perp},
\end{equation*}
where the rows of $\bV_{\parallel}$ lie in $\operatorname{row}(\bK)$, and the rows of $\bV_{\perp}$ are orthogonal to it. Consequently, $\bV_{\perp}\bK^\top = \mathbf{0}$, and the memory state simplifies to:
\begin{equation*}
    \bS = (\bV_{\parallel} + \bV_{\perp})\bK^\top = \bV_{\parallel}\bK^\top.
\end{equation*}
Assume $\bV_{\parallel}\neq\mathbf{0}$. We define the scalar quantity $\nu(\bV)$ appearing in the main text as the effective rank of the projected values:
\begin{equation*}
    \nu(\bV) \coloneqq \operatorname{er}(\bV_{\parallel}) = \frac{\|\bV_{\parallel}\|_F^2}{\|\bV_{\parallel}\|_2^2}.
\end{equation*}
The quantity $\nu(\bV)$ also depends on $\bK$ through this projection. Since multiplication by $\bK^\top$ is injective on $\operatorname{row}(\bK)$, these assumptions ensure $\bS\neq\mathbf{0}$. We now derive the lower bound for the effective rank $\operatorname{er}(\bS) = \|\bS\|_F^2 / \|\bS\|_2^2$.

First, we bound the numerator (Frobenius norm) from below. We use the property $\|\mathbf{A}\mathbf{B}\|_F \geq \|\mathbf{A}\|_F \sigma_{\min}(\mathbf{B})$, which holds here because the rows of $\bV_{\parallel}$ lie in $\operatorname{col}(\bK^\top)$:
\begin{equation*}
    \|\bS\|_F^2 = \|\bV_{\parallel}\bK^\top\|_F^2 \geq \|\bV_{\parallel}\|_F^2 \sigma_{\min}^2(\bK).
\end{equation*}

Next, we bound the denominator (Spectral norm) from above using the standard sub-multiplicative property $\|\mathbf{A}\mathbf{B}\|_2 \leq \|\mathbf{A}\|_2 \|\mathbf{B}\|_2$:
\begin{equation*}
    \|\bS\|_2^2 = \|\bV_{\parallel}\bK^\top\|_2^2 \leq \|\bV_{\parallel}\|_2^2 \|\bK\|_2^2.
\end{equation*}

Combining these inequalities yields the lower bound:
\begin{equation*}
    \operatorname{er}(\bS) = \frac{\|\bS\|_F^2}{\|\bS\|_2^2} \geq \frac{\|\bV_{\parallel}\|_F^2 \sigma_{\min}^2(\bK)}{\|\bV_{\parallel}\|_2^2 \|\bK\|_2^2} = \operatorname{er}(\bV_{\parallel}) \frac{1}{\kappa^2(\bK)} = \frac{\nu(\bV)}{\kappa^2(\bK)}.
\end{equation*}
\qed

\subsection{Proof of Algebraic Rank}
\label{appendix:proof-rank}
We first prove the following proposition:
\begin{proposition}
    Let $\bS_0 = 0$ and $\bS_t$ be a matrix satisfying the recursion
    \begin{equation*}
        \bS_t = \bS_{t-1} (\alpha_t \bI - \beta_t \bmk_t \bmk_t^\top) + \gamma_t \bmv_t \bmk_t^\top
    \end{equation*}
    for some non-zero scalars $\alpha_t, \beta_t, \gamma_t$ and some vectors $\bmv_t$, $\bmk_t$. Then
    \begin{equation*}
        \row \bS_t \subseteq \operatorname{span} \langle \bmk_1,\ldots,\bmk_t \rangle \quad \text{and} \quad \operatorname{col} \bS_t \subseteq \operatorname{span} \langle \bmv_1,\ldots,\bmv_t \rangle.
    \end{equation*}
    Consequently,
    \begin{equation*}
        \rank \bS_t \leq \min \left( \rank \bK_t, \rank \bV_t \right)
    \end{equation*}
    where $\bK_t = (\bmk_1,\ldots,\bmk_t)$ and $\bV_t = (\bmv_1,\ldots,\bmv_t)$ are the matrices obtained by stacking the $\bmk$- and $\bmv$-vectors, respectively.
\end{proposition}
\begin{proof}
We start by showing the first claim by induction. For $t=1$, we have $\bS_1 = \gamma_1 \bmv_1 \bmk_1^\top$. Consider $t > 1$ and assume the claim is true for every $s < t$. Then
\begin{align*}
    \bS_t &= \bS_{t-1}(\alpha_t \bI - \beta_t \bmk_t \bmk_t^\top) + \gamma_t \bmv_t \bmk_t^\top \\
    &= \alpha_t \bS_{t-1} + (\gamma_t \bmv_t - \beta_t \bS_{t-1}\bmk_t)\bmk_t^\top \\
    &= \alpha_t \bS_{t-1} + \bmu_t\bmk_t^\top,
\end{align*}
where we defined $\bmu_t \coloneqq \gamma_t \bmv_t - \beta_t \bS_{t-1}\bmk_t$. Next, we use that for two matrices $\bA$ and $\bB$,
\begin{equation*}
\text{row}(\bA + \bB) \subseteq \text{span}(\text{row}(\bA), \text{row}(\bB))
\end{equation*}
and thus
\begin{align*}
\text{row}(\alpha_t \bS_{t-1} + \bmu_t \bmk_t^\top) &\subseteq \text{span}(\text{row}(\alpha_t \bS_{t-1}), \text{row}(\bmu_t \bmk_t^\top)) \\
&\subseteq \text{span}(\text{row}(\bS_{t-1}), \bmk_t) \\
&\subseteq \text{span}(\bmk_1,...,\bmk_t).
\end{align*}
This concludes the first claim.

For the second claim, we proceed analogously.

Indeed, we again show this claim by induction. The case $t=1$ is clear. For any $t>1$, we compute, using \mbox{$\text{col}(\bA\bB) \subseteq \text{col}(\bA)$},
\begin{align*}
\text{col}(\bS_t) &\subseteq \text{col}(\bS_{t-1}(\alpha_t \bI - \beta_t \bmk_t \bmk_t^\top), \gamma_t \bmv_t \bmk_t^\top) \\
&\subseteq \text{col}(\bS_{t-1}, \gamma_t \bmv_t \bmk_t^\top) \\
&\subseteq \text{col}(\bmv_1,\ldots,\bmv_t).
\end{align*}
This concludes the proof.
\end{proof}

Interestingly, PCA-based transformations are guaranteed not to decrease the rank utilization of the keys:
\begin{lemma}[Monotonicity of Rank Utilization]
    For a non-zero key matrix with at least $d_k$ samples, pruning the low-variance directions via PCA increases or maintains rank utilization. That is, for any $1 \leq d_k' < d_k$:
    \begin{equation*}
        u(\bK') \geq u(\bK).
    \end{equation*}
\end{lemma}
\begin{proof}
    Assume towards a contradiction that the utilization decreases, i.e., $u(\bK') < u(\bK)$. This means that
    \begin{equation*}
         \frac{1}{d_k'} \sum_{i=1}^{d_k'} \sigma_i^2 < \frac{1}{d_k} \sum_{i=1}^{d_k} \sigma_i^2.
    \end{equation*}
    In words, the average energy of the top $d_k'$ principal components is strictly less than the average energy of the full spectrum. This is a contradiction, as the singular values are non-increasing ($\sigma_1 \geq \dots \geq \sigma_{d_k}$). Thus, our assumption must have been wrong.
\end{proof}

\subsection{Proof of Proposition~\ref{thm:snr-degradation}}
\label{appendix:output-util-proof}

Assume $\bS\in\mathbb{R}^{d_v,d_k}$ has full column rank, so $d=d_k$ and $\sigma_{d_k}>0$. Since $R$ is invariant to rescaling the query, take $\|\bmq^*\|_2=1$. For nonzero $\bm{\varepsilon}$, write
\begin{equation*}
    R(\bmq^*,\bm{\varepsilon})
    =\frac{\|\bS\bm{\varepsilon}\|_2/\|\bm{\varepsilon}\|_2}{\|\bS\bmq^*\|_2}
    \leq\frac{\sigma_1}{\|\bS\bmq^*\|_2}.
\end{equation*}

\paragraph{Worst-Case Amplification.}
Let $\bmw_i$ be right singular vectors of $\bS$, ordered by decreasing singular value. For each fixed query, the upper bound is attained by taking $\bm{\varepsilon}=\bmw_1$. Hence
\begin{equation*}
    \max_{\bm{\varepsilon}\neq\mathbf{0}}R(\bmq^*,\bm{\varepsilon})
    =\frac{\sigma_1}{\|\bS\bmq^*\|_2}.
\end{equation*}
For unit-norm queries, $\|\bS\bmq^*\|_2\geq\sigma_{d_k}$, with equality at $\bmq^*=\bmw_{d_k}$. Maximizing over the query therefore gives
\begin{equation*}
    \max_{\|\bmq^*\|_2=1}\max_{\bm{\varepsilon}\neq\mathbf{0}}R(\bmq^*,\bm{\varepsilon})
    =\frac{\sigma_1}{\sigma_{d_k}}=\kappa(\bS).
\end{equation*}
Furthermore,
\begin{equation*}
    u(\bS)=\frac{1}{d_k}\sum_{i=1}^{d_k}\frac{\sigma_i^2}{\sigma_1^2}
    \geq\frac{\sigma_{d_k}^2}{\sigma_1^2}
    =\frac{1}{\kappa(\bS)^2}.
\end{equation*}
Rearranging proves Equation~\eqref{eq:snr-degradation}.

\paragraph{Isotropic Endpoint.}
Under the full-column-rank assumption, $u(\bS)=1$ holds if and only if all $d_k$ singular values equal $\sigma_1$. Then $\bS^\top\bS=\sigma_1^2\bI_{d_k}$, so $\|\bS\bmx\|_2=\sigma_1\|\bmx\|_2$ for every $\bmx$, and $R=1$ for all nonzero queries and perturbations.
\qed

\subsection{Conditioning of Query Weight Gradients}
\label{app:matrix-derivative}
For the linear query map $\bmq=\bW_{\bQ}^\top\bmx$, hold the state $\bS$ and nonzero input $\bmx$ fixed. The readout $\bmo=\bS\bW_{\bQ}^\top\bmx$ has Jacobian
\begin{equation*}
    \bJ \coloneqq \frac{\partial\bmo}{\partial\operatorname{vec}(\bW_{\bQ})}
    =\bS\otimes\bmx^\top.
\end{equation*}
For a scalar loss $\mathcal{L}$ with output gradient $\bmg=\nabla_{\bmo}\mathcal{L}$, the backward map is
\begin{equation*}
    \nabla_{\operatorname{vec}(\bW_{\bQ})}\mathcal{L}
    =\bJ^\top\bmg=(\bS^\top\otimes\bmx)\bmg.
\end{equation*}
If $\bS$ has full column rank, the nonzero singular values of this map are $\|\bmx\|_2\sigma_i(\bS)$, for $i=1,\ldots,d_k$. Their ratio is therefore
\begin{equation*}
    \frac{\sigma_{\max}(\bJ^\top)}{\sigma_{\min}^{+}(\bJ^\top)}
    =\kappa(\bS)\geq\frac{1}{\sqrt{u(\bS)}},
\end{equation*}
where $\sigma_{\min}^{+}$ denotes the smallest nonzero singular value. Thus low rank utilization makes the backward map ill-conditioned on its nonzero singular subspaces. This statement concerns the local linear readout. Its effect on a particular gradient depends on the direction of $\bmg$.

\paragraph{Derivation.}
The vectorization identity~\citep[Equation (520)]{petersen2008matrix} gives
\begin{equation*}
    \bmo=\operatorname{vec}(\bmo^\top)
    =\operatorname{vec}(\bmx^\top\bW_{\bQ}\bS^\top)
    =(\bS\otimes\bmx^\top)\operatorname{vec}(\bW_{\bQ}).
\end{equation*}
Differentiating yields $\bJ$, and the chain rule gives the backward map. The singular values of a Kronecker product are the pairwise products of those of its factors~\citep[Theorem 4.2.15]{horn1991topics}. Since $\bmx^\top$ has the single nonzero singular value $\|\bmx\|_2$, the stated singular values follow. Finally, the proof of Proposition~\ref{thm:snr-degradation} gives $\kappa(\bS)\geq1/\sqrt{u(\bS)}$.
\qed

\section{Extended Results}
\label{app:extended-results}
\subsection{End-to-End Efficiency}
\label{app:end-to-end-efficiency}
We benchmark DeltaNet and Gated DeltaNet 1.3B end to end on an H100-SXM with CUDA-graph capture. All configurations use the same unmodified flash-linear-attention kernels, including the standard chunked kernels for prefill. State reduction changes only $d_k'$. We use 10 warmup steps and 100 measured steps and report the median across three independent trials. Table~\ref{tab:decode-main} reports the complete decode results for prompt length 512, and Table~\ref{tab:prefill-full-grid} the complete prefill grid.

\begin{table}[ht]
\centering
\caption{Complete full-model decode results on an H100-SXM for DeltaNet and Gated DeltaNet 1.3B (prompt length 512, batch size $B$). Base is ms/token. Relative entries are speedup.}
\label{tab:decode-main}
\small
\setlength{\tabcolsep}{4pt}
\begin{sc}
\begin{tabular}{c ccc ccc}
\toprule
& \multicolumn{3}{c}{\normalfont GDN} & \multicolumn{3}{c}{DN} \\
\cmidrule(lr){2-4}\cmidrule(lr){5-7}
$B$ & Base & 50\% & 75\% & Base & 50\% & 75\% \\
\midrule
1 & 2.72 & $1.03\times$ & $1.05\times$ & 2.22 & $1.03\times$ & $1.04\times$ \\
2 & 2.85 & $1.06\times$ & $1.09\times$ & 2.31 & $1.06\times$ & $1.07\times$ \\
4 & 3.18 & $1.10\times$ & $1.16\times$ & 2.56 & $1.09\times$ & $1.14\times$ \\
8 & 3.79 & $1.17\times$ & $1.28\times$ & 2.99 & $1.14\times$ & $1.23\times$ \\
16 & 4.96 & $1.25\times$ & $1.45\times$ & 3.88 & $1.23\times$ & $1.38\times$ \\
32 & 7.20 & $1.37\times$ & $1.67\times$ & 5.48 & $1.31\times$ & $1.57\times$ \\
64 & 11.61 & $1.47\times$ & $1.94\times$ & 8.68 & $1.41\times$ & $1.78\times$ \\
\bottomrule
\end{tabular}
\end{sc}
\end{table}

\begin{table}[ht]
\centering
\caption{Complete full-model prefill speedup grid on an H100-SXM for DeltaNet and Gated DeltaNet 1.3B, over all measured batch sizes $B$ and context lengths $T$. Cells absent from the grid are configurations where the chunked kernel exceeds device memory.}
\label{tab:prefill-full-grid}
\small
\setlength{\tabcolsep}{4pt}
\begin{tabular}{ccccc}
\toprule
$B\times T$ & \makecell{GDN\\50\%} & \makecell{GDN\\75\%} & \makecell{DN\\50\%} & \makecell{DN\\75\%} \\
\midrule
$1\times2$K & $1.00\times$ & $1.01\times$ & $1.04\times$ & $1.01\times$ \\
$2\times2$K & $1.00\times$ & $1.01\times$ & $1.03\times$ & $1.00\times$ \\
$4\times2$K & $1.16\times$ & $1.24\times$ & $1.12\times$ & $1.19\times$ \\
$8\times2$K & $1.14\times$ & $1.21\times$ & $1.12\times$ & $1.18\times$ \\
$16\times2$K & $1.14\times$ & $1.22\times$ & $1.12\times$ & $1.18\times$ \\
$32\times2$K & $1.14\times$ & $1.21\times$ & $1.12\times$ & $1.19\times$ \\
$64\times2$K & $1.14\times$ & $1.22\times$ & $1.12\times$ & $1.19\times$ \\
$128\times2$K & $1.14\times$ & $1.22\times$ & $1.12\times$ & $1.19\times$ \\
\midrule
$1\times8$K & $1.16\times$ & $1.21\times$ & $1.14\times$ & $1.19\times$ \\
$2\times8$K & $1.14\times$ & $1.21\times$ & $1.12\times$ & $1.18\times$ \\
$4\times8$K & $1.14\times$ & $1.21\times$ & $1.11\times$ & $1.18\times$ \\
$8\times8$K & $1.15\times$ & $1.22\times$ & $1.12\times$ & $1.18\times$ \\
$16\times8$K & $1.14\times$ & $1.22\times$ & $1.12\times$ & $1.18\times$ \\
$32\times8$K & $1.14\times$ & $1.22\times$ & $1.12\times$ & $1.18\times$ \\
\midrule
$1\times16$K & $1.15\times$ & $1.22\times$ & $1.12\times$ & $1.18\times$ \\
$2\times16$K & $1.14\times$ & $1.20\times$ & $1.12\times$ & $1.18\times$ \\
$4\times16$K & $1.15\times$ & $1.22\times$ & $1.12\times$ & $1.18\times$ \\
$8\times16$K & $1.14\times$ & $1.21\times$ & $1.12\times$ & $1.18\times$ \\
$16\times16$K & $1.14\times$ & $1.22\times$ & $1.12\times$ & $1.18\times$ \\
\bottomrule
\end{tabular}
\end{table}

For comparison, Table~\ref{tab:mixer-scaling} isolates the sequence mixer. These kernel-level measurements expose the computational benefit of reducing $d_k$, but exclude projections, normalization, feed-forward layers, and other full-model costs captured by Tables~\ref{tab:decode-main} and~\ref{tab:prefill-full-grid}.

\begin{table}[ht]
\centering
\caption{Gated DeltaNet 1.3B sequence-mixer-only prefill speedup and peak-VRAM ratio relative to the uncompressed baseline at the same batch size. Measurements use an NVIDIA H100 ($T=2048$, 16 heads, $h=2048$, $d_v=128$).}
\label{tab:mixer-scaling}
\begin{small}
\begin{tabular}{cccc}
\toprule
Batch size & State reduction & Speedup & VRAM ratio \\
\midrule
\multirow{2}{*}{8}  & 50\% & $1.26\times$ & $0.73\times$ \\
                     & 75\% & $1.32\times$ & $0.60\times$ \\
\midrule
\multirow{2}{*}{16} & 50\% & $1.28\times$ & $0.73\times$ \\
                     & 75\% & $1.50\times$ & $0.59\times$ \\
\midrule
\multirow{2}{*}{32} & 50\% & $1.28\times$ & $0.72\times$ \\
                     & 75\% & $1.51\times$ & $0.59\times$ \\
\bottomrule
\end{tabular}
\end{small}
\end{table}

\clearpage
\subsection{Language Modeling}
This subsection contains the extensive zero-shot~\citep{eval-harness} task evaluations. Table~\ref{tab:unified_summary_50} contains averaged results for all models at a fixed compression ratio of 50\%.
\begin{table}[ht]
\centering
\caption{Comparison of compression methods at a 50\% compression ratio across all model sizes, both pre- and post-RFT. We report perplexity for WikiText (Wiki) and Lambada (LMB), as well as average accuracy for LM Evals (best results in \textbf{bold}, second best \underline{underlined}).}
\label{tab:unified_summary_50}
\resizebox{\textwidth}{!}{%
\begin{sc}
\setlength{\tabcolsep}{4.5pt}
\begin{tabular}{ll|ccc|ccc}
\toprule
& & \multicolumn{3}{c|}{\textbf{Pre-RFT}} & \multicolumn{3}{c}{\textbf{Post-RFT}} \\
\cmidrule(lr){3-5} \cmidrule(lr){6-8}
\textbf{Model} & \textbf{Method} & Wiki~$\downarrow$ & LMB~$\downarrow$ & LM Evals Avg~$\uparrow$ & Wiki~$\downarrow$ & LMB~$\downarrow$ & LM Evals Avg~$\uparrow$ \\
\midrule
\multirow[t]{6}{*}{\textbf{DeltaNet 370M}}
& L1 & 3843.5 & 57304.8 & 33.0 & 32.5 & 56.8 & 43.0 \\
& Rand & 1032.6 & 2882.8 & 34.3 & 32.6 & 49.9 & 42.6 \\
& DRRQR  & \textbf{31.4} & \underline{36.0} & \underline{44.7} & \textbf{29.4} & \underline{36.6} & \textbf{44.5} \\
& Grad & \underline{31.7} & \textbf{33.3} & \textbf{44.9} & \textbf{29.4} & \textbf{36.3} & \underline{44.4} \\
& S-Wanda  & 915.9 & 6118.0 & 35.1 & \underline{32.1} & 53.2 & 43.4 \\
\midrule
& Baseline & 29.8 & 37.0 & 44.2 & 29.8 & 37.0 & 44.2 \\
\midrule
\multirow[t]{6}{*}{\textbf{DeltaNet 1.3B}}
& L1 & 66.1 & 298.7 & 41.6 & 19.0 & 16.1 & 48.2 \\
& Rand & 41.2 & 68.0 & 44.4 & 19.4 & 14.4 & 48.2 \\
& DRRQR  & \underline{20.5} & \underline{47.2} & \underline{45.7} & \underline{17.5} & \underline{11.8} & \underline{49.7} \\
& Grad & \textbf{18.3} & \textbf{15.4} & \textbf{48.8} & \textbf{17.2} & \textbf{11.3} & \textbf{50.3} \\
& S-Wanda  & 60.0 & 361.3 & 40.8 & 18.5 & 15.9 & 48.3 \\
\midrule
& Baseline & 16.7 & 11.9 & 50.0 & 16.7 & 11.9 & 50.0 \\
\midrule
\multirow[t]{6}{*}{\textbf{Gated DeltaNet 370M}}
& L1 & 41.8 & 156.6 & 40.4 & \underline{29.2} & 44.6 & 43.5 \\
& Rand & 35.5 & 50.7 & 43.4 & 29.3 & 46.5 & 43.6 \\
& DRRQR  & \textbf{31.6} & \underline{39.4} & \textbf{44.0} & \textbf{28.7} & \underline{40.6} & \textbf{43.8} \\
& Grad & \underline{33.3} & \textbf{39.2} & \underline{43.8} & \textbf{28.7} & \textbf{39.4} & \underline{43.7} \\
& S-Wanda  & 40.0 & 141.5 & 40.9 & \underline{29.2} & 44.6 & 43.5 \\
\midrule
& Baseline & 28.8 & 35.9 & 44.4 & 28.8 & 35.9 & 44.4 \\
\midrule
\multirow[t]{6}{*}{\textbf{Gated DeltaNet 1.3B}}
& L1 & 26.5 & 34.0 & 53.3 & 16.8 & 13.0 & 57.2 \\
& Rand & 18.8 & 14.4 & 56.5 & 16.9 & \underline{11.5} & 57.8 \\
& DRRQR  & \textbf{17.3} & \underline{14.2} & \underline{57.5} & \textbf{16.3} & 12.0 & \underline{58.0} \\
& Grad & \underline{17.4} & \textbf{10.1} & \textbf{58.6} & \underline{16.4} & \textbf{10.5} & \textbf{58.3} \\
& S-Wanda  & 21.9 & 24.1 & 55.1 & 16.6 & 11.7 & 57.8 \\
\midrule
& Baseline & 16.8 & 9.7 & 59.4 & 16.8 & 9.7 & 59.4 \\
\bottomrule
\end{tabular}
\end{sc}
}
\end{table}

\begin{table}[t]
\centering
\caption{Zero-shot performance of DeltaNet 370M models, evaluated using the {lm-eval-harness}~\citep{eval-harness}, given different compression ratios. Pre- and post RFT.}
\label{tab:zs-dn-370M}
\small
\setlength{\tabcolsep}{3.5pt}
\resizebox{\columnwidth}{!}{%
\begin{tabular}{lll|cc|cccccc|c}
\toprule
\thead[t]{RFT} & \thead[t]{Method} & \thead[t]{Compr.} & \makecell[t]{Wiki.~$\downarrow$\\ppl} & \makecell[t]{LMB.~$\downarrow$\\ppl} & \makecell[t]{ARC-e~$\uparrow$\\acc\_n} & \makecell[t]{ARC-c~$\uparrow$\\acc\_n} & \makecell[t]{Hella.~$\uparrow$\\acc\_n} & \makecell[t]{Wino.~$\uparrow$\\acc} & \makecell[t]{PIQA~$\uparrow$\\acc\_n} & \makecell[t]{LMB.~$\uparrow$\\acc} & \makecell[t]{Avg~$\uparrow$} \\
\midrule
$\times$ & L1 & 75\% & 100083.0 & 9661610.6 & 27.2 & 28.0 & 27.7 & 52.2 & 50.9 & 0.0 & 31.0 \\
$\times$ & Rand & 75\% & 136955.3 & 5879873.3 & 27.0 & 27.4 & 26.6 & 48.9 & 51.4 & 0.0 & 30.2 \\
$\times$ & DRRQR & 75\% & \textbf{42.0} & \underline{74.1} & 51.9 & 27.5 & 38.7 & 50.9 & 64.6 & 25.3 & \underline{43.1} \\
$\times$ & Grad & 75\% & \underline{46.7} & \textbf{58.0} & 52.0 & 27.0 & 38.5 & 49.8 & 64.7 & 29.1 & \textbf{43.5} \\
$\times$ & S-Wanda & 75\% & 33512.5 & 4955134.6 & 28.5 & 26.8 & 27.7 & 51.1 & 48.9 & 0.0 & 30.5 \\
\addlinespace
\checkmark & L1 & 75\% & 36.3 & 146.7 & 46.0 & 26.2 & 36.1 & 51.0 & 63.0 & 17.1 & 39.9 \\
\checkmark & Rand & 75\% & 40.3 & 129.9 & 46.8 & 26.1 & 35.8 & 51.9 & 62.8 & 17.7 & 40.2 \\
\checkmark & DRRQR & 75\% & \textbf{31.4} & \underline{43.7} & 52.0 & 27.5 & 38.1 & 52.4 & 64.6 & 29.7 & \underline{44.0} \\
\checkmark & Grad & 75\% & \underline{31.5} & \textbf{42.5} & 52.1 & 28.0 & 38.0 & 51.6 & 64.8 & 30.2 & \textbf{44.1} \\
\checkmark & S-Wanda & 75\% & 35.4 & 101.7 & 46.5 & 26.7 & 37.1 & 50.9 & 63.3 & 20.3 & 40.8 \\
\midrule
$\times$ & L1 & 50\% & 3843.5 & 57304.8 & 35.8 & 23.6 & 31.2 & 50.8 & 55.0 & 1.6 & 33.0 \\
$\times$ & Rand & 50\% & 1032.6 & 2882.8 & 37.8 & 23.7 & 30.3 & 50.5 & 56.6 & 6.9 & 34.3 \\
$\times$ & DRRQR & 50\% & \textbf{31.4} & \underline{36.0} & 52.1 & 27.6 & 38.7 & 52.0 & 65.1 & 32.7 & \underline{44.7} \\
$\times$ & Grad & 50\% & \underline{31.7} & \textbf{33.3} & 51.8 & 27.9 & 38.7 & 52.2 & 65.5 & 33.4 & \textbf{44.9} \\
$\times$ & S-Wanda & 50\% & 915.9 & 6118.0 & 39.9 & 24.3 & 33.5 & 50.1 & 57.4 & 5.4 & 35.1 \\
\addlinespace
\checkmark & L1 & 50\% & 32.5 & 56.8 & 49.2 & 27.0 & 37.7 & 53.0 & 64.6 & 26.4 & 43.0 \\
\checkmark & Rand & 50\% & 32.6 & 49.9 & 49.2 & 27.9 & 37.6 & 49.0 & 64.9 & 27.3 & 42.6 \\
\checkmark & DRRQR & 50\% & \textbf{29.4} & \underline{36.6} & 51.8 & 27.7 & 38.1 & 52.1 & 65.3 & 31.7 & \textbf{44.5} \\
\checkmark & Grad & 50\% & \textbf{29.4} & \textbf{36.3} & 51.0 & 27.7 & 38.1 & 52.1 & 65.5 & 32.0 & \underline{44.4} \\
\checkmark & S-Wanda & 50\% & \underline{32.1} & 53.2 & 51.1 & 26.1 & 37.9 & 52.6 & 65.0 & 27.7 & 43.4 \\
\midrule
$\times$ & L1 & 40\% & 1232.0 & 5431.6 & 39.6 & 23.9 & 34.0 & 50.2 & 56.4 & 5.7 & 34.9 \\
$\times$ & Rand & 40\% & 121.3 & 329.1 & 42.8 & 25.4 & 34.2 & 49.8 & 61.5 & 15.0 & 38.1 \\
$\times$ & DRRQR & 40\% & \textbf{30.2} & \underline{33.4} & 52.6 & 28.0 & 38.7 & 51.9 & 65.2 & 33.1 & \textbf{44.9} \\
$\times$ & Grad & 40\% & \underline{30.4} & \textbf{32.7} & 52.2 & 27.7 & 38.7 & 51.8 & 65.3 & 33.4 & \textbf{44.9} \\
$\times$ & S-Wanda & 40\% & 276.8 & 603.4 & 45.3 & 25.4 & 35.8 & 50.3 & 60.3 & 13.4 & \underline{38.4} \\
\addlinespace
\checkmark & L1 & 40\% & 31.7 & 53.9 & 50.8 & 26.7 & 37.7 & 53.0 & 65.1 & 26.5 & 43.3 \\
\checkmark & Rand & 40\% & 31.0 & 43.6 & 49.1 & 27.6 & 38.0 & 52.5 & 65.2 & 29.2 & 43.6 \\
\checkmark & DRRQR & 40\% & \textbf{29.0} & \underline{36.3} & 51.4 & 27.8 & 38.2 & 52.0 & 65.4 & 32.0 & \textbf{44.5} \\
\checkmark & Grad & 40\% & \underline{29.1} & \textbf{35.7} & 51.5 & 27.4 & 38.2 & 51.5 & 65.5 & 32.1 & \underline{44.4} \\
\checkmark & S-Wanda & 40\% & 31.2 & 48.3 & 50.7 & 26.9 & 38.0 & 52.4 & 64.8 & 28.8 & 43.6 \\
\midrule
$\times$ & L1 & 30\% & 270.1 & 307.0 & 45.6 & 24.7 & 36.2 & 50.2 & 60.1 & 18.5 & 39.2 \\
$\times$ & Rand & 30\% & 52.7 & 103.2 & 48.4 & 26.8 & 36.1 & 50.4 & 62.9 & 22.7 & 41.2 \\
$\times$ & DRRQR & 30\% & \textbf{29.3} & \textbf{32.7} & 52.4 & 27.7 & 38.6 & 52.7 & 65.0 & 33.5 & \textbf{45.0} \\
$\times$ & Grad & 30\% & \underline{29.4} & \underline{33.0} & 51.7 & 27.9 & 38.6 & 52.5 & 64.7 & 33.4 & \underline{44.8} \\
$\times$ & S-Wanda & 30\% & 138.9 & 133.9 & 46.3 & 26.0 & 37.1 & 50.0 & 62.0 & 22.7 & 40.7 \\
\addlinespace
\checkmark & L1 & 30\% & 30.8 & 50.9 & 51.1 & 27.0 & 37.9 & 53.6 & 64.7 & 27.4 & 43.6 \\
\checkmark & Rand & 30\% & \underline{30.0} & \underline{40.7} & 50.1 & 27.6 & 38.1 & 52.2 & 65.5 & 30.3 & 44.0 \\
\checkmark & DRRQR & 30\% & \textbf{28.8} & \textbf{36.0} & 51.3 & 27.6 & 38.2 & 52.1 & 65.1 & 32.3 & \underline{44.4} \\
\checkmark & Grad & 30\% & \textbf{28.8} & \textbf{36.0} & 51.3 & 27.7 & 38.1 & 51.9 & 65.5 & 32.3 & \textbf{44.5} \\
\checkmark & S-Wanda & 30\% & 30.7 & 45.4 & 51.2 & 27.2 & 38.0 & 52.6 & 64.9 & 29.6 & 43.9 \\
\midrule
\textbf{Baseline} & -- & 0\% & 29.8 & 37.0 & 51.1 & 27.6 & 38.1 & 52.2 & 65.0 & 31.3 & 44.2 \\
\bottomrule
\end{tabular}
}
\end{table}

\begin{table}[t]
\centering
\caption{Zero-shot performance of DeltaNet 1.3B models, evaluated using the {lm-eval-harness}~\citep{eval-harness}, given different compression ratios. Pre- and post RFT.}
\label{tab:zs-dn-13B}
\centering
\small
\setlength{\tabcolsep}{3.5pt}
\resizebox{\columnwidth}{!}{%
\begin{tabular}{lll|cc|cccccc|c}
\toprule
\thead[t]{RFT} & \thead[t]{Method} & \thead[t]{Compr.} & \makecell[t]{Wiki.~$\downarrow$\\ppl} & \makecell[t]{LMB.~$\downarrow$\\ppl} & \makecell[t]{ARC-e~$\uparrow$\\acc\_n} & \makecell[t]{ARC-c~$\uparrow$\\acc\_n} & \makecell[t]{Hella.~$\uparrow$\\acc\_n} & \makecell[t]{Wino.~$\uparrow$\\acc} & \makecell[t]{PIQA~$\uparrow$\\acc\_n} & \makecell[t]{LMB.~$\uparrow$\\acc} & \makecell[t]{Avg~$\uparrow$} \\
\midrule
$\times$ & L1 & 75\% & 2860.8 & 6728.9 & 36.0 & 26.4 & 41.5 & 51.9 & 62.4 & 2.6 & 36.8 \\
$\times$ & Rand & 75\% & 53664.2 & 3585.3 & 37.4 & 26.9 & 36.5 & 50.2 & 60.3 & 3.6 & 35.8 \\
$\times$ & DRRQR & 75\% & \underline{43.9} & \underline{375.5} & 43.6 & 27.7 & 49.1 & 51.3 & 69.5 & 9.9 & \underline{41.8} \\
$\times$ & Grad & 75\% & \textbf{27.0} & \textbf{61.5} & 49.7 & 27.0 & 49.3 & 52.3 & 69.8 & 22.4 & \textbf{45.1} \\
$\times$ & S-Wanda & 75\% & 3160.0 & 6856.0 & 37.0 & 27.4 & 40.5 & 51.8 & 64.3 & 2.0 & 37.2 \\
\addlinespace
\checkmark & L1 & 75\% & 20.2 & 26.6 & 48.1 & 25.8 & 47.3 & 51.1 & 68.7 & 33.1 & 45.7 \\
\checkmark & Rand & 75\% & 22.9 & 22.9 & 47.8 & 27.1 & 46.5 & 51.5 & 69.7 & 35.3 & 46.3 \\
\checkmark & DRRQR & 75\% & \underline{19.2} & \underline{17.5} & 50.5 & 27.3 & 48.7 & 53.5 & 69.6 & 40.5 & \underline{48.4} \\
\checkmark & Grad & 75\% & \textbf{19.0} & \textbf{14.7} & 50.7 & 26.8 & 49.1 & 52.4 & 69.9 & 43.6 & \textbf{48.8} \\
\checkmark & S-Wanda & 75\% & 20.0 & 23.1 & 48.6 & 26.4 & 47.2 & 51.9 & 69.5 & 35.6 & 46.5 \\
\midrule
$\times$ & L1 & 50\% & 66.1 & 298.7 & 42.4 & 27.2 & 47.9 & 54.4 & 67.1 & 10.6 & 41.6 \\
$\times$ & Rand & 50\% & 41.2 & 68.0 & 46.9 & 27.8 & 47.9 & 53.7 & 69.5 & 20.4 & 44.4 \\
$\times$ & DRRQR & 50\% & \underline{20.5} & \underline{47.2} & 46.5 & 27.8 & 51.3 & 53.0 & 69.9 & 25.5 & \underline{45.7} \\
$\times$ & Grad & 50\% & \textbf{18.3} & \textbf{15.4} & 48.8 & 27.5 & 50.9 & 53.4 & 70.2 & 42.2 & \textbf{48.8} \\
$\times$ & S-Wanda & 50\% & 60.0 & 361.3 & 40.8 & 27.0 & 47.1 & 53.7 & 66.4 & 9.9 & 40.8 \\
\addlinespace
\checkmark & L1 & 50\% & 19.0 & 16.1 & 48.9 & 27.1 & 48.6 & 53.7 & 69.7 & 41.1 & 48.2 \\
\checkmark & Rand & 50\% & 19.4 & 14.4 & 48.5 & 26.2 & 48.6 & 52.5 & 69.9 & 43.4 & 48.2 \\
\checkmark & DRRQR & 50\% & \underline{17.5} & \underline{11.8} & 51.0 & 26.6 & 49.7 & 53.3 & 69.7 & 48.2 & \underline{49.7} \\
\checkmark & Grad & 50\% & \textbf{17.2} & \textbf{11.3} & 51.2 & 26.5 & 49.9 & 54.5 & 70.7 & 48.7 & \textbf{50.3} \\
\checkmark & S-Wanda & 50\% & 18.5 & 15.9 & 49.7 & 27.5 & 48.5 & 52.0 & 70.2 & 41.5 & 48.3 \\
\midrule
$\times$ & L1 & 40\% & 45.2 & 194.4 & 44.1 & 27.3 & 48.2 & 53.8 & 68.5 & 13.6 & 42.6 \\
$\times$ & Rand & 40\% & 23.9 & \underline{30.0} & 49.3 & 27.6 & 49.1 & 52.3 & 69.3 & 30.6 & 46.4 \\
$\times$ & DRRQR & 40\% & \underline{19.3} & 37.8 & 47.8 & 28.3 & 52.0 & 54.4 & 69.5 & 28.4 & \underline{46.7} \\
$\times$ & Grad & 40\% & \textbf{17.9} & \textbf{13.0} & 48.8 & 26.9 & 51.0 & 53.5 & 70.3 & 45.6 & \textbf{49.4} \\
$\times$ & S-Wanda & 40\% & 44.7 & 189.0 & 42.6 & 27.6 & 47.8 & 54.1 & 68.3 & 13.7 & 42.4 \\
\addlinespace
\checkmark & L1 & 40\% & 18.5 & 13.7 & 49.8 & 27.1 & 48.8 & 52.2 & 69.8 & 44.6 & 48.7 \\
\checkmark & Rand & 40\% & 18.6 & 13.0 & 49.6 & 26.4 & 48.9 & 52.8 & 70.2 & 44.8 & 48.8 \\
\checkmark & DRRQR & 40\% & \textbf{17.1} & \underline{11.1} & 51.5 & 27.0 & 49.9 & 53.9 & 69.7 & 49.2 & \underline{50.2} \\
\checkmark & Grad & 40\% & \textbf{17.1} & \textbf{10.4} & 50.8 & 26.8 & 50.1 & 54.5 & 70.7 & 50.8 & \textbf{50.6} \\
\checkmark & S-Wanda & 40\% & \underline{18.2} & 13.7 & 50.7 & 27.2 & 48.9 & 51.8 & 70.3 & 45.0 & 49.0 \\
\midrule
$\times$ & L1 & 30\% & 37.6 & 128.6 & 44.5 & 27.1 & 48.7 & 53.4 & 68.6 & 16.6 & 43.2 \\
$\times$ & Rand & 30\% & 19.8 & \underline{22.9} & 48.9 & 27.3 & 49.3 & 53.7 & 69.6 & 36.7 & 47.6 \\
$\times$ & DRRQR & 30\% & \underline{18.3} & 27.4 & 48.5 & 29.0 & 52.3 & 55.2 & 70.4 & 33.2 & \underline{48.1} \\
$\times$ & Grad & 30\% & \textbf{17.2} & \textbf{11.9} & 50.3 & 26.8 & 50.8 & 54.6 & 70.8 & 47.5 & \textbf{50.1} \\
$\times$ & S-Wanda & 30\% & 36.9 & 119.0 & 44.4 & 26.1 & 48.6 & 52.6 & 68.7 & 17.4 & 43.0 \\
\addlinespace
\checkmark & L1 & 30\% & 18.1 & 12.8 & 50.0 & 27.0 & 49.3 & 52.9 & 70.0 & 45.9 & 49.2 \\
\checkmark & Rand & 30\% & 17.4 & 12.0 & 50.4 & 26.9 & 49.6 & 53.9 & 70.3 & 46.6 & 49.6 \\
\checkmark & DRRQR & 30\% & \textbf{16.8} & \underline{10.5} & 51.1 & 27.4 & 50.0 & 53.9 & 70.0 & 50.2 & \underline{50.4} \\
\checkmark & Grad & 30\% & \underline{16.9} & \textbf{10.1} & 51.3 & 26.5 & 50.4 & 54.8 & 70.6 & 51.1 & \textbf{50.8} \\
\checkmark & S-Wanda & 30\% & 18.0 & 13.0 & 50.5 & 27.0 & 49.2 & 54.0 & 70.3 & 45.4 & 49.4 \\
\midrule
\textbf{Baseline} & -- & 0\% & 16.7 & 11.9 & 51.3 & 26.1 & 50.6 & 53.3 & 70.5 & 48.4 & 50.0 \\
\bottomrule
\end{tabular}
}
\end{table}

\begin{table}[t]
\centering
\caption{Zero-shot performance of Gated DeltaNet 370M models, evaluated using the {lm-eval-harness}~\citep{eval-harness}, given different compression ratios. Pre- and post RFT.}
\label{tab:zs-gdn-370M}
\centering
\small
\setlength{\tabcolsep}{3.5pt}
\resizebox{\columnwidth}{!}{%
\begin{tabular}{lll|cc|cccccc|c}
\toprule
\thead[t]{RFT} & \thead[t]{Method} & \thead[t]{Compr.} & \makecell[t]{Wiki.~$\downarrow$\\ppl} & \makecell[t]{LMB.~$\downarrow$\\ppl} & \makecell[t]{ARC-e~$\uparrow$\\acc\_n} & \makecell[t]{ARC-c~$\uparrow$\\acc\_n} & \makecell[t]{Hella.~$\uparrow$\\acc\_n} & \makecell[t]{Wino.~$\uparrow$\\acc} & \makecell[t]{PIQA~$\uparrow$\\acc\_n} & \makecell[t]{LMB.~$\uparrow$\\acc} & \makecell[t]{Avg~$\uparrow$} \\
\midrule
$\times$ & L1 & 75\% & 116.4 & 1399.7 & 36.8 & 27.6 & 34.7 & 50.3 & 59.7 & 8.5 & 36.3 \\
$\times$ & Rand & 75\% & 135.3 & 1078.1 & 41.5 & 26.5 & 34.4 & 51.2 & 60.4 & 11.2 & 37.5 \\
$\times$ & DRRQR & 75\% & \textbf{45.9} & \textbf{97.8} & 49.7 & 27.8 & 38.4 & 50.1 & 63.7 & 23.8 & \textbf{42.3} \\
$\times$ & Grad & 75\% & \underline{79.6} & \underline{215.6} & 41.5 & 27.9 & 37.4 & 50.0 & 62.9 & 20.1 & \underline{40.0} \\
$\times$ & S-Wanda & 75\% & 96.6 & 1205.5 & 37.2 & 26.5 & 35.5 & 50.4 & 60.6 & 8.8 & 36.5 \\
\addlinespace
\checkmark & L1 & 75\% & 32.2 & 60.0 & 49.6 & 27.4 & 38.2 & 49.3 & 64.6 & 25.2 & 42.4 \\
\checkmark & Rand & 75\% & 33.5 & 70.9 & 49.5 & 26.4 & 37.9 & 51.9 & 65.0 & 22.5 & 42.2 \\
\checkmark & DRRQR & 75\% & \textbf{31.4} & \underline{53.5} & 50.8 & 27.1 & 38.5 & 49.9 & 65.3 & 26.3 & \underline{43.0} \\
\checkmark & Grad & 75\% & \textbf{31.4} & \textbf{49.1} & 50.3 & 27.4 & 38.5 & 50.6 & 65.8 & 27.1 & \textbf{43.3} \\
\checkmark & S-Wanda & 75\% & \underline{32.1} & 58.2 & 50.0 & 27.0 & 38.0 & 49.3 & 64.5 & 25.6 & 42.4 \\
\midrule
$\times$ & L1 & 50\% & 41.8 & 156.6 & 45.6 & 28.3 & 38.8 & 47.8 & 64.0 & 17.9 & 40.4 \\
$\times$ & Rand & 50\% & 35.5 & 50.7 & 47.1 & 27.0 & 39.1 & 53.1 & 64.9 & 29.4 & 43.4 \\
$\times$ & DRRQR & 50\% & \textbf{31.6} & \underline{39.4} & 49.8 & 27.9 & 39.9 & 49.5 & 65.2 & 31.5 & \textbf{44.0} \\
$\times$ & Grad & 50\% & \underline{33.3} & \textbf{39.2} & 48.2 & 27.6 & 40.0 & 49.2 & 65.8 & 31.8 & \underline{43.8} \\
$\times$ & S-Wanda & 50\% & 40.0 & 141.5 & 45.5 & 28.7 & 38.9 & 49.6 & 64.0 & 18.9 & 40.9 \\
\addlinespace
\checkmark & L1 & 50\% & \underline{29.2} & 44.6 & 50.2 & 27.8 & 39.0 & 49.6 & 65.3 & 29.0 & 43.5 \\
\checkmark & Rand & 50\% & 29.3 & 46.5 & 50.3 & 26.6 & 38.9 & 51.9 & 66.1 & 28.2 & 43.6 \\
\checkmark & DRRQR & 50\% & \textbf{28.7} & \underline{40.6} & 50.9 & 27.5 & 39.2 & 50.4 & 65.3 & 29.8 & \textbf{43.8} \\
\checkmark & Grad & 50\% & \textbf{28.7} & \textbf{39.4} & 50.3 & 27.4 & 39.0 & 50.6 & 65.7 & 29.6 & \underline{43.7} \\
\checkmark & S-Wanda & 50\% & \underline{29.2} & 44.6 & 50.3 & 27.3 & 39.1 & 50.2 & 65.0 & 29.0 & 43.5 \\
\midrule
$\times$ & L1 & 40\% & 36.9 & 96.4 & 46.1 & 28.3 & 38.7 & 49.0 & 65.0 & 21.5 & 41.4 \\
$\times$ & Rand & 40\% & 32.2 & 42.2 & 48.8 & 27.3 & 39.2 & 53.0 & 65.3 & 30.5 & 44.0 \\
$\times$ & DRRQR & 40\% & \textbf{30.1} & \underline{38.3} & 50.5 & 28.0 & 40.0 & 49.6 & 65.4 & 31.7 & \underline{44.2} \\
$\times$ & Grad & 40\% & \underline{30.7} & \textbf{34.6} & 49.5 & 27.5 & 40.2 & 50.4 & 65.8 & 33.0 & \textbf{44.4} \\
$\times$ & S-Wanda & 40\% & 35.9 & 95.8 & 46.7 & 28.0 & 39.2 & 50.2 & 64.1 & 22.1 & 41.7 \\
\addlinespace
\checkmark & L1 & 40\% & 28.6 & 40.6 & 50.8 & 28.2 & 39.1 & 50.1 & 65.4 & 30.8 & \underline{44.0} \\
\checkmark & Rand & 40\% & 28.6 & 43.1 & 50.7 & 26.5 & 38.9 & 50.9 & 65.9 & 29.5 & 43.7 \\
\checkmark & DRRQR & 40\% & \textbf{28.2} & \underline{39.5} & 51.1 & 27.8 & 39.4 & 50.0 & 65.4 & 30.5 & \underline{44.0} \\
\checkmark & Grad & 40\% & \textbf{28.2} & \textbf{38.0} & 50.6 & 27.7 & 39.2 & 50.7 & 65.6 & 30.5 & \textbf{44.1} \\
\checkmark & S-Wanda & 40\% & \underline{28.5} & 40.6 & 50.9 & 27.8 & 38.8 & 49.6 & 65.2 & 30.8 & 43.9 \\
\midrule
$\times$ & L1 & 30\% & 33.9 & 65.1 & 47.7 & 28.2 & 39.4 & 49.3 & 64.9 & 25.1 & 42.4 \\
$\times$ & Rand & 30\% & \underline{29.9} & 38.4 & 48.9 & 26.5 & 39.6 & 51.3 & 65.2 & 31.3 & 43.8 \\
$\times$ & DRRQR & 30\% & \textbf{29.0} & \underline{36.6} & 51.1 & 28.2 & 39.9 & 50.4 & 65.3 & 32.4 & \textbf{44.5} \\
$\times$ & Grad & 30\% & \textbf{29.0} & \textbf{32.4} & 50.6 & 27.0 & 40.2 & 50.4 & 64.7 & 33.7 & \underline{44.4} \\
$\times$ & S-Wanda & 30\% & 33.2 & 66.5 & 47.4 & 27.7 & 39.5 & 50.0 & 65.2 & 25.3 & 42.5 \\
\addlinespace
\checkmark & L1 & 30\% & 28.1 & 38.5 & 50.6 & 27.9 & 39.1 & 49.9 & 65.7 & 31.1 & 44.1 \\
\checkmark & Rand & 30\% & 28.2 & 40.9 & 50.8 & 27.0 & 39.3 & 50.5 & 65.4 & 29.7 & 43.8 \\
\checkmark & DRRQR & 30\% & \textbf{27.8} & \underline{38.0} & 51.2 & 27.9 & 39.4 & 50.0 & 65.5 & 31.4 & \underline{44.2} \\
\checkmark & Grad & 30\% & \textbf{27.8} & \textbf{35.8} & 51.3 & 27.7 & 39.4 & 50.6 & 65.5 & 31.7 & \textbf{44.4} \\
\checkmark & S-Wanda & 30\% & \underline{28.0} & 38.8 & 50.5 & 28.0 & 39.1 & 49.0 & 65.3 & 30.7 & 43.8 \\
\midrule
\textbf{Baseline} & -- & 0\% & 28.8 & 35.9 & 51.5 & 28.2 & 39.6 & 50.4 & 65.5 & 31.5 & 44.4 \\
\bottomrule
\end{tabular}
}
\end{table}

\begin{table}[t]
\centering
\caption{Zero-shot performance of Gated DeltaNet 1.3B models, evaluated using the {lm-eval-harness}~\citep{eval-harness}, given different compression ratios. Pre- and post RFT.}
\label{tab:zs-gdn-13B}
\centering
\small
\setlength{\tabcolsep}{3.5pt}
\resizebox{\columnwidth}{!}{%
\begin{tabular}{lll|cc|cccccc|c}
\toprule
\thead[t]{RFT} & \thead[t]{Method} & \thead[t]{Compr.} & \makecell[t]{Wiki.~$\downarrow$\\ppl} & \makecell[t]{LMB.~$\downarrow$\\ppl} & \makecell[t]{ARC-e~$\uparrow$\\acc\_n} & \makecell[t]{ARC-c~$\uparrow$\\acc\_n} & \makecell[t]{Hella.~$\uparrow$\\acc\_n} & \makecell[t]{Wino.~$\uparrow$\\acc} & \makecell[t]{PIQA~$\uparrow$\\acc\_n} & \makecell[t]{LMB.~$\uparrow$\\acc} & \makecell[t]{Avg~$\uparrow$} \\
\midrule
$\times$ & L1 & 75\% & 68.8 & 279.5 & 54.4 & 32.8 & 47.2 & 55.4 & 66.2 & 15.9 & 45.3 \\
$\times$ & Rand & 75\% & 32.8 & 37.8 & 56.0 & 34.3 & 51.8 & 54.5 & 70.2 & 30.8 & 49.6 \\
$\times$ & DRRQR & 75\% & \underline{22.5} & \underline{23.6} & 62.6 & 36.5 & 55.7 & 59.8 & 72.1 & 37.0 & \underline{54.0} \\
$\times$ & Grad & 75\% & \textbf{22.3} & \textbf{17.7} & 64.0 & 38.3 & 57.1 & 60.2 & 71.4 & 42.5 & \textbf{55.6} \\
$\times$ & S-Wanda & 75\% & 30.6 & 77.0 & 59.7 & 36.3 & 54.3 & 58.9 & 69.9 & 24.6 & 50.6 \\
\addlinespace
\checkmark & L1 & 75\% & 18.6 & 16.8 & 66.1 & 37.8 & 57.0 & 58.0 & 72.4 & 40.7 & 55.3 \\
\checkmark & Rand & 75\% & 19.3 & 18.1 & 64.7 & 36.5 & 56.6 & 57.8 & 72.5 & 39.5 & 54.6 \\
\checkmark & DRRQR & 75\% & \underline{17.9} & \underline{14.7} & 64.0 & 37.8 & 58.4 & 61.1 & 73.1 & 43.1 & \underline{56.2} \\
\checkmark & Grad & 75\% & \textbf{17.8} & \textbf{12.5} & 65.1 & 37.5 & 58.4 & 60.9 & 72.9 & 45.8 & \textbf{56.8} \\
\checkmark & S-Wanda & 75\% & 18.2 & 15.1 & 63.9 & 37.7 & 57.5 & 59.0 & 72.6 & 41.7 & 55.4 \\
\midrule
$\times$ & L1 & 50\% & 26.5 & 34.0 & 64.8 & 36.5 & 56.3 & 58.2 & 71.6 & 32.3 & 53.3 \\
$\times$ & Rand & 50\% & 18.8 & 14.4 & 65.5 & 38.3 & 58.8 & 59.7 & 73.2 & 43.7 & 56.5 \\
$\times$ & DRRQR & 50\% & \textbf{17.3} & \underline{14.2} & 66.2 & 39.2 & 59.9 & 62.5 & 73.1 & 44.0 & \underline{57.5} \\
$\times$ & Grad & 50\% & \underline{17.4} & \textbf{10.1} & 65.9 & 39.8 & 60.5 & 61.2 & 72.7 & 51.3 & \textbf{58.6} \\
$\times$ & S-Wanda & 50\% & 21.9 & 24.1 & 65.1 & 38.5 & 58.7 & 60.0 & 72.1 & 36.1 & 55.1 \\
\addlinespace
\checkmark & L1 & 50\% & 16.8 & 13.0 & 66.0 & 39.2 & 59.2 & 59.8 & 73.6 & 45.3 & 57.2 \\
\checkmark & Rand & 50\% & 16.9 & \underline{11.5} & 66.6 & 38.9 & 58.8 & 61.2 & 73.6 & 47.8 & 57.8 \\
\checkmark & DRRQR & 50\% & \textbf{16.3} & 12.0 & 66.8 & 39.4 & 59.4 & 61.1 & 73.7 & 47.4 & \underline{58.0} \\
\checkmark & Grad & 50\% & \underline{16.4} & \textbf{10.5} & 65.7 & 39.3 & 59.6 & 61.1 & 73.9 & 49.9 & \textbf{58.3} \\
\checkmark & S-Wanda & 50\% & 16.6 & 11.7 & 66.6 & 39.2 & 59.4 & 60.5 & 73.8 & 47.3 & 57.8 \\
\midrule
$\times$ & L1 & 40\% & 23.7 & 25.3 & 64.3 & 38.5 & 57.8 & 58.6 & 71.7 & 36.1 & 54.5 \\
$\times$ & Rand & 40\% & \underline{18.3} & \underline{11.0} & 65.7 & 39.3 & 59.7 & 60.5 & 73.9 & 49.7 & 58.1 \\
$\times$ & DRRQR & 40\% & \textbf{16.7} & 11.9 & 67.2 & 39.4 & 60.5 & 60.9 & 73.5 & 47.8 & \underline{58.2} \\
$\times$ & Grad & 40\% & \textbf{16.7} & \textbf{9.4} & 67.8 & 40.5 & 60.6 & 62.3 & 73.5 & 52.6 & \textbf{59.6} \\
$\times$ & S-Wanda & 40\% & 20.7 & 21.6 & 65.4 & 39.6 & 59.5 & 60.3 & 72.4 & 37.4 & 55.8 \\
\addlinespace
\checkmark & L1 & 40\% & 16.5 & 11.9 & 67.2 & 38.8 & 59.6 & 59.8 & 74.0 & 47.5 & 57.8 \\
\checkmark & Rand & 40\% & 16.5 & \underline{10.7} & 66.4 & 39.2 & 59.6 & 61.3 & 73.8 & 49.5 & 58.3 \\
\checkmark & DRRQR & 40\% & \textbf{16.1} & 11.0 & 67.2 & 40.0 & 59.8 & 61.3 & 73.9 & 49.8 & \textbf{58.7} \\
\checkmark & Grad & 40\% & \textbf{16.1} & \textbf{10.1} & 66.7 & 39.4 & 59.9 & 61.2 & 74.2 & 50.5 & \underline{58.6} \\
\checkmark & S-Wanda & 40\% & \underline{16.4} & 11.5 & 66.8 & 39.5 & 59.9 & 60.8 & 74.2 & 47.7 & 58.1 \\
\midrule
$\times$ & L1 & 30\% & 22.2 & 24.3 & 64.5 & 38.1 & 58.4 & 61.0 & 71.9 & 35.4 & 54.9 \\
$\times$ & Rand & 30\% & \underline{16.7} & \underline{11.4} & 65.2 & 39.2 & 60.2 & 60.2 & 73.9 & 48.7 & 57.9 \\
$\times$ & DRRQR & 30\% & \textbf{16.3} & 11.6 & 67.0 & 39.7 & 60.5 & 61.0 & 73.7 & 47.5 & \underline{58.2} \\
$\times$ & Grad & 30\% & \textbf{16.3} & \textbf{9.2} & 68.5 & 40.4 & 60.7 & 62.4 & 73.7 & 52.8 & \textbf{59.7} \\
$\times$ & S-Wanda & 30\% & 19.8 & 18.4 & 65.5 & 39.8 & 60.1 & 61.5 & 73.1 & 39.7 & 56.6 \\
\addlinespace
\checkmark & L1 & 30\% & 16.3 & 11.3 & 67.5 & 39.2 & 59.7 & 61.2 & 74.1 & 48.3 & 58.3 \\
\checkmark & Rand & 30\% & \underline{16.1} & \underline{10.2} & 66.5 & 39.6 & 59.8 & 61.6 & 74.4 & 50.4 & 58.7 \\
\checkmark & DRRQR & 30\% & \textbf{15.9} & 10.7 & 67.3 & 39.9 & 59.9 & 61.8 & 74.3 & 49.7 & \underline{58.8} \\
\checkmark & Grad & 30\% & \textbf{15.9} & \textbf{9.9} & 66.9 & 40.4 & 60.0 & 61.6 & 74.4 & 50.9 & \textbf{59.0} \\
\checkmark & S-Wanda & 30\% & \underline{16.1} & 11.0 & 67.6 & 40.3 & 60.0 & 60.5 & 74.0 & 49.0 & 58.6 \\
\midrule
\textbf{Baseline} & -- & 0\% & 16.8 & 9.7 & 67.7 & 41.0 & 60.1 & 62.2 & 74.0 & 51.5 & 59.4 \\
\bottomrule
\end{tabular}
}
\end{table}

\newpage

\clearpage

\section{More Plots}
\begin{figure}[htbp]
    \centering
    \includegraphics[width=0.7\linewidth]{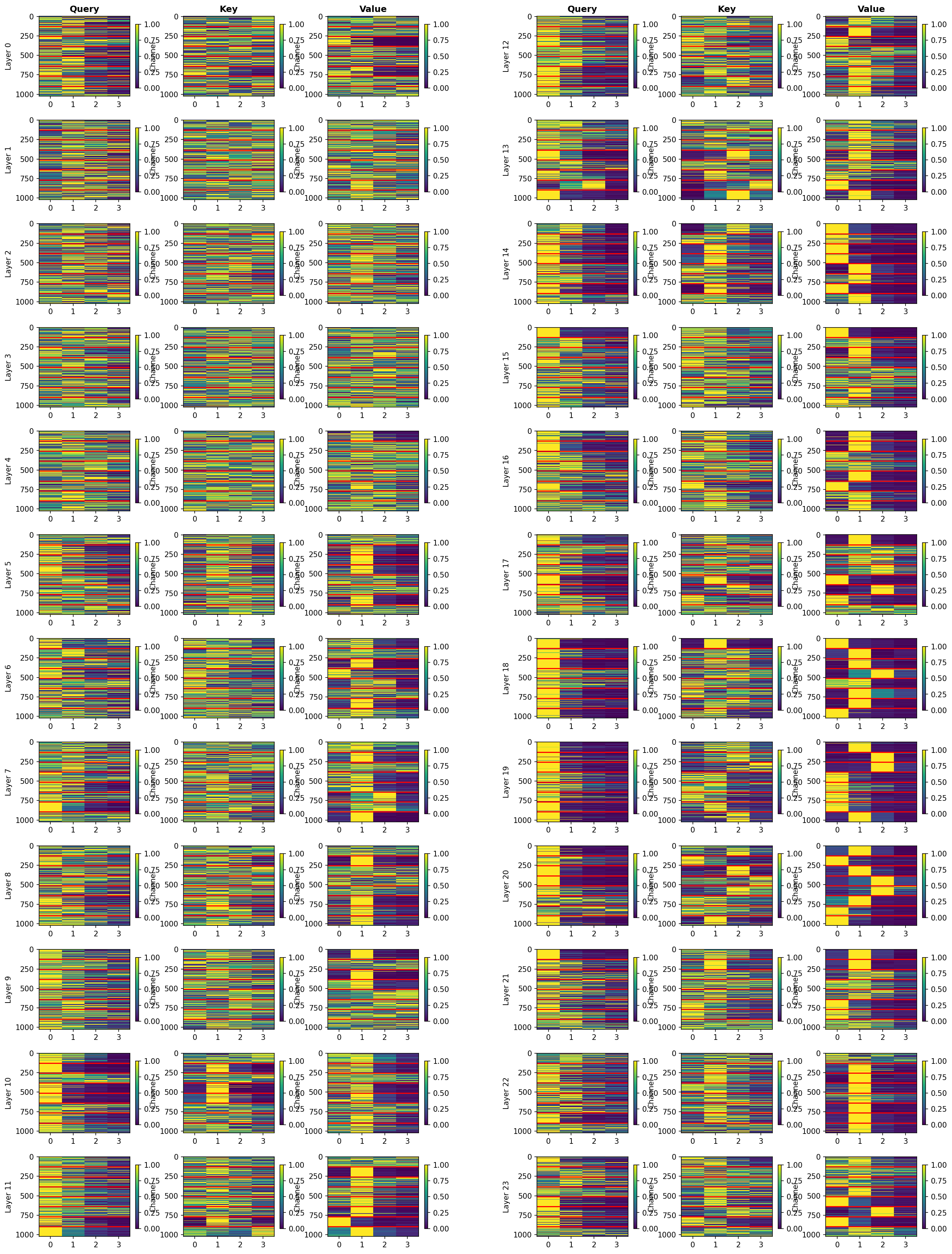}
    \caption{Convolution filters for queries, keys, and values of a standard DeltaNet 370M (non-shared). High similarity within heads (separated by red lines) suggests implicit sharing.}
    \label{fig:filters_dn}
\end{figure}

\begin{figure}[htbp]
    \centering
    \includegraphics[width=0.7\linewidth]{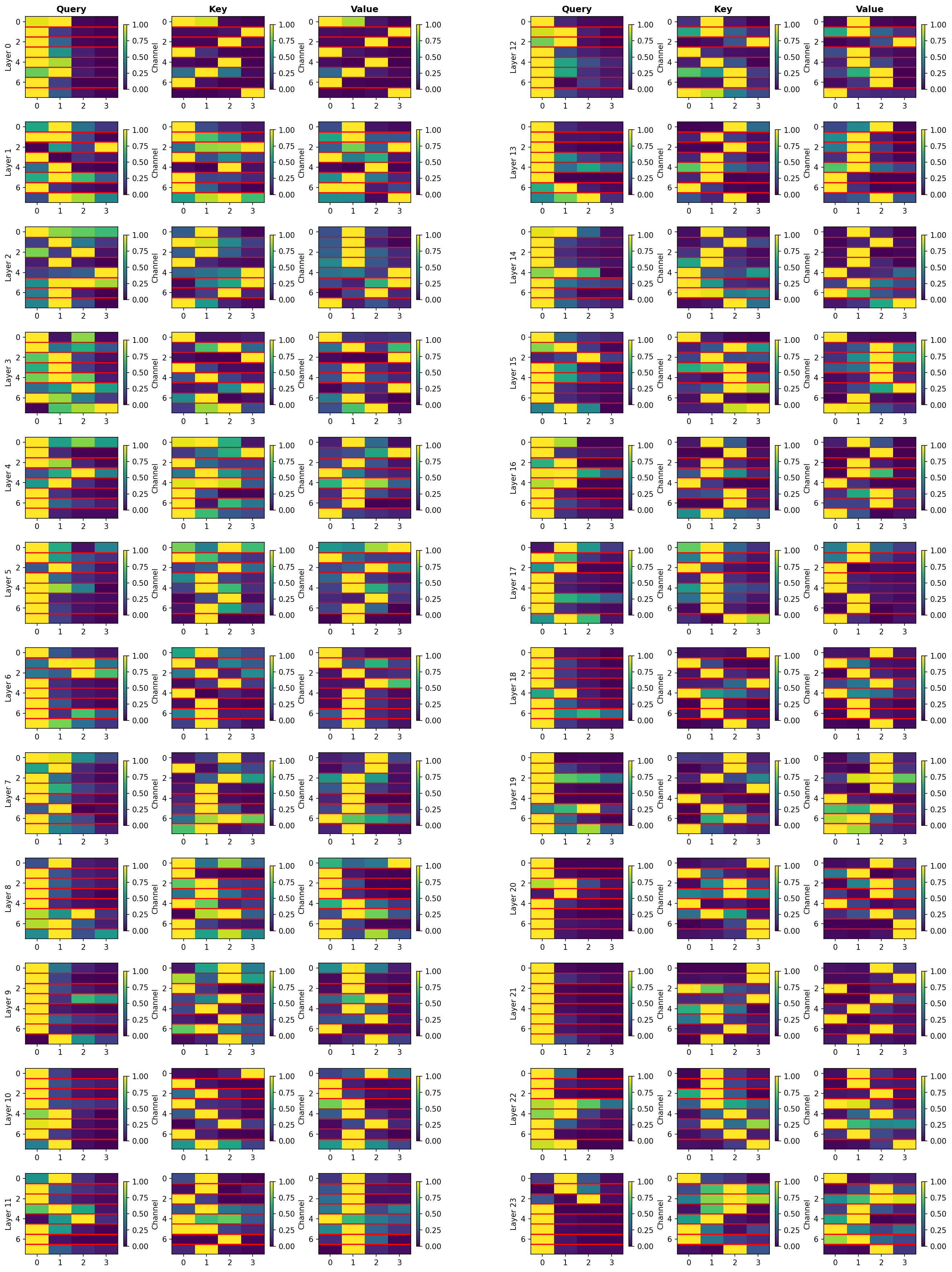}
    \caption{Convolution filters for DeltaNet 370M with explicitly shared convolutions. The model learns distinct patterns for Q, K, and V.}
    \label{fig:filters_dn2}
\end{figure}


\end{document}

%% file: vmr-symbols-vecbold.tex
\usepackage{amssymb}
\usepackage{amsfonts}
\usepackage{mathrsfs}
\usepackage{xspace}
\usepackage{bm}
\usepackage{upgreek}

\newcommand{\safemath}[2]{\newcommand{#1}{\ensuremath{#2}\xspace}}

\safemath{\bma}{\mathbf{a}}
\safemath{\bmb}{\mathbf{b}}
\safemath{\bmc}{\mathbf{c}}
\safemath{\bmd}{\mathbf{d}}
\safemath{\bme}{\mathbf{e}}
\safemath{\bmf}{\mathbf{f}}
\safemath{\bmg}{\mathbf{g}}
\safemath{\bmh}{\mathbf{h}}
\safemath{\bmi}{\mathbf{i}}
\safemath{\bmj}{\mathbf{j}}
\safemath{\bmk}{\mathbf{k}}
\safemath{\bml}{\mathbf{l}}
\safemath{\bmm}{\mathbf{m}}
\safemath{\bmn}{\mathbf{n}}
\safemath{\bmo}{\mathbf{o}}
\safemath{\bmp}{\mathbf{p}}
\safemath{\bmq}{\mathbf{q}}
\safemath{\bmr}{\mathbf{r}}
\safemath{\bms}{\mathbf{s}}
\safemath{\bmt}{\mathbf{t}}
\safemath{\bmu}{\mathbf{u}}
\safemath{\bmv}{\mathbf{v}}
\safemath{\bmw}{\mathbf{w}}
\safemath{\bmx}{\mathbf{x}}
\safemath{\bmy}{\mathbf{y}}
\safemath{\bmz}{\mathbf{z}}
\safemath{\bmzero}{\mathbf{0}}
\safemath{\bmone}{\mathbf{1}}

\safemath{\bmbeta}{\mathbf{beta}}

\bmdefine{\biad}{a}
\bmdefine{\bibd}{b}
\bmdefine{\bicd}{c}
\bmdefine{\bidd}{d}
\bmdefine{\bied}{e}
\bmdefine{\bifd}{f}
\bmdefine{\bigd}{g}
\bmdefine{\bihd}{h}
\bmdefine{\biid}{i}
\bmdefine{\bijd}{j}
\bmdefine{\bikd}{k}
\bmdefine{\bild}{l}
\bmdefine{\bimd}{m}
\bmdefine{\bind}{n}
\bmdefine{\biod}{o}
\bmdefine{\bipd}{p}
\bmdefine{\biqd}{q}
\bmdefine{\bird}{r}
\bmdefine{\bisd}{s}
\bmdefine{\bitd}{t}
\bmdefine{\biud}{u}
\bmdefine{\bivd}{v}
\bmdefine{\biwd}{w}
\bmdefine{\bixd}{x}
\bmdefine{\biyd}{y}
\bmdefine{\bizd}{z}

\bmdefine{\bixid}{\xi}
\bmdefine{\bilambdad}{\lambda}
\bmdefine{\bimud}{\mu}
\bmdefine{\bithetad}{\theta}
\bmdefine{\biphid}{\phi}

\safemath{\bmia}{\biad}
\safemath{\bmib}{\bibd}
\safemath{\bmic}{\bicd}
\safemath{\bmid}{\bidd}
\safemath{\bmie}{\bied}
\safemath{\bmif}{\bifd}
\safemath{\bmig}{\bigd}
\safemath{\bmih}{\bihd}
\safemath{\bmii}{\biid}
\safemath{\bmij}{\bijd}
\safemath{\bmik}{\bikd}
\safemath{\bmil}{\bild}
\safemath{\bmim}{\bimd}
\safemath{\bmin}{\bind}
\safemath{\bmio}{\biod}
\safemath{\bmip}{\bipd}
\safemath{\bmiq}{\biqd}
\safemath{\bmir}{\bird}
\safemath{\bmis}{\bisd}
\safemath{\bmit}{\bitd}
\safemath{\bmiu}{\biud}
\safemath{\bmiv}{\bivd}
\safemath{\bmiw}{\biwd}
\safemath{\bmix}{\bixd}
\safemath{\bmiy}{\biyd}
\safemath{\bmiz}{\bizd}

\safemath{\bmxi}{\bixid}
\safemath{\bmlambda}{\bilambdad}
\safemath{\bmmu}{\bimud}
\safemath{\bmtheta}{\bithetad}
\safemath{\bmphi}{\biphid}

\safemath{\bA}{\mathbf{A}}
\safemath{\bB}{\mathbf{B}}
\safemath{\bC}{\mathbf{C}}
\safemath{\bD}{\mathbf{D}}
\safemath{\bE}{\mathbf{E}}
\safemath{\bF}{\mathbf{F}}
\safemath{\bG}{\mathbf{G}}
\safemath{\bH}{\mathbf{H}}
\safemath{\bI}{\mathbf{I}}
\safemath{\bJ}{\mathbf{J}}
\safemath{\bK}{\mathbf{K}}
\safemath{\bL}{\mathbf{L}}
\safemath{\bM}{\mathbf{M}}
\safemath{\bN}{\mathbf{N}}
\safemath{\bO}{\mathbf{O}}
\safemath{\bP}{\mathbf{P}}
\safemath{\bQ}{\mathbf{Q}}
\safemath{\bR}{\mathbf{R}}
\safemath{\bS}{\mathbf{S}}
\safemath{\bT}{\mathbf{T}}
\safemath{\bU}{\mathbf{U}}
\safemath{\bV}{\mathbf{V}}
\safemath{\bW}{\mathbf{W}}
\safemath{\bX}{\mathbf{X}}
\safemath{\bY}{\mathbf{Y}}
\safemath{\bZ}{\mathbf{Z}}

\safemath{\bZero}{\mathbf{0}}
\safemath{\bOne}{\mathbf{1}}
\safemath{\bDelta}{\mathbf{\Delta}}
\safemath{\bLambda}{\mathbf{\UpLambda}}
\safemath{\bPhi}{\mathbf{\Upphi}}
\safemath{\bSigma}{\mathbf{\Upsigma}}
\safemath{\bOmega}{\mathbf{\Upomega}}
\safemath{\bTheta}{\mathbf{\Uptheta}}

\bmdefine{\biAd}{A}
\bmdefine{\biBd}{B}
\bmdefine{\biCd}{C}
\bmdefine{\biDd}{D}
\bmdefine{\biEd}{E}
\bmdefine{\biFd}{F}
\bmdefine{\biGd}{G}
\bmdefine{\biHd}{H}
\bmdefine{\biId}{I}
\bmdefine{\biJd}{J}
\bmdefine{\biKd}{K}
\bmdefine{\biLd}{L}
\bmdefine{\biMd}{M}
\bmdefine{\biOd}{N}
\bmdefine{\biPd}{O}
\bmdefine{\biQd}{P}
\bmdefine{\biRd}{R}
\bmdefine{\biSd}{S}
\bmdefine{\biTd}{T}
\bmdefine{\biUd}{U}
\bmdefine{\biVd}{V}
\bmdefine{\biWd}{W}
\bmdefine{\biXd}{X}
\bmdefine{\biYd}{Y}
\bmdefine{\biZd}{Z}

\bmdefine{\biDelta}{\Delta}
\bmdefine{\biLambda}{\Lambda}
\bmdefine{\biPhi}{\Phi}
\bmdefine{\biSigma}{\Sigma}
\bmdefine{\biOmega}{\Omega}
\bmdefine{\biTheta}{\Theta}

\safemath{\bimA}{\biAd}
\safemath{\bimB}{\biBd}
\safemath{\bimC}{\biCd}
\safemath{\bimD}{\biDd}
\safemath{\bimE}{\biEd}
\safemath{\bimF}{\biFd}
\safemath{\bimG}{\biGd}
\safemath{\bimH}{\biHd}
\safemath{\bimI}{\biId}
\safemath{\bimJ}{\biJd}
\safemath{\bimK}{\biKd}
\safemath{\bimL}{\biLd}
\safemath{\bimM}{\biMd}
\safemath{\bimN}{\biNd}
\safemath{\bimO}{\biOd}
\safemath{\bimP}{\biPd}
\safemath{\bimQ}{\biQd}
\safemath{\bimR}{\biRd}
\safemath{\bimS}{\biSd}
\safemath{\bimT}{\biTd}
\safemath{\bimU}{\biUd}
\safemath{\bimV}{\biVd}
\safemath{\bimW}{\biWd}
\safemath{\bimX}{\biXd}
\safemath{\bimY}{\biYd}
\safemath{\bimZ}{\biZd}

\safemath{\bimDelta}{\biDelta}
\safemath{\bimLambda}{\biLambda}
\safemath{\bimPhi}{\biPhi}
\safemath{\bimSigma}{\biSigma}
\safemath{\bimOmega}{\biOmega}
\safemath{\bimTheta}{\biTheta}

\safemath{\setA}{\mathcal{A}}
\safemath{\setB}{\mathcal{B}}
\safemath{\setC}{\mathcal{C}}
\safemath{\setD}{\mathcal{D}}
\safemath{\setE}{\mathcal{E}}
\safemath{\setF}{\mathcal{F}}
\safemath{\setG}{\mathcal{G}}
\safemath{\setH}{\mathcal{H}}
\safemath{\setI}{\mathcal{I}}
\safemath{\setJ}{\mathcal{J}}
\safemath{\setK}{\mathcal{K}}
\safemath{\setL}{\mathcal{L}}
\safemath{\setM}{\mathcal{M}}
\safemath{\setN}{\mathcal{N}}
\safemath{\setO}{\mathcal{O}}
\safemath{\setP}{\mathcal{P}}
\safemath{\setQ}{\mathcal{Q}}
\safemath{\setR}{\mathcal{R}}
\safemath{\setS}{\mathcal{S}}
\safemath{\setT}{\mathcal{T}}
\safemath{\setU}{\mathcal{U}}
\safemath{\setV}{\mathcal{V}}
\safemath{\setW}{\mathcal{W}}
\safemath{\setX}{\mathcal{X}}
\safemath{\setY}{\mathcal{Y}}
\safemath{\setZ}{\mathcal{Z}}
\safemath{\emptySet}{\varnothing}

\safemath{\colA}{\mathscr{A}}
\safemath{\colB}{\mathscr{B}}
\safemath{\colC}{\mathscr{C}}
\safemath{\colD}{\mathscr{D}}
\safemath{\colE}{\mathscr{E}}
\safemath{\colF}{\mathscr{F}}
\safemath{\colG}{\mathscr{G}}
\safemath{\colH}{\mathscr{H}}
\safemath{\colI}{\mathscr{I}}
\safemath{\colJ}{\mathscr{J}}
\safemath{\colK}{\mathscr{K}}
\safemath{\colL}{\mathscr{L}}
\safemath{\colM}{\mathscr{M}}
\safemath{\colN}{\mathscr{N}}
\safemath{\colO}{\mathscr{O}}
\safemath{\colP}{\mathscr{P}}
\safemath{\colQ}{\mathscr{Q}}
\safemath{\colR}{\mathscr{R}}
\safemath{\colS}{\mathscr{S}}
\safemath{\colT}{\mathscr{T}}
\safemath{\colU}{\mathscr{U}}
\safemath{\colV}{\mathscr{V}}
\safemath{\colW}{\mathscr{W}}
\safemath{\colX}{\mathscr{X}}
\safemath{\colY}{\mathscr{Y}}
\safemath{\colZ}{\mathscr{Z}}

\safemath{\opA}{\mathbb{A}}
\safemath{\opB}{\mathbb{B}}
\safemath{\opC}{\mathbb{C}}
\safemath{\opD}{\mathbb{D}}
\safemath{\opE}{\mathbb{E}}
\safemath{\opF}{\mathbb{F}}
\safemath{\opG}{\mathbb{G}}
\safemath{\opH}{\mathbb{H}}
\safemath{\opI}{\mathbb{I}}
\safemath{\opJ}{\mathbb{J}}
\safemath{\opK}{\mathbb{K}}
\safemath{\opL}{\mathbb{L}}
\safemath{\opM}{\mathbb{M}}
\safemath{\opN}{\mathbb{N}}
\safemath{\opO}{\mathbb{O}}
\safemath{\opP}{\mathbb{P}}
\safemath{\opQ}{\mathbb{Q}}
\safemath{\opR}{\mathbb{R}}
\safemath{\opS}{\mathbb{S}}
\safemath{\opT}{\mathbb{T}}
\safemath{\opU}{\mathbb{U}}
\safemath{\opV}{\mathbb{V}}
\safemath{\opW}{\mathbb{W}}
\safemath{\opX}{\mathbb{X}}
\safemath{\opY}{\mathbb{Y}}
\safemath{\opZ}{\mathbb{Z}}
\safemath{\opZero}{\mathbb{O}}
\safemath{\identityop}{\opI}

\safemath{\veca}{\bma}
\safemath{\vecb}{\bmb}
\safemath{\vecc}{\bmc}
\safemath{\vecd}{\bmd}
\safemath{\vece}{\bme}
\safemath{\vecf}{\bmf}
\safemath{\vecg}{\bmg}
\safemath{\vech}{\bmh}
\safemath{\veci}{\bmi}
\safemath{\vecj}{\bmj}
\safemath{\veck}{\bmk}
\safemath{\vecl}{\bml}
\safemath{\vecm}{\bmm}
\safemath{\vecn}{\bmn}
\safemath{\veco}{\bmo}
\safemath{\vecp}{\bmmp}
\safemath{\vecq}{\bmq}
\safemath{\vecr}{\bmr}
\safemath{\vecs}{\bms}
\safemath{\vect}{\bmt}
\safemath{\vecu}{\bmu}
\safemath{\vecv}{\bmv}
\safemath{\vecw}{\bmw}
\safemath{\vecx}{\bmx}
\safemath{\vecy}{\bmy}
\safemath{\vecz}{\bmz}

\safemath{\veczero}{\bmzero}
\safemath{\vecone}{\bmone}
\safemath{\vecxi}{\bmxi}
\safemath{\veclambda}{\bmlambda}
\safemath{\vecmu}{\bmmu}
\safemath{\vectheta}{\bmtheta}
\safemath{\vecphi}{\bmphi}

\safemath{\matA}{\bA}
\safemath{\matB}{\bB}
\safemath{\matC}{\bC}
\safemath{\matD}{\bD}
\safemath{\matE}{\bE}
\safemath{\matF}{\bF}
\safemath{\matG}{\bG}
\safemath{\matH}{\bH}
\safemath{\matI}{\bI}
\safemath{\matJ}{\bJ}
\safemath{\matK}{\bK}
\safemath{\matL}{\bL}
\safemath{\matM}{\bM}
\safemath{\matN}{\bN}
\safemath{\matO}{\bO}
\safemath{\matP}{\bP}
\safemath{\matQ}{\bQ}
\safemath{\matR}{\bR}
\safemath{\matS}{\bS}
\safemath{\matT}{\bT}
\safemath{\matU}{\bU}
\safemath{\matV}{\bV}
\safemath{\matW}{\bW}
\safemath{\matX}{\bX}
\safemath{\matY}{\bY}
\safemath{\matZ}{\bZ}
\safemath{\matzero}{\bmzero}

\safemath{\matDelta}{\bDelta}
\safemath{\matLambda}{\bLambda}
\safemath{\matPhi}{\bPhi}
\safemath{\matSigma}{\bSigma}
\safemath{\matOmega}{\bOmega}
\safemath{\matTheta}{\bTheta}

\safemath{\matidentity}{\matI}
\safemath{\matone}{\matO}

\safemath{\rnda}{A}
\safemath{\rndb}{B}
\safemath{\rndc}{C}
\safemath{\rndd}{D}
\safemath{\rnde}{E}
\safemath{\rndf}{F}
\safemath{\rndg}{G}
\safemath{\rndh}{H}
\safemath{\rndi}{I}
\safemath{\rndj}{J}
\safemath{\rndk}{K}
\safemath{\rndl}{L}
\safemath{\rndm}{M}
\safemath{\rndn}{N}
\safemath{\rndo}{O}
\safemath{\rndp}{P}
\safemath{\rndq}{Q}
\safemath{\rndr}{R}
\safemath{\rnds}{S}
\safemath{\rndt}{T}
\safemath{\rndu}{U}
\safemath{\rndv}{V}
\safemath{\rndw}{W}
\safemath{\rndx}{X}
\safemath{\rndy}{Y}
\safemath{\rndz}{Z}

\safemath{\rveca}{\bimA}
\safemath{\rvecb}{\bimB}
\safemath{\rvecc}{\bimC}
\safemath{\rvecd}{\bimD}
\safemath{\rvece}{\bimE}
\safemath{\rvecf}{\bimF}
\safemath{\rvecg}{\bimG}
\safemath{\rvech}{\bimH}
\safemath{\rveci}{\bimI}
\safemath{\rvecj}{\bimJ}
\safemath{\rveck}{\bimK}
\safemath{\rvecl}{\bimL}
\safemath{\rvecm}{\bimM}
\safemath{\rvecn}{\bimN}
\safemath{\rveco}{\bomO}
\safemath{\rvecp}{\bimP}
\safemath{\rvecq}{\bimQ}
\safemath{\rvecr}{\bimR}
\safemath{\rvecs}{\bimS}
\safemath{\rvect}{\bimT}
\safemath{\rvecu}{\bimU}
\safemath{\rvecv}{\bimV}
\safemath{\rvecw}{\bimW}
\safemath{\rvecx}{\bimX}
\safemath{\rvecy}{\bimY}
\safemath{\rvecz}{\bimZ}

\safemath{\rvecxi}{\bmxi}
\safemath{\rveclambda}{\bmlambda}
\safemath{\rvecmu}{\bmmu}
\safemath{\rvectheta}{\bmtheta}
\safemath{\rvecphi}{\bmphi}

\safemath{\rmatA}{\bimA}
\safemath{\rmatB}{\bimB}
\safemath{\rmatC}{\bimC}
\safemath{\rmatD}{\bimD}
\safemath{\rmatE}{\bimE}
\safemath{\rmatF}{\bimF}
\safemath{\rmatG}{\bimG}
\safemath{\rmatH}{\bimH}
\safemath{\rmatI}{\bimI}
\safemath{\rmatJ}{\bimJ}
\safemath{\rmatK}{\bimK}
\safemath{\rmatL}{\bimL}
\safemath{\rmatM}{\bimM}
\safemath{\rmatN}{\bimN}
\safemath{\rmatO}{\bimO}
\safemath{\rmatP}{\bimP}
\safemath{\rmatQ}{\bimQ}
\safemath{\rmatR}{\bimR}
\safemath{\rmatS}{\bimS}
\safemath{\rmatT}{\bimT}
\safemath{\rmatU}{\bimU}
\safemath{\rmatV}{\bimV}
\safemath{\rmatW}{\bimW}
\safemath{\rmatX}{\bimX}
\safemath{\rmatY}{\bimY}
\safemath{\rmatZ}{\bimZ}

\safemath{\rmatDelta}{\bimDelta}
\safemath{\rmatLambda}{\bimLambda}
\safemath{\rmatPhi}{\bimPhi}
\safemath{\rmatSigma}{\bimSigma}
\safemath{\rmatOmega}{\bimOmega}
\safemath{\rmatTheta}{\bimTheta}

%% file: standard-macros.tex
\usepackage{amssymb}
\usepackage{amsfonts}
\usepackage{mathrsfs}
\usepackage{xspace}
\usepackage{bm}

\usepackage{scalerel}

\DeclareMathOperator{\row}{row}

\newenvironment{textbmatrix}{	\setlength{\arraycolsep}{2.5pt}%
								\big[\begin{matrix}}{\end{matrix}\big]%
								\raisebox{0.08ex}{\vphantom{M}}}

\def\be{\begin{equation}}
\def\ee{\end{equation}}
\def\een{\nonumber \end{equation}}
\def\mat{\begin{bmatrix}}
\def\emat{\end{bmatrix}}
\def\btm{\begin{textbmatrix}}
\def\etm{\end{textbmatrix}}

\def\ba#1\ea{\begin{align}#1\end{align}}
\def\bas#1\eas{\begin{align*}#1\end{align*}}
\def\bs#1\es{\begin{split}#1\end{split}} 
\def\bg#1\eg{\begin{gather}#1\end{gather}} 
\def\bi#1\ei{\begin{itemize}#1\end{itemize}}

\DeclareMathOperator{\rank}{rank}			

\safemath{\dirac}{\delta}					
\safemath{\krond}{\dirac}					

\safemath{\upto}{\uparrow}
\safemath{\downto}{\downarrow}
\safemath{\iu}{j}							
\safemath{\ev}{\lambda}						
\safemath{\hilseqspace}{l^{2}}				
\newcommand{\banachfunspace}[1]{\setL^{#1}}	
\safemath{\hilfunspace}{\banachfunspace{2}}	

\safemath{\SNR}{\text{\sc snr}} 				
\safemath{\No}{N_0}							
\safemath{\Es}{E_s}							
\safemath{\Eb}{E_b}							
\safemath{\EbNo}{\frac{\Eb}{\No}}
\safemath{\EsNo}{\frac{\Es}{\No}}

\DeclareMathOperator{\CHop}{\ensuremath{\opH}} 
\safemath{\tvir}{\rndh_{\CHop}}				
\safemath{\tvtf}{\rndl_{\CHop}}				
\safemath{\spf}{\rnds_{\CHop}}				
\safemath{\bff}{H_{\CHop}}					

\safemath{\ircf}{r_{h}}						
\safemath{\tftvcf}{r_{s}}					
\safemath{\tfcf}{r_{l}}						
\safemath{\bfcf}{r_{H}}						

\safemath{\tcorr}{c_h}						
\safemath{\scf}{c_{s}}						
\safemath{\tfcorr}{c_{l}}					
\safemath{\fcorr}{c_{H}}						

\safemath{\mi}{I}							
\safemath{\capacity}{C}						

\safemath{\normal}{\mathcal{N}}			
\safemath{\jpg}{\mathcal{CN}}			
\safemath{\mchain}{\leftrightarrow}		

\safemath{\dB}{\,\mathrm{dB}}
\safemath{\dBm}{\,\mathrm{dBm}}
\safemath{\Hz}{\,\mathrm{Hz}}
\safemath{\kHz}{\,\mathrm{kHz}}
\safemath{\MHz}{\,\mathrm{MHz}}
\safemath{\GHz}{\,\mathrm{GHz}}
\safemath{\s}{\,\mathrm{s}}
\safemath{\ms}{\,\mathrm{ms}}
\safemath{\mus}{\,\mathrm{\mu s}}
\safemath{\ns}{\,\mathrm{ns}}
\safemath{\meter}{\,\mathrm{m}}
\safemath{\mm}{\,\mathrm{mm}}
\safemath{\cm}{\,\mathrm{cm}}
\safemath{\m}{\,\mathrm{m}}
\safemath{\W}{\,\mathrm{W}}
\safemath{\J}{\,\mathrm{J}}
\safemath{\K}{\,\mathrm{K}}
\safemath{\bit}{\,\mathrm{bit}}

\safemath{\define}{=}			

\safemath{\equivalent}{\sim}
\safemath{\distas}{\sim}					
\safemath{\sdiff}{\Delta}				

\safemath{\reals}{\mathbb{R}}
\safemath{\positivereals}{\reals_{+}}
\safemath{\integers}{\mathbb{Z}}
\safemath{\posint}{\integers_{+}}
\safemath{\naturals}{\mathbb{N}}
\safemath{\posnaturals}{\naturals_{+}}
\safemath{\complexset}{\mathbb{C}}
\safemath{\rationals}{\mathbb{Q}}

%% file: media/svd_shared_range_L7_skip128.pgf
\begingroup%
\makeatletter%
\begin{pgfpicture}%
\pgfpathrectangle{\pgfpointorigin}{\pgfqpoint{2.750000in}{1.510000in}}%
\pgfusepath{use as bounding box, clip}%
\begin{pgfscope}%
\pgfsetbuttcap%
\pgfsetmiterjoin%
\definecolor{currentfill}{rgb}{1.000000,1.000000,1.000000}%
\pgfsetfillcolor{currentfill}%
\pgfsetlinewidth{0.000000pt}%
\definecolor{currentstroke}{rgb}{1.000000,1.000000,1.000000}%
\pgfsetstrokecolor{currentstroke}%
\pgfsetdash{}{0pt}%
\pgfpathmoveto{\pgfqpoint{0.000000in}{0.000000in}}%
\pgfpathlineto{\pgfqpoint{2.750000in}{0.000000in}}%
\pgfpathlineto{\pgfqpoint{2.750000in}{1.510000in}}%
\pgfpathlineto{\pgfqpoint{0.000000in}{1.510000in}}%
\pgfpathlineto{\pgfqpoint{0.000000in}{0.000000in}}%
\pgfpathclose%
\pgfusepath{fill}%
\end{pgfscope}%
\begin{pgfscope}%
\pgfsetbuttcap%
\pgfsetmiterjoin%
\definecolor{currentfill}{rgb}{1.000000,1.000000,1.000000}%
\pgfsetfillcolor{currentfill}%
\pgfsetlinewidth{0.000000pt}%
\definecolor{currentstroke}{rgb}{0.000000,0.000000,0.000000}%
\pgfsetstrokecolor{currentstroke}%
\pgfsetstrokeopacity{0.000000}%
\pgfsetdash{}{0pt}%
\pgfpathmoveto{\pgfqpoint{0.420000in}{0.320000in}}%
\pgfpathlineto{\pgfqpoint{1.430000in}{0.320000in}}%
\pgfpathlineto{\pgfqpoint{1.430000in}{1.330000in}}%
\pgfpathlineto{\pgfqpoint{0.420000in}{1.330000in}}%
\pgfpathlineto{\pgfqpoint{0.420000in}{0.320000in}}%
\pgfpathclose%
\pgfusepath{fill}%
\end{pgfscope}%
\begin{pgfscope}%
\pgfpathrectangle{\pgfqpoint{0.420000in}{0.320000in}}{\pgfqpoint{1.010000in}{1.010000in}}%
\pgfusepath{clip}%
\pgfsetrectcap%
\pgfsetroundjoin%
\pgfsetlinewidth{0.501875pt}%
\definecolor{currentstroke}{rgb}{0.886275,0.882353,0.866667}%
\pgfsetstrokecolor{currentstroke}%
\pgfsetdash{}{0pt}%
\pgfpathmoveto{\pgfqpoint{0.420000in}{0.320000in}}%
\pgfpathlineto{\pgfqpoint{0.420000in}{1.330000in}}%
\pgfusepath{stroke}%
\end{pgfscope}%
\begin{pgfscope}%
\pgfsetbuttcap%
\pgfsetroundjoin%
\definecolor{currentfill}{rgb}{0.321569,0.317647,0.305882}%
\pgfsetfillcolor{currentfill}%
\pgfsetlinewidth{0.602250pt}%
\definecolor{currentstroke}{rgb}{0.321569,0.317647,0.305882}%
\pgfsetstrokecolor{currentstroke}%
\pgfsetdash{}{0pt}%
\pgfsys@defobject{currentmarker}{\pgfqpoint{0.000000in}{-0.034722in}}{\pgfqpoint{0.000000in}{0.000000in}}{%
\pgfpathmoveto{\pgfqpoint{0.000000in}{0.000000in}}%
\pgfpathlineto{\pgfqpoint{0.000000in}{-0.034722in}}%
\pgfusepath{stroke,fill}%
}%
\begin{pgfscope}%
\pgfsys@transformshift{0.420000in}{0.320000in}%
\pgfsys@useobject{currentmarker}{}%
\end{pgfscope}%
\end{pgfscope}%
\begin{pgfscope}%
\definecolor{textcolor}{rgb}{0.101961,0.101961,0.098039}%
\pgfsetstrokecolor{textcolor}%
\pgfsetfillcolor{textcolor}%
\pgftext[x=0.420000in,y=0.236667in,,top]{\color{textcolor}{\rmfamily\fontsize{8.000000}{9.600000}\selectfont\catcode`\^=\active\def^{\ifmmode\sp\else\^{}\fi}\catcode`\%=\active\def
\end{pgfscope}%
\begin{pgfscope}%
\pgfpathrectangle{\pgfqpoint{0.420000in}{0.320000in}}{\pgfqpoint{1.010000in}{1.010000in}}%
\pgfusepath{clip}%
\pgfsetrectcap%
\pgfsetroundjoin%
\pgfsetlinewidth{0.501875pt}%
\definecolor{currentstroke}{rgb}{0.886275,0.882353,0.866667}%
\pgfsetstrokecolor{currentstroke}%
\pgfsetdash{}{0pt}%
\pgfpathmoveto{\pgfqpoint{0.622000in}{0.320000in}}%
\pgfpathlineto{\pgfqpoint{0.622000in}{1.330000in}}%
\pgfusepath{stroke}%
\end{pgfscope}%
\begin{pgfscope}%
\pgfsetbuttcap%
\pgfsetroundjoin%
\definecolor{currentfill}{rgb}{0.321569,0.317647,0.305882}%
\pgfsetfillcolor{currentfill}%
\pgfsetlinewidth{0.602250pt}%
\definecolor{currentstroke}{rgb}{0.321569,0.317647,0.305882}%
\pgfsetstrokecolor{currentstroke}%
\pgfsetdash{}{0pt}%
\pgfsys@defobject{currentmarker}{\pgfqpoint{0.000000in}{-0.034722in}}{\pgfqpoint{0.000000in}{0.000000in}}{%
\pgfpathmoveto{\pgfqpoint{0.000000in}{0.000000in}}%
\pgfpathlineto{\pgfqpoint{0.000000in}{-0.034722in}}%
\pgfusepath{stroke,fill}%
}%
\begin{pgfscope}%
\pgfsys@transformshift{0.622000in}{0.320000in}%
\pgfsys@useobject{currentmarker}{}%
\end{pgfscope}%
\end{pgfscope}%
\begin{pgfscope}%
\pgfpathrectangle{\pgfqpoint{0.420000in}{0.320000in}}{\pgfqpoint{1.010000in}{1.010000in}}%
\pgfusepath{clip}%
\pgfsetrectcap%
\pgfsetroundjoin%
\pgfsetlinewidth{0.501875pt}%
\definecolor{currentstroke}{rgb}{0.886275,0.882353,0.866667}%
\pgfsetstrokecolor{currentstroke}%
\pgfsetdash{}{0pt}%
\pgfpathmoveto{\pgfqpoint{0.824000in}{0.320000in}}%
\pgfpathlineto{\pgfqpoint{0.824000in}{1.330000in}}%
\pgfusepath{stroke}%
\end{pgfscope}%
\begin{pgfscope}%
\pgfsetbuttcap%
\pgfsetroundjoin%
\definecolor{currentfill}{rgb}{0.321569,0.317647,0.305882}%
\pgfsetfillcolor{currentfill}%
\pgfsetlinewidth{0.602250pt}%
\definecolor{currentstroke}{rgb}{0.321569,0.317647,0.305882}%
\pgfsetstrokecolor{currentstroke}%
\pgfsetdash{}{0pt}%
\pgfsys@defobject{currentmarker}{\pgfqpoint{0.000000in}{-0.034722in}}{\pgfqpoint{0.000000in}{0.000000in}}{%
\pgfpathmoveto{\pgfqpoint{0.000000in}{0.000000in}}%
\pgfpathlineto{\pgfqpoint{0.000000in}{-0.034722in}}%
\pgfusepath{stroke,fill}%
}%
\begin{pgfscope}%
\pgfsys@transformshift{0.824000in}{0.320000in}%
\pgfsys@useobject{currentmarker}{}%
\end{pgfscope}%
\end{pgfscope}%
\begin{pgfscope}%
\definecolor{textcolor}{rgb}{0.101961,0.101961,0.098039}%
\pgfsetstrokecolor{textcolor}%
\pgfsetfillcolor{textcolor}%
\pgftext[x=0.824000in,y=0.236667in,,top]{\color{textcolor}{\rmfamily\fontsize{8.000000}{9.600000}\selectfont\catcode`\^=\active\def^{\ifmmode\sp\else\^{}\fi}\catcode`\%=\active\def
\end{pgfscope}%
\begin{pgfscope}%
\pgfpathrectangle{\pgfqpoint{0.420000in}{0.320000in}}{\pgfqpoint{1.010000in}{1.010000in}}%
\pgfusepath{clip}%
\pgfsetrectcap%
\pgfsetroundjoin%
\pgfsetlinewidth{0.501875pt}%
\definecolor{currentstroke}{rgb}{0.886275,0.882353,0.866667}%
\pgfsetstrokecolor{currentstroke}%
\pgfsetdash{}{0pt}%
\pgfpathmoveto{\pgfqpoint{1.026000in}{0.320000in}}%
\pgfpathlineto{\pgfqpoint{1.026000in}{1.330000in}}%
\pgfusepath{stroke}%
\end{pgfscope}%
\begin{pgfscope}%
\pgfsetbuttcap%
\pgfsetroundjoin%
\definecolor{currentfill}{rgb}{0.321569,0.317647,0.305882}%
\pgfsetfillcolor{currentfill}%
\pgfsetlinewidth{0.602250pt}%
\definecolor{currentstroke}{rgb}{0.321569,0.317647,0.305882}%
\pgfsetstrokecolor{currentstroke}%
\pgfsetdash{}{0pt}%
\pgfsys@defobject{currentmarker}{\pgfqpoint{0.000000in}{-0.034722in}}{\pgfqpoint{0.000000in}{0.000000in}}{%
\pgfpathmoveto{\pgfqpoint{0.000000in}{0.000000in}}%
\pgfpathlineto{\pgfqpoint{0.000000in}{-0.034722in}}%
\pgfusepath{stroke,fill}%
}%
\begin{pgfscope}%
\pgfsys@transformshift{1.026000in}{0.320000in}%
\pgfsys@useobject{currentmarker}{}%
\end{pgfscope}%
\end{pgfscope}%
\begin{pgfscope}%
\pgfpathrectangle{\pgfqpoint{0.420000in}{0.320000in}}{\pgfqpoint{1.010000in}{1.010000in}}%
\pgfusepath{clip}%
\pgfsetrectcap%
\pgfsetroundjoin%
\pgfsetlinewidth{0.501875pt}%
\definecolor{currentstroke}{rgb}{0.886275,0.882353,0.866667}%
\pgfsetstrokecolor{currentstroke}%
\pgfsetdash{}{0pt}%
\pgfpathmoveto{\pgfqpoint{1.228000in}{0.320000in}}%
\pgfpathlineto{\pgfqpoint{1.228000in}{1.330000in}}%
\pgfusepath{stroke}%
\end{pgfscope}%
\begin{pgfscope}%
\pgfsetbuttcap%
\pgfsetroundjoin%
\definecolor{currentfill}{rgb}{0.321569,0.317647,0.305882}%
\pgfsetfillcolor{currentfill}%
\pgfsetlinewidth{0.602250pt}%
\definecolor{currentstroke}{rgb}{0.321569,0.317647,0.305882}%
\pgfsetstrokecolor{currentstroke}%
\pgfsetdash{}{0pt}%
\pgfsys@defobject{currentmarker}{\pgfqpoint{0.000000in}{-0.034722in}}{\pgfqpoint{0.000000in}{0.000000in}}{%
\pgfpathmoveto{\pgfqpoint{0.000000in}{0.000000in}}%
\pgfpathlineto{\pgfqpoint{0.000000in}{-0.034722in}}%
\pgfusepath{stroke,fill}%
}%
\begin{pgfscope}%
\pgfsys@transformshift{1.228000in}{0.320000in}%
\pgfsys@useobject{currentmarker}{}%
\end{pgfscope}%
\end{pgfscope}%
\begin{pgfscope}%
\definecolor{textcolor}{rgb}{0.101961,0.101961,0.098039}%
\pgfsetstrokecolor{textcolor}%
\pgfsetfillcolor{textcolor}%
\pgftext[x=1.228000in,y=0.236667in,,top]{\color{textcolor}{\rmfamily\fontsize{8.000000}{9.600000}\selectfont\catcode`\^=\active\def^{\ifmmode\sp\else\^{}\fi}\catcode`\%=\active\def
\end{pgfscope}%
\begin{pgfscope}%
\pgfpathrectangle{\pgfqpoint{0.420000in}{0.320000in}}{\pgfqpoint{1.010000in}{1.010000in}}%
\pgfusepath{clip}%
\pgfsetrectcap%
\pgfsetroundjoin%
\pgfsetlinewidth{0.501875pt}%
\definecolor{currentstroke}{rgb}{0.886275,0.882353,0.866667}%
\pgfsetstrokecolor{currentstroke}%
\pgfsetdash{}{0pt}%
\pgfpathmoveto{\pgfqpoint{1.430000in}{0.320000in}}%
\pgfpathlineto{\pgfqpoint{1.430000in}{1.330000in}}%
\pgfusepath{stroke}%
\end{pgfscope}%
\begin{pgfscope}%
\pgfsetbuttcap%
\pgfsetroundjoin%
\definecolor{currentfill}{rgb}{0.321569,0.317647,0.305882}%
\pgfsetfillcolor{currentfill}%
\pgfsetlinewidth{0.602250pt}%
\definecolor{currentstroke}{rgb}{0.321569,0.317647,0.305882}%
\pgfsetstrokecolor{currentstroke}%
\pgfsetdash{}{0pt}%
\pgfsys@defobject{currentmarker}{\pgfqpoint{0.000000in}{-0.034722in}}{\pgfqpoint{0.000000in}{0.000000in}}{%
\pgfpathmoveto{\pgfqpoint{0.000000in}{0.000000in}}%
\pgfpathlineto{\pgfqpoint{0.000000in}{-0.034722in}}%
\pgfusepath{stroke,fill}%
}%
\begin{pgfscope}%
\pgfsys@transformshift{1.430000in}{0.320000in}%
\pgfsys@useobject{currentmarker}{}%
\end{pgfscope}%
\end{pgfscope}%
\begin{pgfscope}%
\definecolor{textcolor}{rgb}{0.101961,0.101961,0.098039}%
\pgfsetstrokecolor{textcolor}%
\pgfsetfillcolor{textcolor}%
\pgftext[x=0.925000in,y=0.081250in,,top]{\color{textcolor}{\rmfamily\fontsize{9.000000}{10.800000}\selectfont\catcode`\^=\active\def^{\ifmmode\sp\else\^{}\fi}\catcode`\%=\active\def
\end{pgfscope}%
\begin{pgfscope}%
\pgfpathrectangle{\pgfqpoint{0.420000in}{0.320000in}}{\pgfqpoint{1.010000in}{1.010000in}}%
\pgfusepath{clip}%
\pgfsetrectcap%
\pgfsetroundjoin%
\pgfsetlinewidth{0.501875pt}%
\definecolor{currentstroke}{rgb}{0.886275,0.882353,0.866667}%
\pgfsetstrokecolor{currentstroke}%
\pgfsetdash{}{0pt}%
\pgfpathmoveto{\pgfqpoint{0.420000in}{0.555973in}}%
\pgfpathlineto{\pgfqpoint{1.430000in}{0.555973in}}%
\pgfusepath{stroke}%
\end{pgfscope}%
\begin{pgfscope}%
\pgfsetbuttcap%
\pgfsetroundjoin%
\definecolor{currentfill}{rgb}{0.321569,0.317647,0.305882}%
\pgfsetfillcolor{currentfill}%
\pgfsetlinewidth{0.602250pt}%
\definecolor{currentstroke}{rgb}{0.321569,0.317647,0.305882}%
\pgfsetstrokecolor{currentstroke}%
\pgfsetdash{}{0pt}%
\pgfsys@defobject{currentmarker}{\pgfqpoint{-0.034722in}{0.000000in}}{\pgfqpoint{-0.000000in}{0.000000in}}{%
\pgfpathmoveto{\pgfqpoint{-0.000000in}{0.000000in}}%
\pgfpathlineto{\pgfqpoint{-0.034722in}{0.000000in}}%
\pgfusepath{stroke,fill}%
}%
\begin{pgfscope}%
\pgfsys@transformshift{0.420000in}{0.555973in}%
\pgfsys@useobject{currentmarker}{}%
\end{pgfscope}%
\end{pgfscope}%
\begin{pgfscope}%
\definecolor{textcolor}{rgb}{0.101961,0.101961,0.098039}%
\pgfsetstrokecolor{textcolor}%
\pgfsetfillcolor{textcolor}%
\pgftext[x=0.160740in, y=0.518070in, left, base]{\color{textcolor}{\rmfamily\fontsize{8.000000}{9.600000}\selectfont\catcode`\^=\active\def^{\ifmmode\sp\else\^{}\fi}\catcode`\%=\active\def
\end{pgfscope}%
\begin{pgfscope}%
\pgfpathrectangle{\pgfqpoint{0.420000in}{0.320000in}}{\pgfqpoint{1.010000in}{1.010000in}}%
\pgfusepath{clip}%
\pgfsetrectcap%
\pgfsetroundjoin%
\pgfsetlinewidth{0.501875pt}%
\definecolor{currentstroke}{rgb}{0.886275,0.882353,0.866667}%
\pgfsetstrokecolor{currentstroke}%
\pgfsetdash{}{0pt}%
\pgfpathmoveto{\pgfqpoint{0.420000in}{0.936100in}}%
\pgfpathlineto{\pgfqpoint{1.430000in}{0.936100in}}%
\pgfusepath{stroke}%
\end{pgfscope}%
\begin{pgfscope}%
\pgfsetbuttcap%
\pgfsetroundjoin%
\definecolor{currentfill}{rgb}{0.321569,0.317647,0.305882}%
\pgfsetfillcolor{currentfill}%
\pgfsetlinewidth{0.602250pt}%
\definecolor{currentstroke}{rgb}{0.321569,0.317647,0.305882}%
\pgfsetstrokecolor{currentstroke}%
\pgfsetdash{}{0pt}%
\pgfsys@defobject{currentmarker}{\pgfqpoint{-0.034722in}{0.000000in}}{\pgfqpoint{-0.000000in}{0.000000in}}{%
\pgfpathmoveto{\pgfqpoint{-0.000000in}{0.000000in}}%
\pgfpathlineto{\pgfqpoint{-0.034722in}{0.000000in}}%
\pgfusepath{stroke,fill}%
}%
\begin{pgfscope}%
\pgfsys@transformshift{0.420000in}{0.936100in}%
\pgfsys@useobject{currentmarker}{}%
\end{pgfscope}%
\end{pgfscope}%
\begin{pgfscope}%
\definecolor{textcolor}{rgb}{0.101961,0.101961,0.098039}%
\pgfsetstrokecolor{textcolor}%
\pgfsetfillcolor{textcolor}%
\pgftext[x=0.160740in, y=0.898197in, left, base]{\color{textcolor}{\rmfamily\fontsize{8.000000}{9.600000}\selectfont\catcode`\^=\active\def^{\ifmmode\sp\else\^{}\fi}\catcode`\%=\active\def
\end{pgfscope}%
\begin{pgfscope}%
\pgfpathrectangle{\pgfqpoint{0.420000in}{0.320000in}}{\pgfqpoint{1.010000in}{1.010000in}}%
\pgfusepath{clip}%
\pgfsetrectcap%
\pgfsetroundjoin%
\pgfsetlinewidth{0.501875pt}%
\definecolor{currentstroke}{rgb}{0.886275,0.882353,0.866667}%
\pgfsetstrokecolor{currentstroke}%
\pgfsetdash{}{0pt}%
\pgfpathmoveto{\pgfqpoint{0.420000in}{1.316228in}}%
\pgfpathlineto{\pgfqpoint{1.430000in}{1.316228in}}%
\pgfusepath{stroke}%
\end{pgfscope}%
\begin{pgfscope}%
\pgfsetbuttcap%
\pgfsetroundjoin%
\definecolor{currentfill}{rgb}{0.321569,0.317647,0.305882}%
\pgfsetfillcolor{currentfill}%
\pgfsetlinewidth{0.602250pt}%
\definecolor{currentstroke}{rgb}{0.321569,0.317647,0.305882}%
\pgfsetstrokecolor{currentstroke}%
\pgfsetdash{}{0pt}%
\pgfsys@defobject{currentmarker}{\pgfqpoint{-0.034722in}{0.000000in}}{\pgfqpoint{-0.000000in}{0.000000in}}{%
\pgfpathmoveto{\pgfqpoint{-0.000000in}{0.000000in}}%
\pgfpathlineto{\pgfqpoint{-0.034722in}{0.000000in}}%
\pgfusepath{stroke,fill}%
}%
\begin{pgfscope}%
\pgfsys@transformshift{0.420000in}{1.316228in}%
\pgfsys@useobject{currentmarker}{}%
\end{pgfscope}%
\end{pgfscope}%
\begin{pgfscope}%
\definecolor{textcolor}{rgb}{0.101961,0.101961,0.098039}%
\pgfsetstrokecolor{textcolor}%
\pgfsetfillcolor{textcolor}%
\pgftext[x=0.160740in, y=1.278325in, left, base]{\color{textcolor}{\rmfamily\fontsize{8.000000}{9.600000}\selectfont\catcode`\^=\active\def^{\ifmmode\sp\else\^{}\fi}\catcode`\%=\active\def
\end{pgfscope}%
\begin{pgfscope}%
\definecolor{textcolor}{rgb}{0.101961,0.101961,0.098039}%
\pgfsetstrokecolor{textcolor}%
\pgfsetfillcolor{textcolor}%
\pgftext[x=0.105184in,y=0.825000in,,bottom,rotate=90.000000]{\color{textcolor}{\rmfamily\fontsize{9.000000}{10.800000}\selectfont\catcode`\^=\active\def^{\ifmmode\sp\else\^{}\fi}\catcode`\%=\active\def
\end{pgfscope}%
\begin{pgfscope}%
\pgfsetrectcap%
\pgfsetmiterjoin%
\pgfsetlinewidth{0.602250pt}%
\definecolor{currentstroke}{rgb}{0.321569,0.317647,0.305882}%
\pgfsetstrokecolor{currentstroke}%
\pgfsetdash{}{0pt}%
\pgfpathmoveto{\pgfqpoint{0.420000in}{0.320000in}}%
\pgfpathlineto{\pgfqpoint{0.420000in}{1.330000in}}%
\pgfusepath{stroke}%
\end{pgfscope}%
\begin{pgfscope}%
\pgfsetrectcap%
\pgfsetmiterjoin%
\pgfsetlinewidth{0.602250pt}%
\definecolor{currentstroke}{rgb}{0.321569,0.317647,0.305882}%
\pgfsetstrokecolor{currentstroke}%
\pgfsetdash{}{0pt}%
\pgfpathmoveto{\pgfqpoint{0.420000in}{0.320000in}}%
\pgfpathlineto{\pgfqpoint{1.430000in}{0.320000in}}%
\pgfusepath{stroke}%
\end{pgfscope}%
\begin{pgfscope}%
\pgfpathrectangle{\pgfqpoint{0.420000in}{0.320000in}}{\pgfqpoint{1.010000in}{1.010000in}}%
\pgfusepath{clip}%
\pgfsetbuttcap%
\pgfsetmiterjoin%
\definecolor{currentfill}{rgb}{0.862745,0.223529,0.466667}%
\pgfsetfillcolor{currentfill}%
\pgfsetlinewidth{0.000000pt}%
\definecolor{currentstroke}{rgb}{0.000000,0.000000,0.000000}%
\pgfsetstrokecolor{currentstroke}%
\pgfsetstrokeopacity{0.000000}%
\pgfsetdash{}{0pt}%
\pgfpathmoveto{\pgfqpoint{0.487113in}{-189.697812in}}%
\pgfpathlineto{\pgfqpoint{0.487113in}{0.365909in}}%
\pgfpathlineto{\pgfqpoint{0.534448in}{0.365909in}}%
\pgfpathlineto{\pgfqpoint{0.534448in}{0.423124in}}%
\pgfpathlineto{\pgfqpoint{0.581782in}{0.423124in}}%
\pgfpathlineto{\pgfqpoint{0.581782in}{0.480339in}}%
\pgfpathlineto{\pgfqpoint{0.629117in}{0.480339in}}%
\pgfpathlineto{\pgfqpoint{0.629117in}{0.613188in}}%
\pgfpathlineto{\pgfqpoint{0.676452in}{0.613188in}}%
\pgfpathlineto{\pgfqpoint{0.676452in}{0.706578in}}%
\pgfpathlineto{\pgfqpoint{0.723787in}{0.706578in}}%
\pgfpathlineto{\pgfqpoint{0.723787in}{0.797261in}}%
\pgfpathlineto{\pgfqpoint{0.771122in}{0.797261in}}%
\pgfpathlineto{\pgfqpoint{0.771122in}{0.883695in}}%
\pgfpathlineto{\pgfqpoint{0.818457in}{0.883695in}}%
\pgfpathlineto{\pgfqpoint{0.818457in}{0.955830in}}%
\pgfpathlineto{\pgfqpoint{0.865791in}{0.955830in}}%
\pgfpathlineto{\pgfqpoint{0.865791in}{0.997969in}}%
\pgfpathlineto{\pgfqpoint{0.913126in}{0.997969in}}%
\pgfpathlineto{\pgfqpoint{0.913126in}{1.042700in}}%
\pgfpathlineto{\pgfqpoint{0.960461in}{1.042700in}}%
\pgfpathlineto{\pgfqpoint{0.960461in}{1.028229in}}%
\pgfpathlineto{\pgfqpoint{1.007796in}{1.028229in}}%
\pgfpathlineto{\pgfqpoint{1.007796in}{1.050303in}}%
\pgfpathlineto{\pgfqpoint{1.055131in}{1.050303in}}%
\pgfpathlineto{\pgfqpoint{1.055131in}{1.108852in}}%
\pgfpathlineto{\pgfqpoint{1.102466in}{1.108852in}}%
\pgfpathlineto{\pgfqpoint{1.102466in}{1.180365in}}%
\pgfpathlineto{\pgfqpoint{1.149800in}{1.180365in}}%
\pgfpathlineto{\pgfqpoint{1.149800in}{1.231019in}}%
\pgfpathlineto{\pgfqpoint{1.197135in}{1.231019in}}%
\pgfpathlineto{\pgfqpoint{1.197135in}{1.268010in}}%
\pgfpathlineto{\pgfqpoint{1.244470in}{1.268010in}}%
\pgfpathlineto{\pgfqpoint{1.244470in}{1.284091in}}%
\pgfpathlineto{\pgfqpoint{1.291805in}{1.284091in}}%
\pgfpathlineto{\pgfqpoint{1.291805in}{1.249148in}}%
\pgfpathlineto{\pgfqpoint{1.339140in}{1.249148in}}%
\pgfpathlineto{\pgfqpoint{1.339140in}{-189.697812in}}%
\pgfusepath{fill}%
\end{pgfscope}%
\begin{pgfscope}%
\definecolor{textcolor}{rgb}{0.101961,0.101961,0.098039}%
\pgfsetstrokecolor{textcolor}%
\pgfsetfillcolor{textcolor}%
\pgftext[x=0.925000in,y=1.413333in,,base]{\color{textcolor}{\rmfamily\fontsize{8.500000}{10.200000}\selectfont\catcode`\^=\active\def^{\ifmmode\sp\else\^{}\fi}\catcode`\%=\active\def
\end{pgfscope}%
\begin{pgfscope}%
\pgfsetbuttcap%
\pgfsetmiterjoin%
\definecolor{currentfill}{rgb}{1.000000,1.000000,1.000000}%
\pgfsetfillcolor{currentfill}%
\pgfsetlinewidth{0.000000pt}%
\definecolor{currentstroke}{rgb}{0.000000,0.000000,0.000000}%
\pgfsetstrokecolor{currentstroke}%
\pgfsetstrokeopacity{0.000000}%
\pgfsetdash{}{0pt}%
\pgfpathmoveto{\pgfqpoint{1.710000in}{0.320000in}}%
\pgfpathlineto{\pgfqpoint{2.720000in}{0.320000in}}%
\pgfpathlineto{\pgfqpoint{2.720000in}{1.330000in}}%
\pgfpathlineto{\pgfqpoint{1.710000in}{1.330000in}}%
\pgfpathlineto{\pgfqpoint{1.710000in}{0.320000in}}%
\pgfpathclose%
\pgfusepath{fill}%
\end{pgfscope}%
\begin{pgfscope}%
\pgfpathrectangle{\pgfqpoint{1.710000in}{0.320000in}}{\pgfqpoint{1.010000in}{1.010000in}}%
\pgfusepath{clip}%
\pgfsetrectcap%
\pgfsetroundjoin%
\pgfsetlinewidth{0.501875pt}%
\definecolor{currentstroke}{rgb}{0.886275,0.882353,0.866667}%
\pgfsetstrokecolor{currentstroke}%
\pgfsetdash{}{0pt}%
\pgfpathmoveto{\pgfqpoint{1.710000in}{0.320000in}}%
\pgfpathlineto{\pgfqpoint{1.710000in}{1.330000in}}%
\pgfusepath{stroke}%
\end{pgfscope}%
\begin{pgfscope}%
\pgfsetbuttcap%
\pgfsetroundjoin%
\definecolor{currentfill}{rgb}{0.321569,0.317647,0.305882}%
\pgfsetfillcolor{currentfill}%
\pgfsetlinewidth{0.602250pt}%
\definecolor{currentstroke}{rgb}{0.321569,0.317647,0.305882}%
\pgfsetstrokecolor{currentstroke}%
\pgfsetdash{}{0pt}%
\pgfsys@defobject{currentmarker}{\pgfqpoint{0.000000in}{-0.034722in}}{\pgfqpoint{0.000000in}{0.000000in}}{%
\pgfpathmoveto{\pgfqpoint{0.000000in}{0.000000in}}%
\pgfpathlineto{\pgfqpoint{0.000000in}{-0.034722in}}%
\pgfusepath{stroke,fill}%
}%
\begin{pgfscope}%
\pgfsys@transformshift{1.710000in}{0.320000in}%
\pgfsys@useobject{currentmarker}{}%
\end{pgfscope}%
\end{pgfscope}%
\begin{pgfscope}%
\definecolor{textcolor}{rgb}{0.101961,0.101961,0.098039}%
\pgfsetstrokecolor{textcolor}%
\pgfsetfillcolor{textcolor}%
\pgftext[x=1.710000in,y=0.236667in,,top]{\color{textcolor}{\rmfamily\fontsize{8.000000}{9.600000}\selectfont\catcode`\^=\active\def^{\ifmmode\sp\else\^{}\fi}\catcode`\%=\active\def
\end{pgfscope}%
\begin{pgfscope}%
\pgfpathrectangle{\pgfqpoint{1.710000in}{0.320000in}}{\pgfqpoint{1.010000in}{1.010000in}}%
\pgfusepath{clip}%
\pgfsetrectcap%
\pgfsetroundjoin%
\pgfsetlinewidth{0.501875pt}%
\definecolor{currentstroke}{rgb}{0.886275,0.882353,0.866667}%
\pgfsetstrokecolor{currentstroke}%
\pgfsetdash{}{0pt}%
\pgfpathmoveto{\pgfqpoint{1.912000in}{0.320000in}}%
\pgfpathlineto{\pgfqpoint{1.912000in}{1.330000in}}%
\pgfusepath{stroke}%
\end{pgfscope}%
\begin{pgfscope}%
\pgfsetbuttcap%
\pgfsetroundjoin%
\definecolor{currentfill}{rgb}{0.321569,0.317647,0.305882}%
\pgfsetfillcolor{currentfill}%
\pgfsetlinewidth{0.602250pt}%
\definecolor{currentstroke}{rgb}{0.321569,0.317647,0.305882}%
\pgfsetstrokecolor{currentstroke}%
\pgfsetdash{}{0pt}%
\pgfsys@defobject{currentmarker}{\pgfqpoint{0.000000in}{-0.034722in}}{\pgfqpoint{0.000000in}{0.000000in}}{%
\pgfpathmoveto{\pgfqpoint{0.000000in}{0.000000in}}%
\pgfpathlineto{\pgfqpoint{0.000000in}{-0.034722in}}%
\pgfusepath{stroke,fill}%
}%
\begin{pgfscope}%
\pgfsys@transformshift{1.912000in}{0.320000in}%
\pgfsys@useobject{currentmarker}{}%
\end{pgfscope}%
\end{pgfscope}%
\begin{pgfscope}%
\pgfpathrectangle{\pgfqpoint{1.710000in}{0.320000in}}{\pgfqpoint{1.010000in}{1.010000in}}%
\pgfusepath{clip}%
\pgfsetrectcap%
\pgfsetroundjoin%
\pgfsetlinewidth{0.501875pt}%
\definecolor{currentstroke}{rgb}{0.886275,0.882353,0.866667}%
\pgfsetstrokecolor{currentstroke}%
\pgfsetdash{}{0pt}%
\pgfpathmoveto{\pgfqpoint{2.114000in}{0.320000in}}%
\pgfpathlineto{\pgfqpoint{2.114000in}{1.330000in}}%
\pgfusepath{stroke}%
\end{pgfscope}%
\begin{pgfscope}%
\pgfsetbuttcap%
\pgfsetroundjoin%
\definecolor{currentfill}{rgb}{0.321569,0.317647,0.305882}%
\pgfsetfillcolor{currentfill}%
\pgfsetlinewidth{0.602250pt}%
\definecolor{currentstroke}{rgb}{0.321569,0.317647,0.305882}%
\pgfsetstrokecolor{currentstroke}%
\pgfsetdash{}{0pt}%
\pgfsys@defobject{currentmarker}{\pgfqpoint{0.000000in}{-0.034722in}}{\pgfqpoint{0.000000in}{0.000000in}}{%
\pgfpathmoveto{\pgfqpoint{0.000000in}{0.000000in}}%
\pgfpathlineto{\pgfqpoint{0.000000in}{-0.034722in}}%
\pgfusepath{stroke,fill}%
}%
\begin{pgfscope}%
\pgfsys@transformshift{2.114000in}{0.320000in}%
\pgfsys@useobject{currentmarker}{}%
\end{pgfscope}%
\end{pgfscope}%
\begin{pgfscope}%
\definecolor{textcolor}{rgb}{0.101961,0.101961,0.098039}%
\pgfsetstrokecolor{textcolor}%
\pgfsetfillcolor{textcolor}%
\pgftext[x=2.114000in,y=0.236667in,,top]{\color{textcolor}{\rmfamily\fontsize{8.000000}{9.600000}\selectfont\catcode`\^=\active\def^{\ifmmode\sp\else\^{}\fi}\catcode`\%=\active\def
\end{pgfscope}%
\begin{pgfscope}%
\pgfpathrectangle{\pgfqpoint{1.710000in}{0.320000in}}{\pgfqpoint{1.010000in}{1.010000in}}%
\pgfusepath{clip}%
\pgfsetrectcap%
\pgfsetroundjoin%
\pgfsetlinewidth{0.501875pt}%
\definecolor{currentstroke}{rgb}{0.886275,0.882353,0.866667}%
\pgfsetstrokecolor{currentstroke}%
\pgfsetdash{}{0pt}%
\pgfpathmoveto{\pgfqpoint{2.316000in}{0.320000in}}%
\pgfpathlineto{\pgfqpoint{2.316000in}{1.330000in}}%
\pgfusepath{stroke}%
\end{pgfscope}%
\begin{pgfscope}%
\pgfsetbuttcap%
\pgfsetroundjoin%
\definecolor{currentfill}{rgb}{0.321569,0.317647,0.305882}%
\pgfsetfillcolor{currentfill}%
\pgfsetlinewidth{0.602250pt}%
\definecolor{currentstroke}{rgb}{0.321569,0.317647,0.305882}%
\pgfsetstrokecolor{currentstroke}%
\pgfsetdash{}{0pt}%
\pgfsys@defobject{currentmarker}{\pgfqpoint{0.000000in}{-0.034722in}}{\pgfqpoint{0.000000in}{0.000000in}}{%
\pgfpathmoveto{\pgfqpoint{0.000000in}{0.000000in}}%
\pgfpathlineto{\pgfqpoint{0.000000in}{-0.034722in}}%
\pgfusepath{stroke,fill}%
}%
\begin{pgfscope}%
\pgfsys@transformshift{2.316000in}{0.320000in}%
\pgfsys@useobject{currentmarker}{}%
\end{pgfscope}%
\end{pgfscope}%
\begin{pgfscope}%
\pgfpathrectangle{\pgfqpoint{1.710000in}{0.320000in}}{\pgfqpoint{1.010000in}{1.010000in}}%
\pgfusepath{clip}%
\pgfsetrectcap%
\pgfsetroundjoin%
\pgfsetlinewidth{0.501875pt}%
\definecolor{currentstroke}{rgb}{0.886275,0.882353,0.866667}%
\pgfsetstrokecolor{currentstroke}%
\pgfsetdash{}{0pt}%
\pgfpathmoveto{\pgfqpoint{2.518000in}{0.320000in}}%
\pgfpathlineto{\pgfqpoint{2.518000in}{1.330000in}}%
\pgfusepath{stroke}%
\end{pgfscope}%
\begin{pgfscope}%
\pgfsetbuttcap%
\pgfsetroundjoin%
\definecolor{currentfill}{rgb}{0.321569,0.317647,0.305882}%
\pgfsetfillcolor{currentfill}%
\pgfsetlinewidth{0.602250pt}%
\definecolor{currentstroke}{rgb}{0.321569,0.317647,0.305882}%
\pgfsetstrokecolor{currentstroke}%
\pgfsetdash{}{0pt}%
\pgfsys@defobject{currentmarker}{\pgfqpoint{0.000000in}{-0.034722in}}{\pgfqpoint{0.000000in}{0.000000in}}{%
\pgfpathmoveto{\pgfqpoint{0.000000in}{0.000000in}}%
\pgfpathlineto{\pgfqpoint{0.000000in}{-0.034722in}}%
\pgfusepath{stroke,fill}%
}%
\begin{pgfscope}%
\pgfsys@transformshift{2.518000in}{0.320000in}%
\pgfsys@useobject{currentmarker}{}%
\end{pgfscope}%
\end{pgfscope}%
\begin{pgfscope}%
\definecolor{textcolor}{rgb}{0.101961,0.101961,0.098039}%
\pgfsetstrokecolor{textcolor}%
\pgfsetfillcolor{textcolor}%
\pgftext[x=2.518000in,y=0.236667in,,top]{\color{textcolor}{\rmfamily\fontsize{8.000000}{9.600000}\selectfont\catcode`\^=\active\def^{\ifmmode\sp\else\^{}\fi}\catcode`\%=\active\def
\end{pgfscope}%
\begin{pgfscope}%
\pgfpathrectangle{\pgfqpoint{1.710000in}{0.320000in}}{\pgfqpoint{1.010000in}{1.010000in}}%
\pgfusepath{clip}%
\pgfsetrectcap%
\pgfsetroundjoin%
\pgfsetlinewidth{0.501875pt}%
\definecolor{currentstroke}{rgb}{0.886275,0.882353,0.866667}%
\pgfsetstrokecolor{currentstroke}%
\pgfsetdash{}{0pt}%
\pgfpathmoveto{\pgfqpoint{2.720000in}{0.320000in}}%
\pgfpathlineto{\pgfqpoint{2.720000in}{1.330000in}}%
\pgfusepath{stroke}%
\end{pgfscope}%
\begin{pgfscope}%
\pgfsetbuttcap%
\pgfsetroundjoin%
\definecolor{currentfill}{rgb}{0.321569,0.317647,0.305882}%
\pgfsetfillcolor{currentfill}%
\pgfsetlinewidth{0.602250pt}%
\definecolor{currentstroke}{rgb}{0.321569,0.317647,0.305882}%
\pgfsetstrokecolor{currentstroke}%
\pgfsetdash{}{0pt}%
\pgfsys@defobject{currentmarker}{\pgfqpoint{0.000000in}{-0.034722in}}{\pgfqpoint{0.000000in}{0.000000in}}{%
\pgfpathmoveto{\pgfqpoint{0.000000in}{0.000000in}}%
\pgfpathlineto{\pgfqpoint{0.000000in}{-0.034722in}}%
\pgfusepath{stroke,fill}%
}%
\begin{pgfscope}%
\pgfsys@transformshift{2.720000in}{0.320000in}%
\pgfsys@useobject{currentmarker}{}%
\end{pgfscope}%
\end{pgfscope}%
\begin{pgfscope}%
\definecolor{textcolor}{rgb}{0.101961,0.101961,0.098039}%
\pgfsetstrokecolor{textcolor}%
\pgfsetfillcolor{textcolor}%
\pgftext[x=2.215000in,y=0.081250in,,top]{\color{textcolor}{\rmfamily\fontsize{9.000000}{10.800000}\selectfont\catcode`\^=\active\def^{\ifmmode\sp\else\^{}\fi}\catcode`\%=\active\def
\end{pgfscope}%
\begin{pgfscope}%
\pgfpathrectangle{\pgfqpoint{1.710000in}{0.320000in}}{\pgfqpoint{1.010000in}{1.010000in}}%
\pgfusepath{clip}%
\pgfsetrectcap%
\pgfsetroundjoin%
\pgfsetlinewidth{0.501875pt}%
\definecolor{currentstroke}{rgb}{0.886275,0.882353,0.866667}%
\pgfsetstrokecolor{currentstroke}%
\pgfsetdash{}{0pt}%
\pgfpathmoveto{\pgfqpoint{1.710000in}{0.555973in}}%
\pgfpathlineto{\pgfqpoint{2.720000in}{0.555973in}}%
\pgfusepath{stroke}%
\end{pgfscope}%
\begin{pgfscope}%
\pgfsetbuttcap%
\pgfsetroundjoin%
\definecolor{currentfill}{rgb}{0.321569,0.317647,0.305882}%
\pgfsetfillcolor{currentfill}%
\pgfsetlinewidth{0.602250pt}%
\definecolor{currentstroke}{rgb}{0.321569,0.317647,0.305882}%
\pgfsetstrokecolor{currentstroke}%
\pgfsetdash{}{0pt}%
\pgfsys@defobject{currentmarker}{\pgfqpoint{-0.034722in}{0.000000in}}{\pgfqpoint{-0.000000in}{0.000000in}}{%
\pgfpathmoveto{\pgfqpoint{-0.000000in}{0.000000in}}%
\pgfpathlineto{\pgfqpoint{-0.034722in}{0.000000in}}%
\pgfusepath{stroke,fill}%
}%
\begin{pgfscope}%
\pgfsys@transformshift{1.710000in}{0.555973in}%
\pgfsys@useobject{currentmarker}{}%
\end{pgfscope}%
\end{pgfscope}%
\begin{pgfscope}%
\pgfpathrectangle{\pgfqpoint{1.710000in}{0.320000in}}{\pgfqpoint{1.010000in}{1.010000in}}%
\pgfusepath{clip}%
\pgfsetrectcap%
\pgfsetroundjoin%
\pgfsetlinewidth{0.501875pt}%
\definecolor{currentstroke}{rgb}{0.886275,0.882353,0.866667}%
\pgfsetstrokecolor{currentstroke}%
\pgfsetdash{}{0pt}%
\pgfpathmoveto{\pgfqpoint{1.710000in}{0.936100in}}%
\pgfpathlineto{\pgfqpoint{2.720000in}{0.936100in}}%
\pgfusepath{stroke}%
\end{pgfscope}%
\begin{pgfscope}%
\pgfsetbuttcap%
\pgfsetroundjoin%
\definecolor{currentfill}{rgb}{0.321569,0.317647,0.305882}%
\pgfsetfillcolor{currentfill}%
\pgfsetlinewidth{0.602250pt}%
\definecolor{currentstroke}{rgb}{0.321569,0.317647,0.305882}%
\pgfsetstrokecolor{currentstroke}%
\pgfsetdash{}{0pt}%
\pgfsys@defobject{currentmarker}{\pgfqpoint{-0.034722in}{0.000000in}}{\pgfqpoint{-0.000000in}{0.000000in}}{%
\pgfpathmoveto{\pgfqpoint{-0.000000in}{0.000000in}}%
\pgfpathlineto{\pgfqpoint{-0.034722in}{0.000000in}}%
\pgfusepath{stroke,fill}%
}%
\begin{pgfscope}%
\pgfsys@transformshift{1.710000in}{0.936100in}%
\pgfsys@useobject{currentmarker}{}%
\end{pgfscope}%
\end{pgfscope}%
\begin{pgfscope}%
\pgfpathrectangle{\pgfqpoint{1.710000in}{0.320000in}}{\pgfqpoint{1.010000in}{1.010000in}}%
\pgfusepath{clip}%
\pgfsetrectcap%
\pgfsetroundjoin%
\pgfsetlinewidth{0.501875pt}%
\definecolor{currentstroke}{rgb}{0.886275,0.882353,0.866667}%
\pgfsetstrokecolor{currentstroke}%
\pgfsetdash{}{0pt}%
\pgfpathmoveto{\pgfqpoint{1.710000in}{1.316228in}}%
\pgfpathlineto{\pgfqpoint{2.720000in}{1.316228in}}%
\pgfusepath{stroke}%
\end{pgfscope}%
\begin{pgfscope}%
\pgfsetbuttcap%
\pgfsetroundjoin%
\definecolor{currentfill}{rgb}{0.321569,0.317647,0.305882}%
\pgfsetfillcolor{currentfill}%
\pgfsetlinewidth{0.602250pt}%
\definecolor{currentstroke}{rgb}{0.321569,0.317647,0.305882}%
\pgfsetstrokecolor{currentstroke}%
\pgfsetdash{}{0pt}%
\pgfsys@defobject{currentmarker}{\pgfqpoint{-0.034722in}{0.000000in}}{\pgfqpoint{-0.000000in}{0.000000in}}{%
\pgfpathmoveto{\pgfqpoint{-0.000000in}{0.000000in}}%
\pgfpathlineto{\pgfqpoint{-0.034722in}{0.000000in}}%
\pgfusepath{stroke,fill}%
}%
\begin{pgfscope}%
\pgfsys@transformshift{1.710000in}{1.316228in}%
\pgfsys@useobject{currentmarker}{}%
\end{pgfscope}%
\end{pgfscope}%
\begin{pgfscope}%
\pgfsetrectcap%
\pgfsetmiterjoin%
\pgfsetlinewidth{0.602250pt}%
\definecolor{currentstroke}{rgb}{0.321569,0.317647,0.305882}%
\pgfsetstrokecolor{currentstroke}%
\pgfsetdash{}{0pt}%
\pgfpathmoveto{\pgfqpoint{1.710000in}{0.320000in}}%
\pgfpathlineto{\pgfqpoint{1.710000in}{1.330000in}}%
\pgfusepath{stroke}%
\end{pgfscope}%
\begin{pgfscope}%
\pgfsetrectcap%
\pgfsetmiterjoin%
\pgfsetlinewidth{0.602250pt}%
\definecolor{currentstroke}{rgb}{0.321569,0.317647,0.305882}%
\pgfsetstrokecolor{currentstroke}%
\pgfsetdash{}{0pt}%
\pgfpathmoveto{\pgfqpoint{1.710000in}{0.320000in}}%
\pgfpathlineto{\pgfqpoint{2.720000in}{0.320000in}}%
\pgfusepath{stroke}%
\end{pgfscope}%
\begin{pgfscope}%
\pgfpathrectangle{\pgfqpoint{1.710000in}{0.320000in}}{\pgfqpoint{1.010000in}{1.010000in}}%
\pgfusepath{clip}%
\pgfsetbuttcap%
\pgfsetmiterjoin%
\definecolor{currentfill}{rgb}{0.031373,0.564706,0.600000}%
\pgfsetfillcolor{currentfill}%
\pgfsetlinewidth{0.000000pt}%
\definecolor{currentstroke}{rgb}{0.000000,0.000000,0.000000}%
\pgfsetstrokecolor{currentstroke}%
\pgfsetstrokeopacity{0.000000}%
\pgfsetdash{}{0pt}%
\pgfpathmoveto{\pgfqpoint{1.777113in}{-189.697812in}}%
\pgfpathlineto{\pgfqpoint{1.777113in}{-189.697812in}}%
\pgfpathlineto{\pgfqpoint{1.824448in}{-189.697812in}}%
\pgfpathlineto{\pgfqpoint{1.824448in}{-189.697812in}}%
\pgfpathlineto{\pgfqpoint{1.871782in}{-189.697812in}}%
\pgfpathlineto{\pgfqpoint{1.871782in}{-189.697812in}}%
\pgfpathlineto{\pgfqpoint{1.919117in}{-189.697812in}}%
\pgfpathlineto{\pgfqpoint{1.919117in}{-189.697812in}}%
\pgfpathlineto{\pgfqpoint{1.966452in}{-189.697812in}}%
\pgfpathlineto{\pgfqpoint{1.966452in}{-189.697812in}}%
\pgfpathlineto{\pgfqpoint{2.013787in}{-189.697812in}}%
\pgfpathlineto{\pgfqpoint{2.013787in}{-189.697812in}}%
\pgfpathlineto{\pgfqpoint{2.061122in}{-189.697812in}}%
\pgfpathlineto{\pgfqpoint{2.061122in}{-189.697812in}}%
\pgfpathlineto{\pgfqpoint{2.108457in}{-189.697812in}}%
\pgfpathlineto{\pgfqpoint{2.108457in}{-189.697812in}}%
\pgfpathlineto{\pgfqpoint{2.155791in}{-189.697812in}}%
\pgfpathlineto{\pgfqpoint{2.155791in}{-189.697812in}}%
\pgfpathlineto{\pgfqpoint{2.203126in}{-189.697812in}}%
\pgfpathlineto{\pgfqpoint{2.203126in}{-189.697812in}}%
\pgfpathlineto{\pgfqpoint{2.250461in}{-189.697812in}}%
\pgfpathlineto{\pgfqpoint{2.250461in}{0.987502in}}%
\pgfpathlineto{\pgfqpoint{2.297796in}{0.987502in}}%
\pgfpathlineto{\pgfqpoint{2.297796in}{1.047418in}}%
\pgfpathlineto{\pgfqpoint{2.345131in}{1.047418in}}%
\pgfpathlineto{\pgfqpoint{2.345131in}{1.054557in}}%
\pgfpathlineto{\pgfqpoint{2.392466in}{1.054557in}}%
\pgfpathlineto{\pgfqpoint{2.392466in}{1.004852in}}%
\pgfpathlineto{\pgfqpoint{2.439800in}{1.004852in}}%
\pgfpathlineto{\pgfqpoint{2.439800in}{0.938460in}}%
\pgfpathlineto{\pgfqpoint{2.487135in}{0.938460in}}%
\pgfpathlineto{\pgfqpoint{2.487135in}{1.015762in}}%
\pgfpathlineto{\pgfqpoint{2.534470in}{1.015762in}}%
\pgfpathlineto{\pgfqpoint{2.534470in}{1.183317in}}%
\pgfpathlineto{\pgfqpoint{2.581805in}{1.183317in}}%
\pgfpathlineto{\pgfqpoint{2.581805in}{1.203707in}}%
\pgfpathlineto{\pgfqpoint{2.629140in}{1.203707in}}%
\pgfpathlineto{\pgfqpoint{2.629140in}{-189.697812in}}%
\pgfusepath{fill}%
\end{pgfscope}%
\begin{pgfscope}%
\definecolor{textcolor}{rgb}{0.101961,0.101961,0.098039}%
\pgfsetstrokecolor{textcolor}%
\pgfsetfillcolor{textcolor}%
\pgftext[x=2.215000in,y=1.413333in,,base]{\color{textcolor}{\rmfamily\fontsize{8.500000}{10.200000}\selectfont\catcode`\^=\active\def^{\ifmmode\sp\else\^{}\fi}\catcode`\%=\active\def
\end{pgfscope}%
\end{pgfpicture}%
\makeatother%
\endgroup%

%% file: media/speedup.pgf
\begingroup%
\makeatletter%
\begin{pgfpicture}%
\pgfpathrectangle{\pgfpointorigin}{\pgfqpoint{3.520000in}{1.590000in}}%
\pgfusepath{use as bounding box, clip}%
\begin{pgfscope}%
\pgfsetbuttcap%
\pgfsetmiterjoin%
\definecolor{currentfill}{rgb}{1.000000,1.000000,1.000000}%
\pgfsetfillcolor{currentfill}%
\pgfsetlinewidth{0.000000pt}%
\definecolor{currentstroke}{rgb}{1.000000,1.000000,1.000000}%
\pgfsetstrokecolor{currentstroke}%
\pgfsetdash{}{0pt}%
\pgfpathmoveto{\pgfqpoint{0.000000in}{0.000000in}}%
\pgfpathlineto{\pgfqpoint{3.520000in}{0.000000in}}%
\pgfpathlineto{\pgfqpoint{3.520000in}{1.590000in}}%
\pgfpathlineto{\pgfqpoint{0.000000in}{1.590000in}}%
\pgfpathlineto{\pgfqpoint{0.000000in}{0.000000in}}%
\pgfpathclose%
\pgfusepath{fill}%
\end{pgfscope}%
\begin{pgfscope}%
\pgfsetbuttcap%
\pgfsetmiterjoin%
\definecolor{currentfill}{rgb}{1.000000,1.000000,1.000000}%
\pgfsetfillcolor{currentfill}%
\pgfsetlinewidth{0.000000pt}%
\definecolor{currentstroke}{rgb}{0.000000,0.000000,0.000000}%
\pgfsetstrokecolor{currentstroke}%
\pgfsetstrokeopacity{0.000000}%
\pgfsetdash{}{0pt}%
\pgfpathmoveto{\pgfqpoint{0.460000in}{0.500000in}}%
\pgfpathlineto{\pgfqpoint{1.470000in}{0.500000in}}%
\pgfpathlineto{\pgfqpoint{1.470000in}{1.510000in}}%
\pgfpathlineto{\pgfqpoint{0.460000in}{1.510000in}}%
\pgfpathlineto{\pgfqpoint{0.460000in}{0.500000in}}%
\pgfpathclose%
\pgfusepath{fill}%
\end{pgfscope}%
\begin{pgfscope}%
\pgfpathrectangle{\pgfqpoint{0.460000in}{0.500000in}}{\pgfqpoint{1.010000in}{1.010000in}}%
\pgfusepath{clip}%
\pgfsetrectcap%
\pgfsetroundjoin%
\pgfsetlinewidth{0.501875pt}%
\definecolor{currentstroke}{rgb}{0.886275,0.882353,0.866667}%
\pgfsetstrokecolor{currentstroke}%
\pgfsetdash{}{0pt}%
\pgfpathmoveto{\pgfqpoint{0.505909in}{0.500000in}}%
\pgfpathlineto{\pgfqpoint{0.505909in}{1.510000in}}%
\pgfusepath{stroke}%
\end{pgfscope}%
\begin{pgfscope}%
\pgfsetbuttcap%
\pgfsetroundjoin%
\definecolor{currentfill}{rgb}{0.321569,0.317647,0.305882}%
\pgfsetfillcolor{currentfill}%
\pgfsetlinewidth{0.602250pt}%
\definecolor{currentstroke}{rgb}{0.321569,0.317647,0.305882}%
\pgfsetstrokecolor{currentstroke}%
\pgfsetdash{}{0pt}%
\pgfsys@defobject{currentmarker}{\pgfqpoint{0.000000in}{-0.034722in}}{\pgfqpoint{0.000000in}{0.000000in}}{%
\pgfpathmoveto{\pgfqpoint{0.000000in}{0.000000in}}%
\pgfpathlineto{\pgfqpoint{0.000000in}{-0.034722in}}%
\pgfusepath{stroke,fill}%
}%
\begin{pgfscope}%
\pgfsys@transformshift{0.505909in}{0.500000in}%
\pgfsys@useobject{currentmarker}{}%
\end{pgfscope}%
\end{pgfscope}%
\begin{pgfscope}%
\definecolor{textcolor}{rgb}{0.101961,0.101961,0.098039}%
\pgfsetstrokecolor{textcolor}%
\pgfsetfillcolor{textcolor}%
\pgftext[x=0.505909in,y=0.416667in,,top]{\color{textcolor}{\rmfamily\fontsize{8.000000}{9.600000}\selectfont\catcode`\^=\active\def^{\ifmmode\sp\else\^{}\fi}\catcode`\%=\active\def
\end{pgfscope}%
\begin{pgfscope}%
\pgfpathrectangle{\pgfqpoint{0.460000in}{0.500000in}}{\pgfqpoint{1.010000in}{1.010000in}}%
\pgfusepath{clip}%
\pgfsetrectcap%
\pgfsetroundjoin%
\pgfsetlinewidth{0.501875pt}%
\definecolor{currentstroke}{rgb}{0.886275,0.882353,0.866667}%
\pgfsetstrokecolor{currentstroke}%
\pgfsetdash{}{0pt}%
\pgfpathmoveto{\pgfqpoint{0.658939in}{0.500000in}}%
\pgfpathlineto{\pgfqpoint{0.658939in}{1.510000in}}%
\pgfusepath{stroke}%
\end{pgfscope}%
\begin{pgfscope}%
\pgfsetbuttcap%
\pgfsetroundjoin%
\definecolor{currentfill}{rgb}{0.321569,0.317647,0.305882}%
\pgfsetfillcolor{currentfill}%
\pgfsetlinewidth{0.602250pt}%
\definecolor{currentstroke}{rgb}{0.321569,0.317647,0.305882}%
\pgfsetstrokecolor{currentstroke}%
\pgfsetdash{}{0pt}%
\pgfsys@defobject{currentmarker}{\pgfqpoint{0.000000in}{-0.034722in}}{\pgfqpoint{0.000000in}{0.000000in}}{%
\pgfpathmoveto{\pgfqpoint{0.000000in}{0.000000in}}%
\pgfpathlineto{\pgfqpoint{0.000000in}{-0.034722in}}%
\pgfusepath{stroke,fill}%
}%
\begin{pgfscope}%
\pgfsys@transformshift{0.658939in}{0.500000in}%
\pgfsys@useobject{currentmarker}{}%
\end{pgfscope}%
\end{pgfscope}%
\begin{pgfscope}%
\pgfpathrectangle{\pgfqpoint{0.460000in}{0.500000in}}{\pgfqpoint{1.010000in}{1.010000in}}%
\pgfusepath{clip}%
\pgfsetrectcap%
\pgfsetroundjoin%
\pgfsetlinewidth{0.501875pt}%
\definecolor{currentstroke}{rgb}{0.886275,0.882353,0.866667}%
\pgfsetstrokecolor{currentstroke}%
\pgfsetdash{}{0pt}%
\pgfpathmoveto{\pgfqpoint{0.811970in}{0.500000in}}%
\pgfpathlineto{\pgfqpoint{0.811970in}{1.510000in}}%
\pgfusepath{stroke}%
\end{pgfscope}%
\begin{pgfscope}%
\pgfsetbuttcap%
\pgfsetroundjoin%
\definecolor{currentfill}{rgb}{0.321569,0.317647,0.305882}%
\pgfsetfillcolor{currentfill}%
\pgfsetlinewidth{0.602250pt}%
\definecolor{currentstroke}{rgb}{0.321569,0.317647,0.305882}%
\pgfsetstrokecolor{currentstroke}%
\pgfsetdash{}{0pt}%
\pgfsys@defobject{currentmarker}{\pgfqpoint{0.000000in}{-0.034722in}}{\pgfqpoint{0.000000in}{0.000000in}}{%
\pgfpathmoveto{\pgfqpoint{0.000000in}{0.000000in}}%
\pgfpathlineto{\pgfqpoint{0.000000in}{-0.034722in}}%
\pgfusepath{stroke,fill}%
}%
\begin{pgfscope}%
\pgfsys@transformshift{0.811970in}{0.500000in}%
\pgfsys@useobject{currentmarker}{}%
\end{pgfscope}%
\end{pgfscope}%
\begin{pgfscope}%
\definecolor{textcolor}{rgb}{0.101961,0.101961,0.098039}%
\pgfsetstrokecolor{textcolor}%
\pgfsetfillcolor{textcolor}%
\pgftext[x=0.811970in,y=0.416667in,,top]{\color{textcolor}{\rmfamily\fontsize{8.000000}{9.600000}\selectfont\catcode`\^=\active\def^{\ifmmode\sp\else\^{}\fi}\catcode`\%=\active\def
\end{pgfscope}%
\begin{pgfscope}%
\pgfpathrectangle{\pgfqpoint{0.460000in}{0.500000in}}{\pgfqpoint{1.010000in}{1.010000in}}%
\pgfusepath{clip}%
\pgfsetrectcap%
\pgfsetroundjoin%
\pgfsetlinewidth{0.501875pt}%
\definecolor{currentstroke}{rgb}{0.886275,0.882353,0.866667}%
\pgfsetstrokecolor{currentstroke}%
\pgfsetdash{}{0pt}%
\pgfpathmoveto{\pgfqpoint{0.965000in}{0.500000in}}%
\pgfpathlineto{\pgfqpoint{0.965000in}{1.510000in}}%
\pgfusepath{stroke}%
\end{pgfscope}%
\begin{pgfscope}%
\pgfsetbuttcap%
\pgfsetroundjoin%
\definecolor{currentfill}{rgb}{0.321569,0.317647,0.305882}%
\pgfsetfillcolor{currentfill}%
\pgfsetlinewidth{0.602250pt}%
\definecolor{currentstroke}{rgb}{0.321569,0.317647,0.305882}%
\pgfsetstrokecolor{currentstroke}%
\pgfsetdash{}{0pt}%
\pgfsys@defobject{currentmarker}{\pgfqpoint{0.000000in}{-0.034722in}}{\pgfqpoint{0.000000in}{0.000000in}}{%
\pgfpathmoveto{\pgfqpoint{0.000000in}{0.000000in}}%
\pgfpathlineto{\pgfqpoint{0.000000in}{-0.034722in}}%
\pgfusepath{stroke,fill}%
}%
\begin{pgfscope}%
\pgfsys@transformshift{0.965000in}{0.500000in}%
\pgfsys@useobject{currentmarker}{}%
\end{pgfscope}%
\end{pgfscope}%
\begin{pgfscope}%
\pgfpathrectangle{\pgfqpoint{0.460000in}{0.500000in}}{\pgfqpoint{1.010000in}{1.010000in}}%
\pgfusepath{clip}%
\pgfsetrectcap%
\pgfsetroundjoin%
\pgfsetlinewidth{0.501875pt}%
\definecolor{currentstroke}{rgb}{0.886275,0.882353,0.866667}%
\pgfsetstrokecolor{currentstroke}%
\pgfsetdash{}{0pt}%
\pgfpathmoveto{\pgfqpoint{1.118030in}{0.500000in}}%
\pgfpathlineto{\pgfqpoint{1.118030in}{1.510000in}}%
\pgfusepath{stroke}%
\end{pgfscope}%
\begin{pgfscope}%
\pgfsetbuttcap%
\pgfsetroundjoin%
\definecolor{currentfill}{rgb}{0.321569,0.317647,0.305882}%
\pgfsetfillcolor{currentfill}%
\pgfsetlinewidth{0.602250pt}%
\definecolor{currentstroke}{rgb}{0.321569,0.317647,0.305882}%
\pgfsetstrokecolor{currentstroke}%
\pgfsetdash{}{0pt}%
\pgfsys@defobject{currentmarker}{\pgfqpoint{0.000000in}{-0.034722in}}{\pgfqpoint{0.000000in}{0.000000in}}{%
\pgfpathmoveto{\pgfqpoint{0.000000in}{0.000000in}}%
\pgfpathlineto{\pgfqpoint{0.000000in}{-0.034722in}}%
\pgfusepath{stroke,fill}%
}%
\begin{pgfscope}%
\pgfsys@transformshift{1.118030in}{0.500000in}%
\pgfsys@useobject{currentmarker}{}%
\end{pgfscope}%
\end{pgfscope}%
\begin{pgfscope}%
\definecolor{textcolor}{rgb}{0.101961,0.101961,0.098039}%
\pgfsetstrokecolor{textcolor}%
\pgfsetfillcolor{textcolor}%
\pgftext[x=1.118030in,y=0.416667in,,top]{\color{textcolor}{\rmfamily\fontsize{8.000000}{9.600000}\selectfont\catcode`\^=\active\def^{\ifmmode\sp\else\^{}\fi}\catcode`\%=\active\def
\end{pgfscope}%
\begin{pgfscope}%
\pgfpathrectangle{\pgfqpoint{0.460000in}{0.500000in}}{\pgfqpoint{1.010000in}{1.010000in}}%
\pgfusepath{clip}%
\pgfsetrectcap%
\pgfsetroundjoin%
\pgfsetlinewidth{0.501875pt}%
\definecolor{currentstroke}{rgb}{0.886275,0.882353,0.866667}%
\pgfsetstrokecolor{currentstroke}%
\pgfsetdash{}{0pt}%
\pgfpathmoveto{\pgfqpoint{1.271061in}{0.500000in}}%
\pgfpathlineto{\pgfqpoint{1.271061in}{1.510000in}}%
\pgfusepath{stroke}%
\end{pgfscope}%
\begin{pgfscope}%
\pgfsetbuttcap%
\pgfsetroundjoin%
\definecolor{currentfill}{rgb}{0.321569,0.317647,0.305882}%
\pgfsetfillcolor{currentfill}%
\pgfsetlinewidth{0.602250pt}%
\definecolor{currentstroke}{rgb}{0.321569,0.317647,0.305882}%
\pgfsetstrokecolor{currentstroke}%
\pgfsetdash{}{0pt}%
\pgfsys@defobject{currentmarker}{\pgfqpoint{0.000000in}{-0.034722in}}{\pgfqpoint{0.000000in}{0.000000in}}{%
\pgfpathmoveto{\pgfqpoint{0.000000in}{0.000000in}}%
\pgfpathlineto{\pgfqpoint{0.000000in}{-0.034722in}}%
\pgfusepath{stroke,fill}%
}%
\begin{pgfscope}%
\pgfsys@transformshift{1.271061in}{0.500000in}%
\pgfsys@useobject{currentmarker}{}%
\end{pgfscope}%
\end{pgfscope}%
\begin{pgfscope}%
\pgfpathrectangle{\pgfqpoint{0.460000in}{0.500000in}}{\pgfqpoint{1.010000in}{1.010000in}}%
\pgfusepath{clip}%
\pgfsetrectcap%
\pgfsetroundjoin%
\pgfsetlinewidth{0.501875pt}%
\definecolor{currentstroke}{rgb}{0.886275,0.882353,0.866667}%
\pgfsetstrokecolor{currentstroke}%
\pgfsetdash{}{0pt}%
\pgfpathmoveto{\pgfqpoint{1.424091in}{0.500000in}}%
\pgfpathlineto{\pgfqpoint{1.424091in}{1.510000in}}%
\pgfusepath{stroke}%
\end{pgfscope}%
\begin{pgfscope}%
\pgfsetbuttcap%
\pgfsetroundjoin%
\definecolor{currentfill}{rgb}{0.321569,0.317647,0.305882}%
\pgfsetfillcolor{currentfill}%
\pgfsetlinewidth{0.602250pt}%
\definecolor{currentstroke}{rgb}{0.321569,0.317647,0.305882}%
\pgfsetstrokecolor{currentstroke}%
\pgfsetdash{}{0pt}%
\pgfsys@defobject{currentmarker}{\pgfqpoint{0.000000in}{-0.034722in}}{\pgfqpoint{0.000000in}{0.000000in}}{%
\pgfpathmoveto{\pgfqpoint{0.000000in}{0.000000in}}%
\pgfpathlineto{\pgfqpoint{0.000000in}{-0.034722in}}%
\pgfusepath{stroke,fill}%
}%
\begin{pgfscope}%
\pgfsys@transformshift{1.424091in}{0.500000in}%
\pgfsys@useobject{currentmarker}{}%
\end{pgfscope}%
\end{pgfscope}%
\begin{pgfscope}%
\definecolor{textcolor}{rgb}{0.101961,0.101961,0.098039}%
\pgfsetstrokecolor{textcolor}%
\pgfsetfillcolor{textcolor}%
\pgftext[x=1.424091in,y=0.416667in,,top]{\color{textcolor}{\rmfamily\fontsize{8.000000}{9.600000}\selectfont\catcode`\^=\active\def^{\ifmmode\sp\else\^{}\fi}\catcode`\%=\active\def
\end{pgfscope}%
\begin{pgfscope}%
\definecolor{textcolor}{rgb}{0.101961,0.101961,0.098039}%
\pgfsetstrokecolor{textcolor}%
\pgfsetfillcolor{textcolor}%
\pgftext[x=0.965000in,y=0.261280in,,top]{\color{textcolor}{\rmfamily\fontsize{9.000000}{10.800000}\selectfont\catcode`\^=\active\def^{\ifmmode\sp\else\^{}\fi}\catcode`\%=\active\def
\end{pgfscope}%
\begin{pgfscope}%
\pgfpathrectangle{\pgfqpoint{0.460000in}{0.500000in}}{\pgfqpoint{1.010000in}{1.010000in}}%
\pgfusepath{clip}%
\pgfsetrectcap%
\pgfsetroundjoin%
\pgfsetlinewidth{0.501875pt}%
\definecolor{currentstroke}{rgb}{0.886275,0.882353,0.866667}%
\pgfsetstrokecolor{currentstroke}%
\pgfsetdash{}{0pt}%
\pgfpathmoveto{\pgfqpoint{0.460000in}{0.500000in}}%
\pgfpathlineto{\pgfqpoint{1.470000in}{0.500000in}}%
\pgfusepath{stroke}%
\end{pgfscope}%
\begin{pgfscope}%
\pgfsetbuttcap%
\pgfsetroundjoin%
\definecolor{currentfill}{rgb}{0.321569,0.317647,0.305882}%
\pgfsetfillcolor{currentfill}%
\pgfsetlinewidth{0.602250pt}%
\definecolor{currentstroke}{rgb}{0.321569,0.317647,0.305882}%
\pgfsetstrokecolor{currentstroke}%
\pgfsetdash{}{0pt}%
\pgfsys@defobject{currentmarker}{\pgfqpoint{-0.034722in}{0.000000in}}{\pgfqpoint{-0.000000in}{0.000000in}}{%
\pgfpathmoveto{\pgfqpoint{-0.000000in}{0.000000in}}%
\pgfpathlineto{\pgfqpoint{-0.034722in}{0.000000in}}%
\pgfusepath{stroke,fill}%
}%
\begin{pgfscope}%
\pgfsys@transformshift{0.460000in}{0.500000in}%
\pgfsys@useobject{currentmarker}{}%
\end{pgfscope}%
\end{pgfscope}%
\begin{pgfscope}%
\definecolor{textcolor}{rgb}{0.101961,0.101961,0.098039}%
\pgfsetstrokecolor{textcolor}%
\pgfsetfillcolor{textcolor}%
\pgftext[x=0.182222in, y=0.462112in, left, base]{\color{textcolor}{\rmfamily\fontsize{8.000000}{9.600000}\selectfont\catcode`\^=\active\def^{\ifmmode\sp\else\^{}\fi}\catcode`\%=\active\def
\end{pgfscope}%
\begin{pgfscope}%
\pgfpathrectangle{\pgfqpoint{0.460000in}{0.500000in}}{\pgfqpoint{1.010000in}{1.010000in}}%
\pgfusepath{clip}%
\pgfsetrectcap%
\pgfsetroundjoin%
\pgfsetlinewidth{0.501875pt}%
\definecolor{currentstroke}{rgb}{0.886275,0.882353,0.866667}%
\pgfsetstrokecolor{currentstroke}%
\pgfsetdash{}{0pt}%
\pgfpathmoveto{\pgfqpoint{0.460000in}{0.756051in}}%
\pgfpathlineto{\pgfqpoint{1.470000in}{0.756051in}}%
\pgfusepath{stroke}%
\end{pgfscope}%
\begin{pgfscope}%
\pgfsetbuttcap%
\pgfsetroundjoin%
\definecolor{currentfill}{rgb}{0.321569,0.317647,0.305882}%
\pgfsetfillcolor{currentfill}%
\pgfsetlinewidth{0.602250pt}%
\definecolor{currentstroke}{rgb}{0.321569,0.317647,0.305882}%
\pgfsetstrokecolor{currentstroke}%
\pgfsetdash{}{0pt}%
\pgfsys@defobject{currentmarker}{\pgfqpoint{-0.034722in}{0.000000in}}{\pgfqpoint{-0.000000in}{0.000000in}}{%
\pgfpathmoveto{\pgfqpoint{-0.000000in}{0.000000in}}%
\pgfpathlineto{\pgfqpoint{-0.034722in}{0.000000in}}%
\pgfusepath{stroke,fill}%
}%
\begin{pgfscope}%
\pgfsys@transformshift{0.460000in}{0.756051in}%
\pgfsys@useobject{currentmarker}{}%
\end{pgfscope}%
\end{pgfscope}%
\begin{pgfscope}%
\definecolor{textcolor}{rgb}{0.101961,0.101961,0.098039}%
\pgfsetstrokecolor{textcolor}%
\pgfsetfillcolor{textcolor}%
\pgftext[x=0.182222in, y=0.718163in, left, base]{\color{textcolor}{\rmfamily\fontsize{8.000000}{9.600000}\selectfont\catcode`\^=\active\def^{\ifmmode\sp\else\^{}\fi}\catcode`\%=\active\def
\end{pgfscope}%
\begin{pgfscope}%
\pgfpathrectangle{\pgfqpoint{0.460000in}{0.500000in}}{\pgfqpoint{1.010000in}{1.010000in}}%
\pgfusepath{clip}%
\pgfsetrectcap%
\pgfsetroundjoin%
\pgfsetlinewidth{0.501875pt}%
\definecolor{currentstroke}{rgb}{0.886275,0.882353,0.866667}%
\pgfsetstrokecolor{currentstroke}%
\pgfsetdash{}{0pt}%
\pgfpathmoveto{\pgfqpoint{0.460000in}{1.012102in}}%
\pgfpathlineto{\pgfqpoint{1.470000in}{1.012102in}}%
\pgfusepath{stroke}%
\end{pgfscope}%
\begin{pgfscope}%
\pgfsetbuttcap%
\pgfsetroundjoin%
\definecolor{currentfill}{rgb}{0.321569,0.317647,0.305882}%
\pgfsetfillcolor{currentfill}%
\pgfsetlinewidth{0.602250pt}%
\definecolor{currentstroke}{rgb}{0.321569,0.317647,0.305882}%
\pgfsetstrokecolor{currentstroke}%
\pgfsetdash{}{0pt}%
\pgfsys@defobject{currentmarker}{\pgfqpoint{-0.034722in}{0.000000in}}{\pgfqpoint{-0.000000in}{0.000000in}}{%
\pgfpathmoveto{\pgfqpoint{-0.000000in}{0.000000in}}%
\pgfpathlineto{\pgfqpoint{-0.034722in}{0.000000in}}%
\pgfusepath{stroke,fill}%
}%
\begin{pgfscope}%
\pgfsys@transformshift{0.460000in}{1.012102in}%
\pgfsys@useobject{currentmarker}{}%
\end{pgfscope}%
\end{pgfscope}%
\begin{pgfscope}%
\definecolor{textcolor}{rgb}{0.101961,0.101961,0.098039}%
\pgfsetstrokecolor{textcolor}%
\pgfsetfillcolor{textcolor}%
\pgftext[x=0.182222in, y=0.974213in, left, base]{\color{textcolor}{\rmfamily\fontsize{8.000000}{9.600000}\selectfont\catcode`\^=\active\def^{\ifmmode\sp\else\^{}\fi}\catcode`\%=\active\def
\end{pgfscope}%
\begin{pgfscope}%
\pgfpathrectangle{\pgfqpoint{0.460000in}{0.500000in}}{\pgfqpoint{1.010000in}{1.010000in}}%
\pgfusepath{clip}%
\pgfsetrectcap%
\pgfsetroundjoin%
\pgfsetlinewidth{0.501875pt}%
\definecolor{currentstroke}{rgb}{0.886275,0.882353,0.866667}%
\pgfsetstrokecolor{currentstroke}%
\pgfsetdash{}{0pt}%
\pgfpathmoveto{\pgfqpoint{0.460000in}{1.268152in}}%
\pgfpathlineto{\pgfqpoint{1.470000in}{1.268152in}}%
\pgfusepath{stroke}%
\end{pgfscope}%
\begin{pgfscope}%
\pgfsetbuttcap%
\pgfsetroundjoin%
\definecolor{currentfill}{rgb}{0.321569,0.317647,0.305882}%
\pgfsetfillcolor{currentfill}%
\pgfsetlinewidth{0.602250pt}%
\definecolor{currentstroke}{rgb}{0.321569,0.317647,0.305882}%
\pgfsetstrokecolor{currentstroke}%
\pgfsetdash{}{0pt}%
\pgfsys@defobject{currentmarker}{\pgfqpoint{-0.034722in}{0.000000in}}{\pgfqpoint{-0.000000in}{0.000000in}}{%
\pgfpathmoveto{\pgfqpoint{-0.000000in}{0.000000in}}%
\pgfpathlineto{\pgfqpoint{-0.034722in}{0.000000in}}%
\pgfusepath{stroke,fill}%
}%
\begin{pgfscope}%
\pgfsys@transformshift{0.460000in}{1.268152in}%
\pgfsys@useobject{currentmarker}{}%
\end{pgfscope}%
\end{pgfscope}%
\begin{pgfscope}%
\definecolor{textcolor}{rgb}{0.101961,0.101961,0.098039}%
\pgfsetstrokecolor{textcolor}%
\pgfsetfillcolor{textcolor}%
\pgftext[x=0.182222in, y=1.230264in, left, base]{\color{textcolor}{\rmfamily\fontsize{8.000000}{9.600000}\selectfont\catcode`\^=\active\def^{\ifmmode\sp\else\^{}\fi}\catcode`\%=\active\def
\end{pgfscope}%
\begin{pgfscope}%
\definecolor{textcolor}{rgb}{0.101961,0.101961,0.098039}%
\pgfsetstrokecolor{textcolor}%
\pgfsetfillcolor{textcolor}%
\pgftext[x=0.126667in,y=1.005000in,,bottom,rotate=90.000000]{\color{textcolor}{\rmfamily\fontsize{9.000000}{10.800000}\selectfont\catcode`\^=\active\def^{\ifmmode\sp\else\^{}\fi}\catcode`\%=\active\def
\end{pgfscope}%
\begin{pgfscope}%
\pgfsetrectcap%
\pgfsetmiterjoin%
\pgfsetlinewidth{0.602250pt}%
\definecolor{currentstroke}{rgb}{0.321569,0.317647,0.305882}%
\pgfsetstrokecolor{currentstroke}%
\pgfsetdash{}{0pt}%
\pgfpathmoveto{\pgfqpoint{0.460000in}{0.500000in}}%
\pgfpathlineto{\pgfqpoint{0.460000in}{1.510000in}}%
\pgfusepath{stroke}%
\end{pgfscope}%
\begin{pgfscope}%
\pgfsetrectcap%
\pgfsetmiterjoin%
\pgfsetlinewidth{0.602250pt}%
\definecolor{currentstroke}{rgb}{0.321569,0.317647,0.305882}%
\pgfsetstrokecolor{currentstroke}%
\pgfsetdash{}{0pt}%
\pgfpathmoveto{\pgfqpoint{0.460000in}{0.500000in}}%
\pgfpathlineto{\pgfqpoint{1.470000in}{0.500000in}}%
\pgfusepath{stroke}%
\end{pgfscope}%
\begin{pgfscope}%
\pgfpathrectangle{\pgfqpoint{0.460000in}{0.500000in}}{\pgfqpoint{1.010000in}{1.010000in}}%
\pgfusepath{clip}%
\pgfsetrectcap%
\pgfsetroundjoin%
\pgfsetlinewidth{1.104125pt}%
\definecolor{currentstroke}{rgb}{0.031373,0.564706,0.600000}%
\pgfsetstrokecolor{currentstroke}%
\pgfsetdash{}{0pt}%
\pgfpathmoveto{\pgfqpoint{0.505909in}{0.535248in}}%
\pgfpathlineto{\pgfqpoint{0.658939in}{0.563482in}}%
\pgfpathlineto{\pgfqpoint{0.811970in}{0.601748in}}%
\pgfpathlineto{\pgfqpoint{0.965000in}{0.673810in}}%
\pgfpathlineto{\pgfqpoint{1.118030in}{0.759281in}}%
\pgfpathlineto{\pgfqpoint{1.271061in}{0.876831in}}%
\pgfpathlineto{\pgfqpoint{1.424091in}{0.979905in}}%
\pgfusepath{stroke}%
\end{pgfscope}%
\begin{pgfscope}%
\pgfpathrectangle{\pgfqpoint{0.460000in}{0.500000in}}{\pgfqpoint{1.010000in}{1.010000in}}%
\pgfusepath{clip}%
\pgfsetrectcap%
\pgfsetroundjoin%
\pgfsetlinewidth{1.104125pt}%
\definecolor{currentstroke}{rgb}{0.862745,0.223529,0.466667}%
\pgfsetstrokecolor{currentstroke}%
\pgfsetdash{}{0pt}%
\pgfpathmoveto{\pgfqpoint{0.505909in}{0.547034in}}%
\pgfpathlineto{\pgfqpoint{0.658939in}{0.589747in}}%
\pgfpathlineto{\pgfqpoint{0.811970in}{0.662789in}}%
\pgfpathlineto{\pgfqpoint{0.965000in}{0.782169in}}%
\pgfpathlineto{\pgfqpoint{1.118030in}{0.956469in}}%
\pgfpathlineto{\pgfqpoint{1.271061in}{1.183334in}}%
\pgfpathlineto{\pgfqpoint{1.424091in}{1.458790in}}%
\pgfusepath{stroke}%
\end{pgfscope}%
\begin{pgfscope}%
\pgfsetbuttcap%
\pgfsetmiterjoin%
\definecolor{currentfill}{rgb}{1.000000,1.000000,1.000000}%
\pgfsetfillcolor{currentfill}%
\pgfsetlinewidth{0.000000pt}%
\definecolor{currentstroke}{rgb}{0.000000,0.000000,0.000000}%
\pgfsetstrokecolor{currentstroke}%
\pgfsetstrokeopacity{0.000000}%
\pgfsetdash{}{0pt}%
\pgfpathmoveto{\pgfqpoint{1.930000in}{0.500000in}}%
\pgfpathlineto{\pgfqpoint{2.940000in}{0.500000in}}%
\pgfpathlineto{\pgfqpoint{2.940000in}{1.510000in}}%
\pgfpathlineto{\pgfqpoint{1.930000in}{1.510000in}}%
\pgfpathlineto{\pgfqpoint{1.930000in}{0.500000in}}%
\pgfpathclose%
\pgfusepath{fill}%
\end{pgfscope}%
\begin{pgfscope}%
\pgfpathrectangle{\pgfqpoint{1.930000in}{0.500000in}}{\pgfqpoint{1.010000in}{1.010000in}}%
\pgfusepath{clip}%
\pgfsetrectcap%
\pgfsetroundjoin%
\pgfsetlinewidth{0.501875pt}%
\definecolor{currentstroke}{rgb}{0.886275,0.882353,0.866667}%
\pgfsetstrokecolor{currentstroke}%
\pgfsetdash{}{0pt}%
\pgfpathmoveto{\pgfqpoint{1.975909in}{0.500000in}}%
\pgfpathlineto{\pgfqpoint{1.975909in}{1.510000in}}%
\pgfusepath{stroke}%
\end{pgfscope}%
\begin{pgfscope}%
\pgfsetbuttcap%
\pgfsetroundjoin%
\definecolor{currentfill}{rgb}{0.321569,0.317647,0.305882}%
\pgfsetfillcolor{currentfill}%
\pgfsetlinewidth{0.602250pt}%
\definecolor{currentstroke}{rgb}{0.321569,0.317647,0.305882}%
\pgfsetstrokecolor{currentstroke}%
\pgfsetdash{}{0pt}%
\pgfsys@defobject{currentmarker}{\pgfqpoint{0.000000in}{-0.034722in}}{\pgfqpoint{0.000000in}{0.000000in}}{%
\pgfpathmoveto{\pgfqpoint{0.000000in}{0.000000in}}%
\pgfpathlineto{\pgfqpoint{0.000000in}{-0.034722in}}%
\pgfusepath{stroke,fill}%
}%
\begin{pgfscope}%
\pgfsys@transformshift{1.975909in}{0.500000in}%
\pgfsys@useobject{currentmarker}{}%
\end{pgfscope}%
\end{pgfscope}%
\begin{pgfscope}%
\definecolor{textcolor}{rgb}{0.101961,0.101961,0.098039}%
\pgfsetstrokecolor{textcolor}%
\pgfsetfillcolor{textcolor}%
\pgftext[x=1.975909in,y=0.416667in,,top]{\color{textcolor}{\rmfamily\fontsize{8.000000}{9.600000}\selectfont\catcode`\^=\active\def^{\ifmmode\sp\else\^{}\fi}\catcode`\%=\active\def
\end{pgfscope}%
\begin{pgfscope}%
\pgfpathrectangle{\pgfqpoint{1.930000in}{0.500000in}}{\pgfqpoint{1.010000in}{1.010000in}}%
\pgfusepath{clip}%
\pgfsetrectcap%
\pgfsetroundjoin%
\pgfsetlinewidth{0.501875pt}%
\definecolor{currentstroke}{rgb}{0.886275,0.882353,0.866667}%
\pgfsetstrokecolor{currentstroke}%
\pgfsetdash{}{0pt}%
\pgfpathmoveto{\pgfqpoint{2.205455in}{0.500000in}}%
\pgfpathlineto{\pgfqpoint{2.205455in}{1.510000in}}%
\pgfusepath{stroke}%
\end{pgfscope}%
\begin{pgfscope}%
\pgfsetbuttcap%
\pgfsetroundjoin%
\definecolor{currentfill}{rgb}{0.321569,0.317647,0.305882}%
\pgfsetfillcolor{currentfill}%
\pgfsetlinewidth{0.602250pt}%
\definecolor{currentstroke}{rgb}{0.321569,0.317647,0.305882}%
\pgfsetstrokecolor{currentstroke}%
\pgfsetdash{}{0pt}%
\pgfsys@defobject{currentmarker}{\pgfqpoint{0.000000in}{-0.034722in}}{\pgfqpoint{0.000000in}{0.000000in}}{%
\pgfpathmoveto{\pgfqpoint{0.000000in}{0.000000in}}%
\pgfpathlineto{\pgfqpoint{0.000000in}{-0.034722in}}%
\pgfusepath{stroke,fill}%
}%
\begin{pgfscope}%
\pgfsys@transformshift{2.205455in}{0.500000in}%
\pgfsys@useobject{currentmarker}{}%
\end{pgfscope}%
\end{pgfscope}%
\begin{pgfscope}%
\pgfpathrectangle{\pgfqpoint{1.930000in}{0.500000in}}{\pgfqpoint{1.010000in}{1.010000in}}%
\pgfusepath{clip}%
\pgfsetrectcap%
\pgfsetroundjoin%
\pgfsetlinewidth{0.501875pt}%
\definecolor{currentstroke}{rgb}{0.886275,0.882353,0.866667}%
\pgfsetstrokecolor{currentstroke}%
\pgfsetdash{}{0pt}%
\pgfpathmoveto{\pgfqpoint{2.435000in}{0.500000in}}%
\pgfpathlineto{\pgfqpoint{2.435000in}{1.510000in}}%
\pgfusepath{stroke}%
\end{pgfscope}%
\begin{pgfscope}%
\pgfsetbuttcap%
\pgfsetroundjoin%
\definecolor{currentfill}{rgb}{0.321569,0.317647,0.305882}%
\pgfsetfillcolor{currentfill}%
\pgfsetlinewidth{0.602250pt}%
\definecolor{currentstroke}{rgb}{0.321569,0.317647,0.305882}%
\pgfsetstrokecolor{currentstroke}%
\pgfsetdash{}{0pt}%
\pgfsys@defobject{currentmarker}{\pgfqpoint{0.000000in}{-0.034722in}}{\pgfqpoint{0.000000in}{0.000000in}}{%
\pgfpathmoveto{\pgfqpoint{0.000000in}{0.000000in}}%
\pgfpathlineto{\pgfqpoint{0.000000in}{-0.034722in}}%
\pgfusepath{stroke,fill}%
}%
\begin{pgfscope}%
\pgfsys@transformshift{2.435000in}{0.500000in}%
\pgfsys@useobject{currentmarker}{}%
\end{pgfscope}%
\end{pgfscope}%
\begin{pgfscope}%
\definecolor{textcolor}{rgb}{0.101961,0.101961,0.098039}%
\pgfsetstrokecolor{textcolor}%
\pgfsetfillcolor{textcolor}%
\pgftext[x=2.435000in,y=0.416667in,,top]{\color{textcolor}{\rmfamily\fontsize{8.000000}{9.600000}\selectfont\catcode`\^=\active\def^{\ifmmode\sp\else\^{}\fi}\catcode`\%=\active\def
\end{pgfscope}%
\begin{pgfscope}%
\pgfpathrectangle{\pgfqpoint{1.930000in}{0.500000in}}{\pgfqpoint{1.010000in}{1.010000in}}%
\pgfusepath{clip}%
\pgfsetrectcap%
\pgfsetroundjoin%
\pgfsetlinewidth{0.501875pt}%
\definecolor{currentstroke}{rgb}{0.886275,0.882353,0.866667}%
\pgfsetstrokecolor{currentstroke}%
\pgfsetdash{}{0pt}%
\pgfpathmoveto{\pgfqpoint{2.664545in}{0.500000in}}%
\pgfpathlineto{\pgfqpoint{2.664545in}{1.510000in}}%
\pgfusepath{stroke}%
\end{pgfscope}%
\begin{pgfscope}%
\pgfsetbuttcap%
\pgfsetroundjoin%
\definecolor{currentfill}{rgb}{0.321569,0.317647,0.305882}%
\pgfsetfillcolor{currentfill}%
\pgfsetlinewidth{0.602250pt}%
\definecolor{currentstroke}{rgb}{0.321569,0.317647,0.305882}%
\pgfsetstrokecolor{currentstroke}%
\pgfsetdash{}{0pt}%
\pgfsys@defobject{currentmarker}{\pgfqpoint{0.000000in}{-0.034722in}}{\pgfqpoint{0.000000in}{0.000000in}}{%
\pgfpathmoveto{\pgfqpoint{0.000000in}{0.000000in}}%
\pgfpathlineto{\pgfqpoint{0.000000in}{-0.034722in}}%
\pgfusepath{stroke,fill}%
}%
\begin{pgfscope}%
\pgfsys@transformshift{2.664545in}{0.500000in}%
\pgfsys@useobject{currentmarker}{}%
\end{pgfscope}%
\end{pgfscope}%
\begin{pgfscope}%
\pgfpathrectangle{\pgfqpoint{1.930000in}{0.500000in}}{\pgfqpoint{1.010000in}{1.010000in}}%
\pgfusepath{clip}%
\pgfsetrectcap%
\pgfsetroundjoin%
\pgfsetlinewidth{0.501875pt}%
\definecolor{currentstroke}{rgb}{0.886275,0.882353,0.866667}%
\pgfsetstrokecolor{currentstroke}%
\pgfsetdash{}{0pt}%
\pgfpathmoveto{\pgfqpoint{2.894091in}{0.500000in}}%
\pgfpathlineto{\pgfqpoint{2.894091in}{1.510000in}}%
\pgfusepath{stroke}%
\end{pgfscope}%
\begin{pgfscope}%
\pgfsetbuttcap%
\pgfsetroundjoin%
\definecolor{currentfill}{rgb}{0.321569,0.317647,0.305882}%
\pgfsetfillcolor{currentfill}%
\pgfsetlinewidth{0.602250pt}%
\definecolor{currentstroke}{rgb}{0.321569,0.317647,0.305882}%
\pgfsetstrokecolor{currentstroke}%
\pgfsetdash{}{0pt}%
\pgfsys@defobject{currentmarker}{\pgfqpoint{0.000000in}{-0.034722in}}{\pgfqpoint{0.000000in}{0.000000in}}{%
\pgfpathmoveto{\pgfqpoint{0.000000in}{0.000000in}}%
\pgfpathlineto{\pgfqpoint{0.000000in}{-0.034722in}}%
\pgfusepath{stroke,fill}%
}%
\begin{pgfscope}%
\pgfsys@transformshift{2.894091in}{0.500000in}%
\pgfsys@useobject{currentmarker}{}%
\end{pgfscope}%
\end{pgfscope}%
\begin{pgfscope}%
\definecolor{textcolor}{rgb}{0.101961,0.101961,0.098039}%
\pgfsetstrokecolor{textcolor}%
\pgfsetfillcolor{textcolor}%
\pgftext[x=2.894091in,y=0.416667in,,top]{\color{textcolor}{\rmfamily\fontsize{8.000000}{9.600000}\selectfont\catcode`\^=\active\def^{\ifmmode\sp\else\^{}\fi}\catcode`\%=\active\def
\end{pgfscope}%
\begin{pgfscope}%
\pgfpathrectangle{\pgfqpoint{1.930000in}{0.500000in}}{\pgfqpoint{1.010000in}{1.010000in}}%
\pgfusepath{clip}%
\pgfsetrectcap%
\pgfsetroundjoin%
\pgfsetlinewidth{0.501875pt}%
\definecolor{currentstroke}{rgb}{0.886275,0.882353,0.866667}%
\pgfsetstrokecolor{currentstroke}%
\pgfsetdash{}{0pt}%
\pgfpathmoveto{\pgfqpoint{1.930000in}{0.500000in}}%
\pgfpathlineto{\pgfqpoint{2.940000in}{0.500000in}}%
\pgfusepath{stroke}%
\end{pgfscope}%
\begin{pgfscope}%
\pgfsetbuttcap%
\pgfsetroundjoin%
\definecolor{currentfill}{rgb}{0.321569,0.317647,0.305882}%
\pgfsetfillcolor{currentfill}%
\pgfsetlinewidth{0.602250pt}%
\definecolor{currentstroke}{rgb}{0.321569,0.317647,0.305882}%
\pgfsetstrokecolor{currentstroke}%
\pgfsetdash{}{0pt}%
\pgfsys@defobject{currentmarker}{\pgfqpoint{-0.034722in}{0.000000in}}{\pgfqpoint{-0.000000in}{0.000000in}}{%
\pgfpathmoveto{\pgfqpoint{-0.000000in}{0.000000in}}%
\pgfpathlineto{\pgfqpoint{-0.034722in}{0.000000in}}%
\pgfusepath{stroke,fill}%
}%
\begin{pgfscope}%
\pgfsys@transformshift{1.930000in}{0.500000in}%
\pgfsys@useobject{currentmarker}{}%
\end{pgfscope}%
\end{pgfscope}%
\begin{pgfscope}%
\definecolor{textcolor}{rgb}{0.101961,0.101961,0.098039}%
\pgfsetstrokecolor{textcolor}%
\pgfsetfillcolor{textcolor}%
\pgftext[x=1.707778in, y=0.462112in, left, base]{\color{textcolor}{\rmfamily\fontsize{8.000000}{9.600000}\selectfont\catcode`\^=\active\def^{\ifmmode\sp\else\^{}\fi}\catcode`\%=\active\def
\end{pgfscope}%
\begin{pgfscope}%
\pgfpathrectangle{\pgfqpoint{1.930000in}{0.500000in}}{\pgfqpoint{1.010000in}{1.010000in}}%
\pgfusepath{clip}%
\pgfsetrectcap%
\pgfsetroundjoin%
\pgfsetlinewidth{0.501875pt}%
\definecolor{currentstroke}{rgb}{0.886275,0.882353,0.866667}%
\pgfsetstrokecolor{currentstroke}%
\pgfsetdash{}{0pt}%
\pgfpathmoveto{\pgfqpoint{1.930000in}{0.847272in}}%
\pgfpathlineto{\pgfqpoint{2.940000in}{0.847272in}}%
\pgfusepath{stroke}%
\end{pgfscope}%
\begin{pgfscope}%
\pgfsetbuttcap%
\pgfsetroundjoin%
\definecolor{currentfill}{rgb}{0.321569,0.317647,0.305882}%
\pgfsetfillcolor{currentfill}%
\pgfsetlinewidth{0.602250pt}%
\definecolor{currentstroke}{rgb}{0.321569,0.317647,0.305882}%
\pgfsetstrokecolor{currentstroke}%
\pgfsetdash{}{0pt}%
\pgfsys@defobject{currentmarker}{\pgfqpoint{-0.034722in}{0.000000in}}{\pgfqpoint{-0.000000in}{0.000000in}}{%
\pgfpathmoveto{\pgfqpoint{-0.000000in}{0.000000in}}%
\pgfpathlineto{\pgfqpoint{-0.034722in}{0.000000in}}%
\pgfusepath{stroke,fill}%
}%
\begin{pgfscope}%
\pgfsys@transformshift{1.930000in}{0.847272in}%
\pgfsys@useobject{currentmarker}{}%
\end{pgfscope}%
\end{pgfscope}%
\begin{pgfscope}%
\definecolor{textcolor}{rgb}{0.101961,0.101961,0.098039}%
\pgfsetstrokecolor{textcolor}%
\pgfsetfillcolor{textcolor}%
\pgftext[x=1.707778in, y=0.809384in, left, base]{\color{textcolor}{\rmfamily\fontsize{8.000000}{9.600000}\selectfont\catcode`\^=\active\def^{\ifmmode\sp\else\^{}\fi}\catcode`\%=\active\def
\end{pgfscope}%
\begin{pgfscope}%
\pgfpathrectangle{\pgfqpoint{1.930000in}{0.500000in}}{\pgfqpoint{1.010000in}{1.010000in}}%
\pgfusepath{clip}%
\pgfsetrectcap%
\pgfsetroundjoin%
\pgfsetlinewidth{0.501875pt}%
\definecolor{currentstroke}{rgb}{0.886275,0.882353,0.866667}%
\pgfsetstrokecolor{currentstroke}%
\pgfsetdash{}{0pt}%
\pgfpathmoveto{\pgfqpoint{1.930000in}{1.194544in}}%
\pgfpathlineto{\pgfqpoint{2.940000in}{1.194544in}}%
\pgfusepath{stroke}%
\end{pgfscope}%
\begin{pgfscope}%
\pgfsetbuttcap%
\pgfsetroundjoin%
\definecolor{currentfill}{rgb}{0.321569,0.317647,0.305882}%
\pgfsetfillcolor{currentfill}%
\pgfsetlinewidth{0.602250pt}%
\definecolor{currentstroke}{rgb}{0.321569,0.317647,0.305882}%
\pgfsetstrokecolor{currentstroke}%
\pgfsetdash{}{0pt}%
\pgfsys@defobject{currentmarker}{\pgfqpoint{-0.034722in}{0.000000in}}{\pgfqpoint{-0.000000in}{0.000000in}}{%
\pgfpathmoveto{\pgfqpoint{-0.000000in}{0.000000in}}%
\pgfpathlineto{\pgfqpoint{-0.034722in}{0.000000in}}%
\pgfusepath{stroke,fill}%
}%
\begin{pgfscope}%
\pgfsys@transformshift{1.930000in}{1.194544in}%
\pgfsys@useobject{currentmarker}{}%
\end{pgfscope}%
\end{pgfscope}%
\begin{pgfscope}%
\definecolor{textcolor}{rgb}{0.101961,0.101961,0.098039}%
\pgfsetstrokecolor{textcolor}%
\pgfsetfillcolor{textcolor}%
\pgftext[x=1.707778in, y=1.156656in, left, base]{\color{textcolor}{\rmfamily\fontsize{8.000000}{9.600000}\selectfont\catcode`\^=\active\def^{\ifmmode\sp\else\^{}\fi}\catcode`\%=\active\def
\end{pgfscope}%
\begin{pgfscope}%
\pgfsetrectcap%
\pgfsetmiterjoin%
\pgfsetlinewidth{0.602250pt}%
\definecolor{currentstroke}{rgb}{0.321569,0.317647,0.305882}%
\pgfsetstrokecolor{currentstroke}%
\pgfsetdash{}{0pt}%
\pgfpathmoveto{\pgfqpoint{1.930000in}{0.500000in}}%
\pgfpathlineto{\pgfqpoint{1.930000in}{1.510000in}}%
\pgfusepath{stroke}%
\end{pgfscope}%
\begin{pgfscope}%
\pgfsetrectcap%
\pgfsetmiterjoin%
\pgfsetlinewidth{0.602250pt}%
\definecolor{currentstroke}{rgb}{0.321569,0.317647,0.305882}%
\pgfsetstrokecolor{currentstroke}%
\pgfsetdash{}{0pt}%
\pgfpathmoveto{\pgfqpoint{1.930000in}{0.500000in}}%
\pgfpathlineto{\pgfqpoint{2.940000in}{0.500000in}}%
\pgfusepath{stroke}%
\end{pgfscope}%
\begin{pgfscope}%
\pgfpathrectangle{\pgfqpoint{1.930000in}{0.500000in}}{\pgfqpoint{1.010000in}{1.010000in}}%
\pgfusepath{clip}%
\pgfsetbuttcap%
\pgfsetroundjoin%
\pgfsetlinewidth{1.104125pt}%
\definecolor{currentstroke}{rgb}{0.031373,0.564706,0.600000}%
\pgfsetstrokecolor{currentstroke}%
\pgfsetdash{{1.100000pt}{1.540000pt}}{0.000000pt}%
\pgfpathmoveto{\pgfqpoint{1.975909in}{0.514901in}}%
\pgfpathlineto{\pgfqpoint{2.205455in}{0.496390in}}%
\pgfpathlineto{\pgfqpoint{2.435000in}{1.051553in}}%
\pgfpathlineto{\pgfqpoint{2.664545in}{0.977446in}}%
\pgfpathlineto{\pgfqpoint{2.894091in}{0.999618in}}%
\pgfusepath{stroke}%
\end{pgfscope}%
\begin{pgfscope}%
\pgfpathrectangle{\pgfqpoint{1.930000in}{0.500000in}}{\pgfqpoint{1.010000in}{1.010000in}}%
\pgfusepath{clip}%
\pgfsetbuttcap%
\pgfsetroundjoin%
\pgfsetlinewidth{1.104125pt}%
\definecolor{currentstroke}{rgb}{0.862745,0.223529,0.466667}%
\pgfsetstrokecolor{currentstroke}%
\pgfsetdash{{1.100000pt}{1.540000pt}}{0.000000pt}%
\pgfpathmoveto{\pgfqpoint{1.975909in}{0.535662in}}%
\pgfpathlineto{\pgfqpoint{2.205455in}{0.524177in}}%
\pgfpathlineto{\pgfqpoint{2.435000in}{1.336364in}}%
\pgfpathlineto{\pgfqpoint{2.664545in}{1.215888in}}%
\pgfpathlineto{\pgfqpoint{2.894091in}{1.252386in}}%
\pgfusepath{stroke}%
\end{pgfscope}%
\begin{pgfscope}%
\pgfpathrectangle{\pgfqpoint{1.930000in}{0.500000in}}{\pgfqpoint{1.010000in}{1.010000in}}%
\pgfusepath{clip}%
\pgfsetbuttcap%
\pgfsetroundjoin%
\pgfsetlinewidth{1.104125pt}%
\definecolor{currentstroke}{rgb}{0.031373,0.564706,0.600000}%
\pgfsetstrokecolor{currentstroke}%
\pgfsetdash{{3.300000pt}{1.430000pt}}{0.000000pt}%
\pgfpathmoveto{\pgfqpoint{1.975909in}{1.072074in}}%
\pgfpathlineto{\pgfqpoint{2.205455in}{0.989042in}}%
\pgfpathlineto{\pgfqpoint{2.435000in}{0.982662in}}%
\pgfpathlineto{\pgfqpoint{2.664545in}{1.010257in}}%
\pgfpathlineto{\pgfqpoint{2.894091in}{0.996360in}}%
\pgfusepath{stroke}%
\end{pgfscope}%
\begin{pgfscope}%
\pgfpathrectangle{\pgfqpoint{1.930000in}{0.500000in}}{\pgfqpoint{1.010000in}{1.010000in}}%
\pgfusepath{clip}%
\pgfsetbuttcap%
\pgfsetroundjoin%
\pgfsetlinewidth{1.104125pt}%
\definecolor{currentstroke}{rgb}{0.862745,0.223529,0.466667}%
\pgfsetstrokecolor{currentstroke}%
\pgfsetdash{{3.300000pt}{1.430000pt}}{0.000000pt}%
\pgfpathmoveto{\pgfqpoint{1.975909in}{1.245414in}}%
\pgfpathlineto{\pgfqpoint{2.205455in}{1.216287in}}%
\pgfpathlineto{\pgfqpoint{2.435000in}{1.237957in}}%
\pgfpathlineto{\pgfqpoint{2.664545in}{1.247989in}}%
\pgfpathlineto{\pgfqpoint{2.894091in}{1.249704in}}%
\pgfusepath{stroke}%
\end{pgfscope}%
\begin{pgfscope}%
\definecolor{textcolor}{rgb}{0.101961,0.101961,0.098039}%
\pgfsetstrokecolor{textcolor}%
\pgfsetfillcolor{textcolor}%
\pgftext[x=0.965000in,y=0.040000in,,base]{\color{textcolor}{\rmfamily\fontsize{9.000000}{10.800000}\selectfont\catcode`\^=\active\def^{\ifmmode\sp\else\^{}\fi}\catcode`\%=\active\def
\end{pgfscope}%
\begin{pgfscope}%
\definecolor{textcolor}{rgb}{0.101961,0.101961,0.098039}%
\pgfsetstrokecolor{textcolor}%
\pgfsetfillcolor{textcolor}%
\pgftext[x=2.435000in,y=0.040000in,,base]{\color{textcolor}{\rmfamily\fontsize{9.000000}{10.800000}\selectfont\catcode`\^=\active\def^{\ifmmode\sp\else\^{}\fi}\catcode`\%=\active\def
\end{pgfscope}%
\begin{pgfscope}%
\pgfsetrectcap%
\pgfsetroundjoin%
\pgfsetlinewidth{1.104125pt}%
\definecolor{currentstroke}{rgb}{0.031373,0.564706,0.600000}%
\pgfsetstrokecolor{currentstroke}%
\pgfsetdash{}{0pt}%
\pgfpathmoveto{\pgfqpoint{3.064444in}{1.253109in}}%
\pgfpathlineto{\pgfqpoint{3.136667in}{1.253109in}}%
\pgfpathlineto{\pgfqpoint{3.208889in}{1.253109in}}%
\pgfusepath{stroke}%
\end{pgfscope}%
\begin{pgfscope}%
\definecolor{textcolor}{rgb}{0.101961,0.101961,0.098039}%
\pgfsetstrokecolor{textcolor}%
\pgfsetfillcolor{textcolor}%
\pgftext[x=3.242222in,y=1.214220in,left,base]{\color{textcolor}{\rmfamily\fontsize{8.000000}{9.600000}\selectfont\catcode`\^=\active\def^{\ifmmode\sp\else\^{}\fi}\catcode`\%=\active\def
\end{pgfscope}%
\begin{pgfscope}%
\pgfsetrectcap%
\pgfsetroundjoin%
\pgfsetlinewidth{1.104125pt}%
\definecolor{currentstroke}{rgb}{0.862745,0.223529,0.466667}%
\pgfsetstrokecolor{currentstroke}%
\pgfsetdash{}{0pt}%
\pgfpathmoveto{\pgfqpoint{3.064444in}{1.095721in}}%
\pgfpathlineto{\pgfqpoint{3.136667in}{1.095721in}}%
\pgfpathlineto{\pgfqpoint{3.208889in}{1.095721in}}%
\pgfusepath{stroke}%
\end{pgfscope}%
\begin{pgfscope}%
\definecolor{textcolor}{rgb}{0.101961,0.101961,0.098039}%
\pgfsetstrokecolor{textcolor}%
\pgfsetfillcolor{textcolor}%
\pgftext[x=3.242222in,y=1.056832in,left,base]{\color{textcolor}{\rmfamily\fontsize{8.000000}{9.600000}\selectfont\catcode`\^=\active\def^{\ifmmode\sp\else\^{}\fi}\catcode`\%=\active\def
\end{pgfscope}%
\begin{pgfscope}%
\pgfsetbuttcap%
\pgfsetroundjoin%
\pgfsetlinewidth{1.104125pt}%
\definecolor{currentstroke}{rgb}{0.321569,0.317647,0.305882}%
\pgfsetstrokecolor{currentstroke}%
\pgfsetdash{{1.100000pt}{1.540000pt}}{0.000000pt}%
\pgfpathmoveto{\pgfqpoint{3.064444in}{0.938333in}}%
\pgfpathlineto{\pgfqpoint{3.136667in}{0.938333in}}%
\pgfpathlineto{\pgfqpoint{3.208889in}{0.938333in}}%
\pgfusepath{stroke}%
\end{pgfscope}%
\begin{pgfscope}%
\definecolor{textcolor}{rgb}{0.101961,0.101961,0.098039}%
\pgfsetstrokecolor{textcolor}%
\pgfsetfillcolor{textcolor}%
\pgftext[x=3.242222in,y=0.899444in,left,base]{\color{textcolor}{\rmfamily\fontsize{8.000000}{9.600000}\selectfont\catcode`\^=\active\def^{\ifmmode\sp\else\^{}\fi}\catcode`\%=\active\def
\end{pgfscope}%
\begin{pgfscope}%
\pgfsetbuttcap%
\pgfsetroundjoin%
\pgfsetlinewidth{1.104125pt}%
\definecolor{currentstroke}{rgb}{0.321569,0.317647,0.305882}%
\pgfsetstrokecolor{currentstroke}%
\pgfsetdash{{3.300000pt}{1.430000pt}}{0.000000pt}%
\pgfpathmoveto{\pgfqpoint{3.064444in}{0.780945in}}%
\pgfpathlineto{\pgfqpoint{3.136667in}{0.780945in}}%
\pgfpathlineto{\pgfqpoint{3.208889in}{0.780945in}}%
\pgfusepath{stroke}%
\end{pgfscope}%
\begin{pgfscope}%
\definecolor{textcolor}{rgb}{0.101961,0.101961,0.098039}%
\pgfsetstrokecolor{textcolor}%
\pgfsetfillcolor{textcolor}%
\pgftext[x=3.242222in,y=0.742057in,left,base]{\color{textcolor}{\rmfamily\fontsize{8.000000}{9.600000}\selectfont\catcode`\^=\active\def^{\ifmmode\sp\else\^{}\fi}\catcode`\%=\active\def
\end{pgfscope}%
\end{pgfpicture}%
\makeatother%
\endgroup%